\documentclass[11pt,a4paper]{article}

\usepackage[utf8]{inputenc}
\usepackage[T2A]{fontenc}

\usepackage{amsmath,amssymb,amsthm}
\usepackage{geometry}
\usepackage{hyperref}
\usepackage{booktabs}
\usepackage{enumitem}
\usepackage{algorithm}
\usepackage{algorithmic}
\usepackage{natbib}
\usepackage{graphicx}
\usepackage{mathtools}
\usepackage{url}
\usepackage{textcomp}
\usepackage{bbm}

\DeclareMathOperator{\tr}{tr}
\newtheorem{theorem}{Theorem}[subsection]
\newtheorem{lemma}[theorem]{Lemma}
\newtheorem{corollary}[theorem]{Corollary}
\newtheorem{proposition}[theorem]{Proposition}

\newtheorem{definition}{Definition}[subsection]
\newtheorem{assumption}{Assumption}[subsection]

\title{\Large \textbf{Ground-Truth-First Benchmarking of Reinforcement Learning via Converse Optimality: From Smooth, Fully Observed Systems to Non-Smooth or Partially Observed Tasks} \\[2ex]}

\author{%
  \large Gr\'egoire Ouerdane\textsuperscript{1,2} \quad
  Sinan Ibrahim\textsuperscript{1,2} \quad
  Stefan Streif\textsuperscript{3} \quad
  Pavel Osinenko\textsuperscript{1,2,4}\\[1.5ex]
  \footnotesize
  \textsuperscript{1}Center for Engineering Systems and Sciences, Skoltech, 30 Bolshoy Boulevard, bld.~1, Moscow, 121205, Russia\\[0.1ex]
	\footnotesize
  \textsuperscript{2}Central University, 7 Gasheka St., bld.~1, Moscow, 123056, Russia\\[0.3ex]
	\footnotesize
  \textsuperscript{3}Technische Universit\"at Chemnitz, Stra\ss{}e der Nationen 62, D-09111 Chemnitz, Germany\\[0.3ex]
	\footnotesize
  \textsuperscript{4}Sirius University of Science and Technology, 1 Olimpiyskiy Avenue, Sirius, 354340, Russia\\[0.7ex]
	\footnotesize
  \footnotesize
  Contact email: \texttt{p.osinenko@gmail.com}%
}

\begin{document}

\date{}

\maketitle

\begin{abstract}
    Benchmarking in reinforcement learning is usually relative: algorithms are ranked by returns on environments with unknown optima. Distance to optimality remains unobservable, while conclusions may depend on reward design, stochasticity, hyperparameters, environment selection, and simulator details. We develop a ground-truth-first framework based on converse optimality, constructing each task with a certified optimal value function and policy withheld from learners and used only for evaluation. The framework unifies two modes: sensitivity benchmarking samples certified environments from continuous-time generators to assess robustness to task variation, while oracle-referenced comparison uses native discrete-time generators to evaluate policies via optimality gap, regret, and policy error. Continuous-time models support physically interpretable sensitivity studies but introduce dependence on the simulator step ($\Delta t$) and numerical scheme. Native discrete-time models align the environment, oracle, and learner on the same transition law, eliminating external integration choices as a confounder. We reproduce quantitative results covering optimizer sensitivity and oracle-referenced comparisons on serial planar arms, dynamically extended nonholonomic vehicles, and quadrotor-inspired tasks. The methodology extends beyond smooth affine Gaussian systems: virtual-input pullback handles non-affine raw-action interfaces and finite-covariance non-Gaussian disturbances. A non-smooth converse optimality formulation enables certified benchmarks with kinked value functions and switching surfaces. For partially observed evaluation, an information-state formulation reduces benchmarking to converse optimality on a sufficient statistic and provides an information-penalty reference. Code, fixtures, and diagnostics are available on the public project page.
\end{abstract}

\tableofcontents

\section{Introduction}

Performance evaluation has become the default empirical language of reinforcement learning (RL). A new method is introduced, evaluated on benchmark environments, and compared against a changing set of baselines. This protocol is valuable, but it leaves a basic control-theoretic quantity hidden: \emph{distance to optimality}. When the optimal policy is unknown, a benchmark can tell us that one algorithm achieved a larger return than another; it cannot tell us whether that algorithm is near optimal, whether all compared methods are poor, or whether the ranking is an artifact of reward scaling, random seeds, task weighting, environment selection, or simulator conventions. This limitation is not superficial; it is fundamentally structural. Deep RL conclusions are well-known to depend strongly on implementation choices, evaluation protocols, random seeds, hyperparameters, and statistical reporting practices \citep{Henderson2018DeepRLThatMatters,Agarwal2021Precipice,patterson2024empirical}. For example, aggregate rankings may fluctuate significantly when tasks are weighted or normalized differently \citep{balduzzi2018re}. Recent critiques therefore argue that benchmark scores alone are insufficient evidence for broad algorithmic claims; controlled scientific tests are needed to explain how and when algorithms work \citep{jordan2024position}. For RL control, two particular questions arise: \emph{does a learning pipeline remain reliable when the environment varies?} and \emph{how far is a learned policy from the true optimum on a fixed task?} Conventional benchmarks struggle with both when the optimum is unavailable.

Converse optimality offers a direct alternative. It reverses the usual direction of optimal control: rather than solving a Bellman or Hamilton--Jacobi--Bellman (HJB) equation, one starts from a desired value function and policy and constructs environments for which they are optimal. The environment is therefore generated together with its certificate, turning evaluation from relative scoring to absolute calibrated measurement: optimality gaps, regret, Bellman residuals, and sensitivity across distributions of certified tasks. The central claim of this paper is that converse optimality transforms RL benchmarks into \emph{known-answer experimental instruments}: it tells not only which method returns the most, but why a method fails, how far it is from optimal, whether conclusions survive controlled environment variation, and whether the strongest baseline is genuinely near the certified optimum.

We organize this view into two complementary modes. In sensitivity benchmarking, a continuous-time generator defines a distribution over certified systems, turning environment variation into an experimental variable that tests pipeline consistency; the learner experiences a discretized version whose behavior depends on the integration step, which is precisely what sensitivity studies are meant to probe. In oracle-referenced comparison, native discrete-time fixtures define the transition law, oracle $(V^\star,\pi^\star)$, and evaluation protocol in the same formalism, removing the discretization confound and enabling absolute evaluation by optimality gap, regret, Bellman residual, and policy error.

\begin{figure}[h]
\centering
\includegraphics[width=0.99\linewidth]{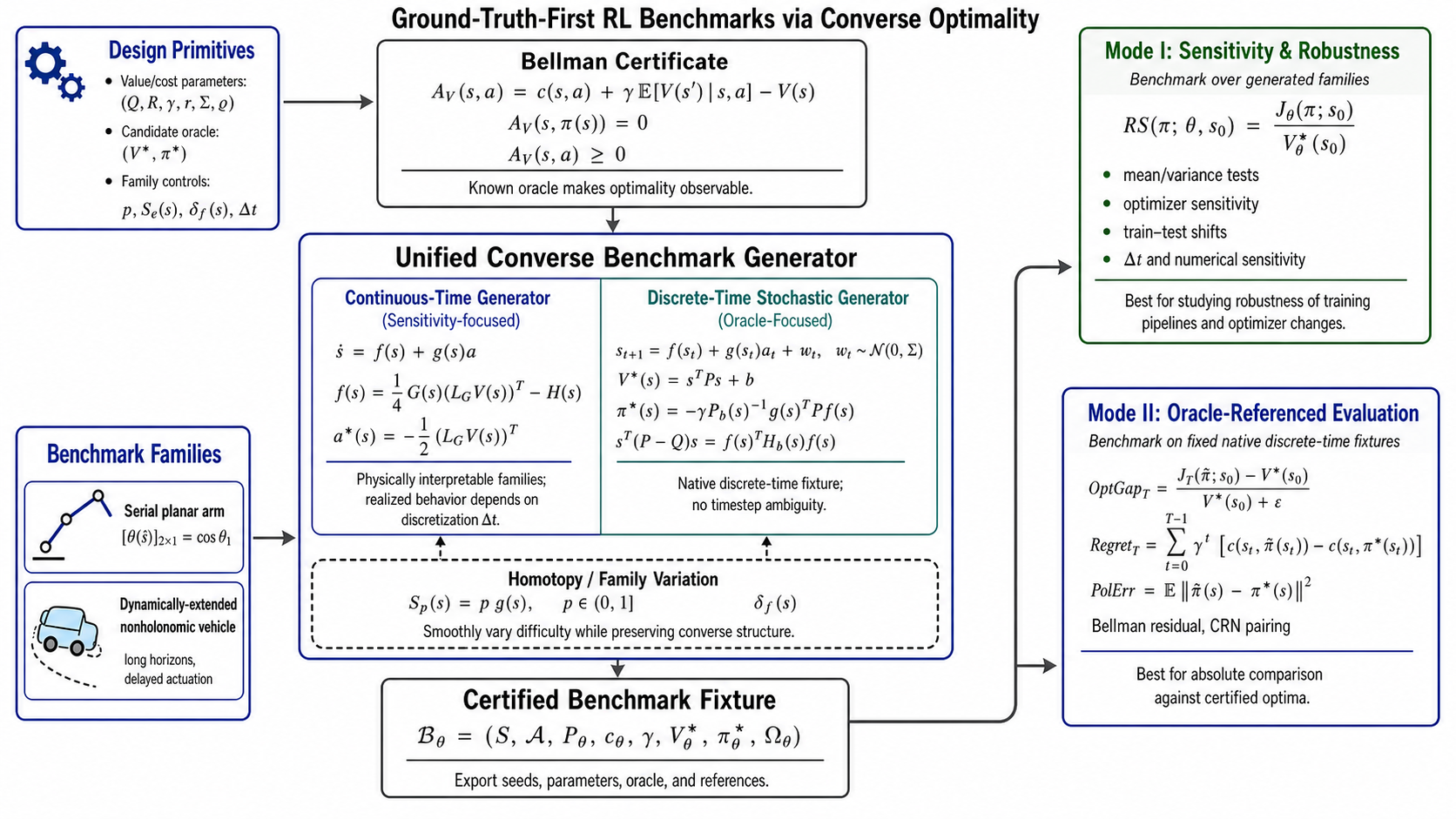}
\caption{Ground-truth-first benchmark-generation pipeline. A converse construction samples certified dynamics and costs, validates the optimality certificate, exports an environment fixture, trains learners without oracle access, and evaluates learned policies against the hidden pair $(V^\star,\pi^\star)$.}
\label{fig:pipeline}
\end{figure}

Time formalism is part of this design. Continuous-time generators produce physically interpretable non-linear systems, but the MDP experienced by a digital learner depends on the integration step $\Delta t$ and numerical scheme. This dependence is not a defect; it makes continuous‑time generators well‑suited for sensitivity studies. Native discrete-time generators define the transition law, oracle, and evaluation protocol in the same formalism, removing the discretization confound. Thus, continuous‑time generators are natural for sensitivity, whereas native discrete‑time fixtures are natural for absolute algorithm comparison.

The framework is not limited to smooth, affine, Gaussian, fully observed examples. The present work extends it along three independent axes. First, non-affine raw-action interfaces and finite-covariance non-Gaussian disturbances are handled by virual-input pullback, and demonstrated by the example of the ConverseQuad benchmark. Second, non-smooth value functions are managed via viscosity solutions and Clarke subdifferentials, enabling benchmarks with kinked values and switching surfaces. Third, partially observed evaluation is reduced to converse optimality on a sufficient information state, with a linear-quadratic Gaussian (LQG) information-penalty reference. These extensions share the same certificate logic; they differ in the interface the learner sees and the regularity of the certified value.

Our contribution is methodological. We review the foundations of converse optimality and contribute: (i) a ground-truth-first fixture formalism exporting certified values and policies; (ii) a separation between continuous-time sensitivity families and native discrete-time oracle-referenced fixtures; (iii) an extension to non-affine and non-Gaussian tasks, including ConverseQuad; (iv) a non-smooth converse theory via viscosity solutions and Clarke subdifferentials; (v) an information-state formulation for partially observed benchmarks; and (vi) a validation and evaluation protocol using Bellman consistency, closed-loop checks, and Common Random Numbers (CRN) schedules, demonstrated on optimizer sensitivity and on robotic control families. Section~\ref{sec:related-work} reviews related work, Sections~\ref{sec:fixtures}--–\ref{sec:mode2} develop the framework, Sections~\ref{sec:ext-nonaffine}--–\ref{sec:ext-pomdp} present the extensions, and Sections~\ref{sec:discussion}--–\ref{sec:conclusion} discuss limitations and open problems. The accompanying project page provides code, fixtures, results, and diagnostics.

\section{Related Work}
\label{sec:related-work}

\paragraph{RL benchmark suites.}
Shared environment suites---ALE, MuJoCo, Gym, Procgen, MiniGrid, bsuite, D4RL~\citep{bellemare2013arcade,todorov2012mujoco,brockman2016openai,cobbe2020leveraging,chevalier2023minigrid,osband2019behaviour,fu2020d4rl}---enable large‑scale empirical comparison but generally lack known optimal policies. Their primary signal is relative return, and a method may outperform baselines while remaining far from optimal.

\paragraph{Empirical reliability and statistical practice.}
RL conclusions depend heavily on implementation details, random seeds, and statistical reporting~\citep{Henderson2018DeepRLThatMatters,Agarwal2021Precipice,patterson2024empirical}. Benchmark scores alone are insufficient evidence for broad claims~\citep{jordan2024position}. We adopt this diagnostic perspective but change the experimental substrate: generating tasks with certified optima makes absolute error observable.

\paragraph{Classical optimal control and inverse optimal control.}
Classical optimal control fixes dynamics and cost and solves for a value function and policy, typically via the HJB equation or dynamic programming \citep{Bellman1957,lewis2012reinforcement}. This is intractable for general nonlinear systems due to the curse of dimensionality. Inverse optimal control and inverse reinforcement learning take the opposite direction: they infer a cost or reward from observed trajectories or demonstrations \citep{ng2000algorithms,abbeel2004apprenticeship}. Converse optimality differs from both. It neither solves a forward problem nor infers a cost from data; it prescribes a value function and constructs dynamics and state costs for which the prescribed pair is provably optimal. This makes it uniquely suited to benchmark generation, because the optimal reference is an input to the construction rather than an output of an intractable computation.

\paragraph{Converse optimality.}
Converse optimality was formalized for continuous-time deterministic systems by Nevistic and Primbs \citep{nevistic1996constrained}, using HJB certificates, and extended to discrete-time systems by Göhrt et al.\ \citep{gohrt2019converse}, using Bellman-equation conditions. Yaremenko et al.\ \citep{yamerenko2025generating} turned the discrete-time construction into a coordinate-free parametric generator for RL benchmarks, and Ibrahim et al.\ \citep{ibrahim2026benchmarking} extended it to stochastic control-affine systems with direct oracle metrics. We build on this line. We reuse its certificate logic, its quadratic specializations, and its CRN evaluation protocol, and we extend it in three directions: non-smooth values, non-affine raw actions with non-Gaussian disturbances, and information-state formulations for partially observed evaluation.

\paragraph{Non-smooth optimal control and viscosity solutions.}
Non-smooth value functions arise naturally in optimal control whenever the control enters linearly without a quadratic penalty (bang-bang control), when the horizon is minimized rather than discounted, when state or control constraints are active, or when the dynamics switch between modes. In these regimes the value function is typically only locally Lipschitz, with kinks along switching manifolds where the classical gradient does not exist. The standard machinery replaces pointwise derivatives with viscosity inequalities \citep{crandall1983viscosity,bardi1997optimal} and classical gradients with Clarke subdifferentials \citep{clarke1990optimization,clarke2001nonsmooth}; when the resulting optimal feedback is discontinuous, the closed loop is interpreted as a differential inclusion in the Filippov sense \citep{filippov1960differential}. Existing converse constructions, however, assume candidate values that are at least $C^1$ or $C^2$ \citep{nevistic1996constrained,gohrt2019converse,yamerenko2025generating}, which excludes precisely the switching-surface and bang-bang structures that make learning hard. We therefore develop a non-smooth converse theory in which the smooth HJB certificate is replaced by a viscosity–Clarke condition, yielding certified benchmarks whose optimal values have known kinks and whose optimal policies may be discontinuous.

\paragraph{Partial observability and information-state reduction.}
Partially observed control is classically reduced to a fully observed problem on an information state: a statistic of the observation–action history that is sufficient for dynamic programming \citep{smallwood1973optimal,kaelbling1998planning}. In finite-state POMDPs this statistic is a belief vector over the hidden state; in linear–Gaussian problems it is the Kalman filter's conditional mean and covariance; in general nonlinear continuous-state models it is an infinite-dimensional belief density. Recurrent policies approximate this reduction in practice \citep{hausknecht2015deep,kapturowski2018recurrent}, and memory-oriented benchmarks probe it empirically \citep{morad2023popgym}, but they do not provide a certified optimal reference. A further complication is that a full-state oracle may exploit information the learner cannot access from its history, so comparing a history-based policy against a latent-state optimum conflates learning error with an information penalty. No converse construction has been formulated on an information state with a certified output-feedback oracle. We address this by reducing partially observed benchmarking to converse optimality on a sufficient statistic, certifying the oracle in the learner's own information set, and isolating the information penalty analytically in an LQG reference.

\section{Ground-Truth-First Fixtures}
\label{sec:fixtures}

A standard RL benchmark provides an environment but not the optimal return. A \emph{ground-truth-first} benchmark adds a certificate that a particular value function \(V^{\star}\) and policy \(\pi^{\star}\) are indeed optimal. We work within the usual Markov decision process (MDP) framework with state space \(\mathcal{S}\), action space \(\mathcal{A}\), transition kernel \(P\), stage cost \(c\), and discount \(\gamma\). For a stationary policy \(\pi\), the value function is
\[
V^{\pi}(s) = \mathbb{E}\!\left[ \sum_{k=0}^{\infty} \gamma^{k} c(s_{k}, \pi(s_{k})) \;\Big|\; s_{0}=s \right],
\]
and the optimal value function satisfies \(V^{\star}(s) = \inf_{\pi} V^{\pi}(s)\). We state the certificate for this discounted, fully observed, smooth discrete-time case, and generalize it in Sections~\ref{sec:ext-nonaffine}--\ref{sec:ext-pomdp} to non-smooth values, non-affine interfaces, and information states.

\paragraph{Cost and reward conventions.}
All certificates are stated in the \emph{cost} convention, where lower is better and \(V^{\star}\) is nonnegative for nonnegative stage costs. When results are reported in the \emph{reward} convention \(r = -c\), so that \(V_{r}^{\star} = -V^{\star}\) and the optimality gap \(\mathrm{Gap}^{r} = -\mathrm{Gap}^{c}\), so that zero is optimal and more negative values indicate larger suboptimality. The convention used in each experiment is stated explicitly in the corresponding caption. This dual convention matters because the Bellman certificate is naturally a cost identity, whereas standard RL libraries report rewards.

\paragraph{Optimality certificate.}
To check optimality, we use the \emph{Bellman residual}
\[
\mathcal{R}_{V}(s,a) = c(s,a) + \gamma \int V(s') \, P(\mathrm{d}s' \mid s,a) - V(s).
\]
The \emph{optimality certificate} states: if for all states and actions \(\mathcal{R}_{V}(s,\pi(s)) = 0\) and \(\mathcal{R}_{V}(s,a) \ge 0\), then \(V = V^{\star}\) and \(\pi\) is optimal (see Appendix~\ref{app:smooth-proofs}). The certificate is checked analytically where possible, and numerically on a declared validation set \(\mathcal{G}\) with tolerance \(\epsilon_{B}\). Accepted fixtures record the maximum residual, the validation domain, and the numerical tolerance in the validation record. This makes the certificate independently auditable: an external party can reload a fixture, recompute the residual, and verify that the prescribed pair is consistent with the transition law and cost.

\paragraph{Fixtures.}
A \emph{benchmark fixture}
\[
\mathfrak{B}_{\theta} = (\mathcal{S}, \mathcal{A}, P_{\theta}, c_{\theta}, \gamma, V_{\theta}^{\star}, \pi_{\theta}^{\star}, \Omega_{\theta})
\]
packages the certified MDP with a validation record. The tuple may additionally carry the time formalism, the regularity class of \(V^{\star}\), the action interface the learner sees, and the observation map. The agent trains on the MDP \((\mathcal{S}, \mathcal{A}, P_{\theta}, c_{\theta}, \gamma)\); \(V_{\theta}^{\star}\) and \(\pi_{\theta}^{\star}\) are withheld as evaluation oracles. When \(\theta\) is fixed, the fixture provides a single known-answer task. When \(\theta\) is sampled from a distribution, one obtains a \emph{distribution over certified MDPs}, enabling the two modes in Table~\ref{tab:modes}.

\begin{table}[t]
\centering
\footnotesize
\begin{tabular}{p{0.12\linewidth}p{0.15\linewidth}p{0.34\linewidth}p{0.25\linewidth}}
\toprule
Mode & Time formalism & Scientific question & Typical outputs \\
\midrule
Sensitivity / robustness & Continuous-time generator & Does a learning pipeline remain reliable when the environment varies? & Relative suboptimality distributions, mean and variance tests, train--test shift studies \\
Oracle-referenced comparison & Native discrete-time fixture & How far is a learned policy from the certified optimum on a fixed fixture? & Optimality gap, regret, Bellman residual, policy error, paired CRN confidence intervals \\
\bottomrule
\end{tabular}
\caption{Two complementary benchmark modes enabled by ground-truth-first fixtures. Mode~I is typically instantiated with a continuous-time generator and Mode~II with a native discrete-time fixture; Section~\ref{sec:mode1} and Section~\ref{sec:mode2} explain why.}
\label{tab:modes}
\end{table}

\paragraph{What the extensions change.}
The extensions in Sections~\ref{sec:ext-nonaffine}--\ref{sec:ext-pomdp} modify one element of this picture each, while leaving the fixture concept unchanged. Non-smooth values replace the pointwise Bellman residual with a viscosity--Clarke condition, so that the certificate remains valid at kinks and switching surfaces. Non-affine raw actions and finite-covariance non-Gaussian disturbances change the action interface and the disturbance law, while preserving the same certificate in virtual-input coordinates. Partial observability replaces the state \(s\) with a sufficient information state \(b\), so that the oracle is attainable from the learner's history rather than from a hidden state; this also makes the information penalty separable from the learning gap. In all three cases, the task and its certificate are constructed together, and the reference pair is withheld until evaluation. The remainder of the paper develops the two modes (Sections~\ref{sec:mode1}--\ref{sec:mode2}), the three extensions (Sections~\ref{sec:ext-nonaffine}--\ref{sec:ext-pomdp}), and validation and reproducibility protocols.

\section{Mode I -- Continuous-Time Sensitivity Benchmarking}
\label{sec:mode1}

The certificate of Section~\ref{sec:fixtures} guarantees optimality when certain conditions hold. To turn that certificate into a usable benchmark, one needs \emph{converse generator templates}---recipes that, given a desired value function \(V\) and policy, produce the dynamics and costs that satisfy the certificate. This section develops Mode~I for continuous-time sensitivity studies. Mode~II (Section~\ref{sec:mode2}) develops native discrete-time fixtures for absolute comparison, and Sections~\ref{sec:ext-nonaffine}--\ref{sec:ext-pomdp} extend both to non-affine interfaces, non-smooth values, and partial observability. Modes~I and~II together cover the smooth, fully observed, affine-action case; the extensions relax each of these in turn, while leaving the certificate logic unchanged.

\subsection{Continuous-time converse generators for sensitivity}

Consider a continuous-time, control-affine system \(\dot{s}=f(s)+g(s)a\) with state \(s\in\mathbb{R}^{n}\) and action \(a\in\mathbb{R}^{m}\). We prescribe a positive definite \(C^{1}\) value function \(V\). A converse construction~\citep{yamerenko2025generating} uses two designer-chosen smooth fields: a vector field \(H(s)\) shaping the drift, and a matrix field \(G(s)\) directing the control. The generator outputs
\[
  g(s)=G(s),\quad
  f(s)=\frac{1}{4}G(s)\big(\mathcal{L}_{G}V(s)\big)^{\!\top}-H(s),\quad
  c(s,a)=\mathcal{L}_{H}V(s)+\|a\|_{2}^{2},
\]
where \(\mathcal{L}_{G}V(s)\) is the Lie derivative of \(V\) along the columns of \(G\). A completion of squares (Appendix~\ref{app:ct_derivation}) gives the Hamiltonian as \(\|a+\tfrac12(\mathcal{L}_{G}V(s))^{\!\top}\|_{2}^{2}\), yielding the certified optimal policy \(a^{\star}\!(s)=-\tfrac12(\mathcal{L}_{G}V(s))^{\!\top}\).

\paragraph{From generator to sensitivity benchmark.}
The fields \(H\) and \(G\) are \emph{knobs} that parametrize a rich family of certified nonlinear control problems (stable/unstable, underactuated/fully actuated, stiff/slack). Sampling them from a distribution produces a \emph{distribution over certified MDPs}---the substrate for sensitivity and robustness studies. Because the certificate is continuous-time, the digital RL agent experiences a discretized version whose dynamics depend on the step size \(\Delta t\) and integration method. This dependency is not a bug: it makes environment variation and numerical realization part of the experimental variable, precisely what sensitivity studies require. It is also why Mode~I is used for sensitivity and not for absolute comparison. Mode~II, and its extensions in Sections~\ref{sec:ext-nonaffine}--\ref{sec:ext-nonsmooth}, removes this confound by working natively in discrete time, so that learner, oracle, and metrics share the same transition law.

\subsection{Sensitivity evaluation protocol}
\label{subsec:sensitivity_protocol}

A continuous-time generator defines a distribution over certified environments indexed by \(\theta\). For a sampled environment \(\theta\) and a learning procedure, a trained policy \(\pi_{\theta}\) is obtained. The \emph{relative suboptimality} (RS) normalizes the incurred cost by the known optimal cost:
\[
  \mathrm{RS}(\pi_{\theta}; \theta, s_{0}) = \frac{J_{\theta}(\pi_{\theta}; s_{0})}{V_{\theta}^{\star}(s_{0}) + \epsilon},
\]
where \(J_{\theta}\) is the cumulative cost and \(\epsilon>0\) is a small constant that guards against division by zero. RS is task-aware: it measures excess cost as a multiple of the optimum, independent of raw task scale. The cost convention is defined in Section~\ref{sec:fixtures}; reward-convention plots flip the sign.

Sampling \(N\) environments and training on each yields a sample \(\{\mathrm{RS}_{i}\}\). Two procedures are compared by their RS distributions---the question is not who wins on one task, but whether one procedure systematically produces a different suboptimality profile across tasks. An \(F\)-test is applied for variance (consistency) and a \(t\)-test for means (average performance). The protocol is: sample environments, train, compute RS with the known \(V_{\theta}^{\star}\), and test. This turns environment variation from a nuisance into a controlled, statistically measurable experimental variable.

\begin{figure}[t]
    \centering
    \includegraphics[width=0.9\linewidth]{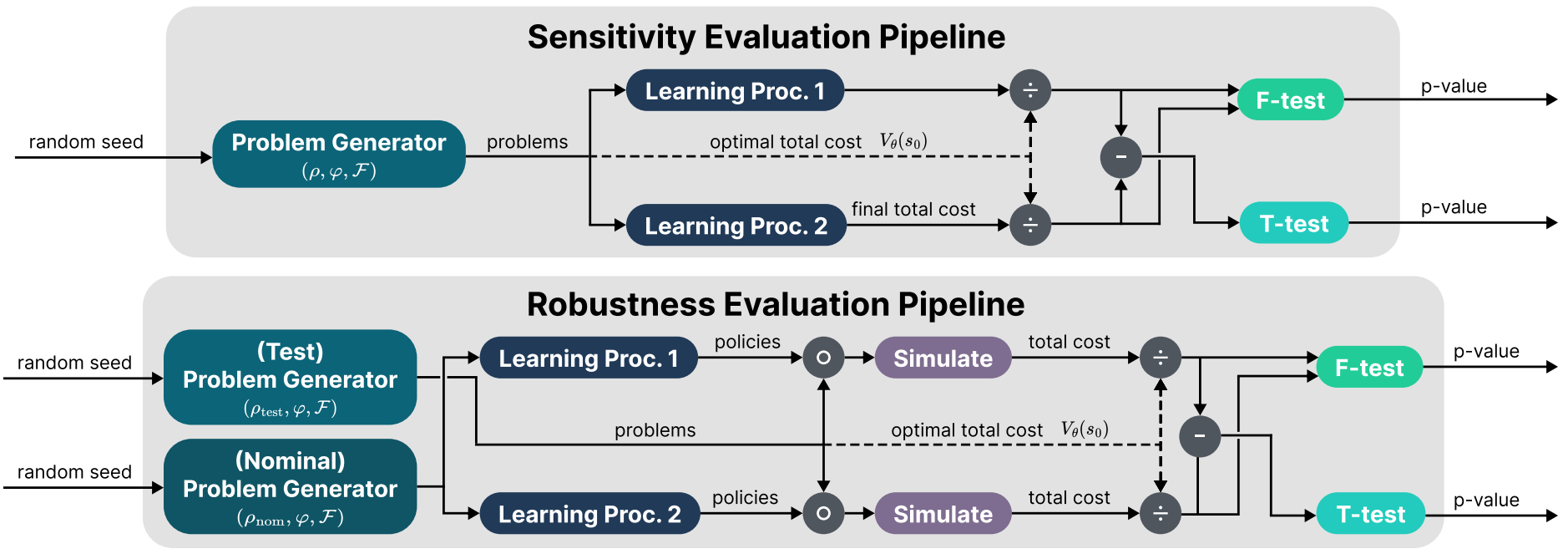}
    \caption{Reproduced from Yaremenko \textit{et al.}~\citep{yamerenko2025generating}. Sensitivity (top) and robustness (bottom) test structures. The symbol $\circ$ indicates that the policy is substituted into the test problem, closing the loop. The problem sample is split evenly at the junction.}
    \label{fig:sens-pipeline}
\end{figure}

\subsection{Sensitivity experiment}

The sensitivity protocol is illustrated with the continuous-time case study of a generator producing a sample of \(160\) unstable three-dimensional control problems, each with a known optimal value~\citep{yamerenko2025generating}. Two PPO pipelines, identical in all respects except the optimizer (Adam \emph{vs.}~RAdam), were trained on independent splits of \(80\) environments.

Relative suboptimality is computed for each trained policy \(\pi_{\theta}\). An \(F\)-test comparing the two RS distributions gave a variance ratio \(F \approx 2.59\) (\(p \approx 1.7\times 10^{-5}\)), rejecting the null of equal variance at any conventional significance level.

This result does not mean one optimizer is uniformly ``better.'' It shows that the optimizer choice materially alters the \emph{spread} of suboptimality across environments---a measure of pipeline \emph{consistency}. A return-only benchmark would have obscured this effect entirely. The converse-generated family makes sensitivity a measurable and statistically testable property. The result is a reported sensitivity study; independent reproduction requires the original generator schedule and training seeds, which are discussed in Appendix~\ref{app:reproducibility}.

\begin{table}[t]
\centering
\footnotesize
\begin{tabular}{p{0.29\linewidth}p{0.61\linewidth}}
\toprule
Quantity & Value \\
\midrule
Generator family & Unstable 3D continuous-time converse-optimal systems \\
Learning algorithm & PPO \citep{schulman2017proximal} \\
Perturbation under test & Adam vs.\ RAdam optimizer \\
Environment sample size & 160 generated problems, split into 80 per optimizer arm \\
Normalization & Relative suboptimality using each environment's known optimal value \\
Test statistic & Variance ratio $F\approx2.59$ \\
$p$-value & $\approx1.7\times10^{-5}$ \\
Interpretation & Optimizer choice materially changed outcome variance across certified environments \\
\bottomrule
\end{tabular}
\caption{Continuous-time sensitivity case study. Two PPO pipelines differed only in optimizer. The measured quantity is relative suboptimality over a generated family of certified environments. Values are reported from the source; the stated test concerns the variance of relative suboptimality.}
\label{tab:sensitivity}
\end{table}

\section{Mode II -- Native Discrete-Time Oracle-Referenced Comparison}
\label{sec:mode2}

Mode~II uses native discrete-time generators. Unlike Mode~I, the transition law itself is the benchmark: learner, oracle, and metrics share the same map, which removes the discretization confound and makes these fixtures suitable for absolute comparison. We proceed from deterministic to stochastic dynamics, then to example families, then to the evaluation protocol. Extensions to non-affine actions and non-Gaussian noise (Section~\ref{sec:ext-nonaffine}), non-smooth values (Section~\ref{sec:ext-nonsmooth}), and partial observability (Section~\ref{sec:ext-pomdp}) preserve the same certificate logic; only the interface and the regularity class change.

\subsection{Deterministic discrete-time generator}

Consider a deterministic, control-affine, discrete-time system \(s_{k+1} = f(s_k) + g(s_k)\,a_k\), with state \(s\in\mathbb{R}^{n}\) and action \(a\in\mathbb{R}^{m}\). The stage cost \(c(s,a) = q(s) + \rho(a)\) is separable, with \(q \ge 0\) and \(\rho\) strictly convex with \(\rho(0)=0\). For a given input-coupling function \(g\) and control cost \(\rho\), the converse optimality problem asks: given a candidate optimal value function \(V^{\star}\), find a drift \(f\) such that \(V^{\star}\) satisfies the Bellman equation for the considered system.

G\"ohrt \textit{et al.}~\citep{gohrt2019converse} provided a complete answer. A pair \((V^{\star}, q)\) is \emph{admissible}---i.e., there exists a drift \(f\) realizing \(V^{\star}\) as the optimal value function---if and only if \(V^{\star}(s) - q(s) \ge 0\) for all \(s\in\mathbb{R}^{n}\). The sufficiency proof is constructive and uses a homotopy argument (see Appendix~\ref{app:det-homotopy}). When the value function, the state cost, and the control cost are quadratic, with \(V^{\star}(s) = s^{\!\top} P s\), \(q(s) = s^{\!\top} Q s\), \(\rho(a) = a^{\!\top} R a\) and \(P,Q,R \succ 0\), the admissible drift can be written explicitly as
\begin{equation}
  f(s) = H(s)^{-1/2}\,S(s)\,(P-Q)^{1/2}\,s,
  \qquad S(s)^{\!\top} S(s) = I,
  \label{eq:det_mn_param}
\end{equation}
where \(H(s) = \big(P^{-1} + g(s)R^{-1}g(s)^{\!\top}\big)^{-1}\). The orthogonal field \(S(s)\) can be chosen freely and serves as a design knob for injecting nonlinear coupling while preserving optimality. The optimal policy is \(a^{\star}(s) = -\big(R + g(s)^{\!\top} P g(s)\big)^{-1} g(s)^{\!\top} P f(s)\).

\subsection{Stochastic discrete-time generator}

The converse framework can also be extended to stochastic systems \(s_{k+1}=f(s_k)+g(s_k)a_k+w_k\) with additive Gaussian noise \(w_k\sim\mathcal{N}(0,\Sigma)\)~\citep{ibrahim2026benchmarking}. The optimal value function satisfies the stochastic Bellman equation
\[
V^{\star}(s)=\inf_{a}\big\{c(s,a)+\gamma\,\mathbb{E}\big[V^{\star}(f(s)+g(s)a+w)\big]\big\}.
\]
An analogous necessary-and-sufficient optimality condition holds (see Appendix~\ref{app:smooth-proofs}). The cost convention of Section~\ref{sec:fixtures} is used throughout.

For benchmark generation, the most useful instance is the \emph{quadratic-Gaussian} case with value function \(V^{\star}(s)=s^{\!\top}Ps+b\), noise certainty equivalence offset \(b\), and quadratic stage costs. The stochastic Bellman equation reduces to a single \emph{energy equation}
\begin{equation}
s^{\!\top}(P-Q)s=f(s)^{\!\top}H_{\gamma}(s)f(s),
\label{eq:energy}
\end{equation}
with an explicit state-dependent metric \(H_{\gamma}\succ 0\), which any drift must satisfy to make \(V^{\star}\) optimal. Equation~\eqref{eq:energy} is not a constraint to be checked after the fact; it is the \emph{design equation}. Any drift satisfying it yields a certified fixture, and the free parameter---the orthogonal field \(S(s)\) below---injects nonlinearity without breaking the certificate. Two explicit parameterizations exist: a \emph{metric-normalized} form \(f(s)=H_{\gamma}^{-1/2}S(s)(P-Q)^{1/2}s\) with an orthogonal field \(S(s)\), and a \emph{square-root} form. Both inject bounded nonlinear coupling while preserving the energy identity exactly. A control-strength homotopy \(g_p(s)=p\,g(s),\;p\in(0,1]\) varies authority, conditioning, and task difficulty without invalidating the certificate; larger \(p\) improves the oracle but amplifies sensitivity to action error, the \emph{control-authority paradox} analysed in Section~\ref{subsec:oracle_experiment}.

In native discrete time, the transition law itself is the benchmark: learner, oracle, and metrics share the same map. This eliminates discretization confounds and makes these fixtures ideal for absolute, calibrated algorithm comparison (Mode~II).

\subsection{Example families}
\label{subsec:families}

The following families instantiate a general recipe rather than exhaust it. Each family highlights a particular design dimension such as state-dependent actuation, dynamic extension, or non-affine interfaces that the generator can accommodate. Complexity is introduced through certified parameters: link count, module count, dynamic-extension depth, coupling strength, control authority, actuator coordinates, and disturbance structure. A certified latent MDP can be wrapped with partial or image-valued observations, but an output-feedback oracle requires an information-state certificate (Section~\ref{sec:ext-pomdp}); a hidden-state oracle is not automatically attainable from the learner's history.

\paragraph{State-dependent actuation and nonlinear coupling.}
As a concrete example, consider a first family which models a serial \(n\)-link planar arm with state \(s=[\theta_1,\omega_1,\ldots,\theta_n,\omega_n]^\top\in\mathbb{R}^{2n}\) and action \(a=[\tau_1,\ldots,\tau_n]^\top\in\mathbb{R}^n\). The input map is configuration dependent with \([g(s)]_{2i,i}=\cos\theta_i\) for \(i=1,\ldots,n\), with all other entries zero. The drift is generated by
\begin{equation}
f_p(s)=H_\gamma^{(p)}(s)^{-1/2}S_p(s)(P-Q)^{1/2}s.
\label{eq:arm-f}
\end{equation}
The orthogonal field \(S_p(s)\) is a product of Givens rotations containing \(\tanh\) terms that create bounded velocity-dependent and neighbor-coupling effects. A sinusoidal term gives a gravity-like bias. The family is non-linear, scalable in dimension, and certified by construction.

\paragraph{Dynamic extension and delayed actuation.}
The second example illustrates dynamic extension with a nonholonomic vehicle family~\citep{ibrahim2026benchmarking}, where for module \(i\), the state is \(s_i=[x_i,y_i,\phi_i,v_i^{(0)},\ldots,v_i^{(r_v-1)},\omega_i^{(0)},\ldots,\omega_i^{(r_\omega-1)}]^\top\) and the action is \(a_i=[a_{v,i},a_{\omega,i}]^\top\). The kinematic states evolve as
\begin{equation}
x_{i,k+1} = x_{i,k}+\tau v^{(0)}_{i,k}\cos\phi_{i,k},\quad
y_{i,k+1} = y_{i,k}+\tau v^{(0)}_{i,k}\sin\phi_{i,k},\quad
\phi_{i,k+1} = \phi_{i,k}+\tau \omega^{(0)}_{i,k}.
\label{eq:nv-kin}
\end{equation}
Velocity chains shift upward, with control entering only at the highest derivative. Modules are stacked through the block-diagonal matrix \(G\), and obtain the overall state \(s_{k+1}=f_p(s_k)+pGu_k+w_k\). Dynamic extension introduces delayed actuation, long credit-assignment paths, and stiff closed-loop behavior, stressing RL algorithms differently from the arm family.

\paragraph{Non-affine actuation and non-Gaussian gusts.}
A third family, ConverseQuad, is presented in Section~\ref{sec:ext-nonaffine}. It exposes a state-dependent aerial actuation model through a nonlinear raw-action coordinate and admits finite-covariance non-Gaussian disturbances while preserving the certified optimum in virtual-input coordinates.

\subsection{Oracle-referenced evaluation protocol}
\label{subsec:oracle_protocol}

Fixtures are validated (see Appendix~\ref{app:reproducibility}) before export. Four absolute metrics are measured.

\paragraph{Normalized optimality gap (tail-corrected).}
The tail-corrected cost gap is
\[
    \mathrm{Gap}_{T}^{c}(\hat{\pi};s_0)
    =
    \frac{
    \mathbb{E}\!\left[
    \sum_{k=0}^{T-1}\gamma^k c(s_k,\hat{\pi}(s_k))
    +
    \gamma^T V^\star(s_T)
    \right]
    -
    V^\star(s_0)
    }{
    |V^\star(s_0)|+\epsilon
    },
\]
where \(\epsilon>0\) guards against division by zero. The tail term removes truncation bias, yielding zero for the optimal policy (up to numerical precision) and nonnegative values for suboptimal policies. Under the reward convention, the optimality gap is \(G^\star = -V^\star\), so zero is optimal and more negative values indicate larger suboptimality. Each plot states explicitly whether \(\mathrm{Gap}_{T}^{c}\) or \(\mathrm{Gap}_{T}^{r}\) is used.

\paragraph{Discounted regret.}
Regret accumulates the per-step excess cost incurred by \(\hat{\pi}\) relative to the optimal policy. It is a trajectory-level diagnostic that complements the scalar gap. Lower cost regret is better.
\[
    \mathrm{Regret}_{T}^{c}(\hat{\pi};s_0)
    =
    \mathbb{E}\!\left[
    \sum_{k=0}^{T-1}
    \gamma^k
    \Big(
    c(s_k,\hat{\pi}(s_k))
    -
    c(s_k,\pi^\star(s_k))
    \Big)
    \right].
\]
If a reward-convention regret is plotted, it is defined by \(\mathrm{Regret}_{T}^{r}=-\mathrm{Regret}_{T}^{c}\), and the sign convention is stated in the corresponding caption.

\paragraph{Policy error and Bellman residual.}
Policy error is \(\mathbb{E}_{s\sim\mu}\|\hat{\pi}(s)-\pi^{\star}(s)\|_{2}^{2}\). The Bellman residual \(\mathcal{R}_{\hat{V}}\) is reported when a critic \(\hat{V}\) is available.

\paragraph{CRN.}
All policies are evaluated on a shared pre-generated schedule of start states and noise streams. Paired bootstrap over trajectories yields confidence intervals, ensuring that differences reflect algorithms, not noise realizations.

\subsection{Oracle-referenced experiment}
\label{subsec:oracle_experiment}

PPO, A2C, SAC, TD3, and DDPG are evaluated on ConverseArm and NVDEx native discrete-time fixtures under a CRN schedule. The three regimes probe different failure modes: the \texttt{serial\_random} tests broad occupancy coverage with diverse initial states, and \texttt{serial\_fixed} tests narrow stationary distributions around a fixed initial state, and NVDEx tests long-horizon delayed actuation (\(T=700\)). Figure~\ref{fig:arm-reward} shows that the oracle-referenced ranking is strongly regime-dependent. On the serial-random fixtures, TD3 has the best point estimate in both optimality gap and regret, followed by SAC and PPO; A2C and DDPG remain substantially farther from the certified optimum.

\begin{figure}[t]
    \centering
    \includegraphics[width=1\linewidth]{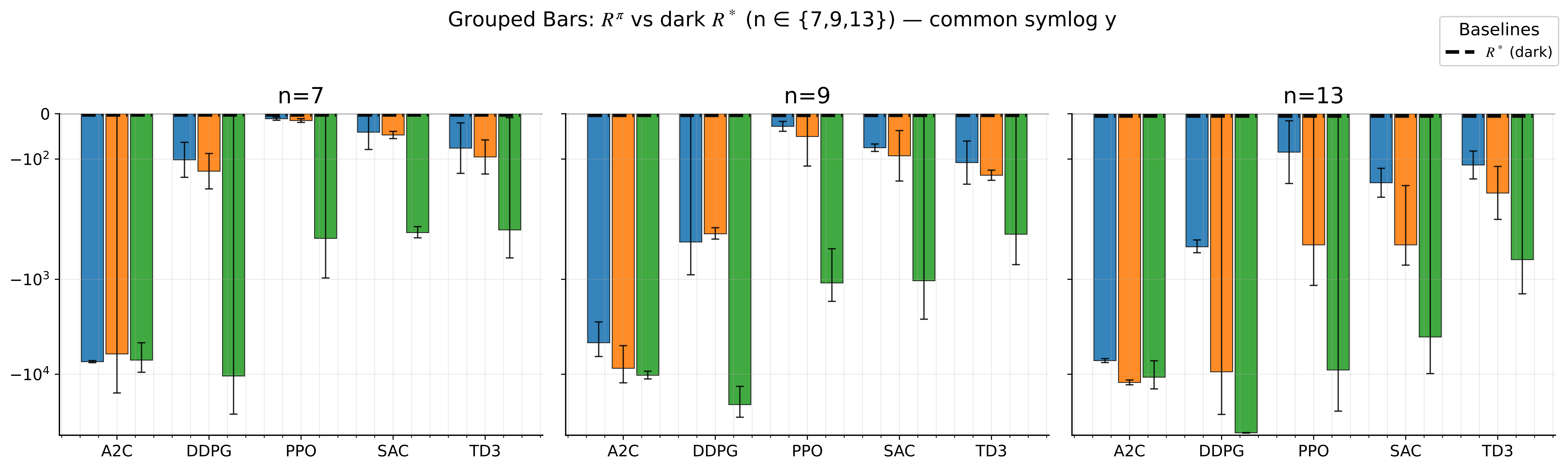}
    \caption{Comparing algorithm absolute performance. Discounted reward bars \((r = -c)\) with 95\% CIs under CRN. Colors correspond to \(p \in \{0.5, 0.6, 0.8\}\).}
    \label{fig:arm-reward}
\end{figure}

The ranking changes on the serial-fixed fixtures. PPO is closest to the oracle by a large margin, while SAC and TD3 exhibit large regret despite their strong performance in many standard continuous-control benchmarks. This indicates that the fixed-initial-state regime creates a narrow, stationary occupancy distribution in which PPO's clipped on-policy updates are more stable, whereas off-policy bootstrapping and replay do not provide the same advantage.

The NVDEx regime shows the strongest separation: PPO is near the certified optimum, while DDPG, SAC, TD3, and A2C remain far from zero. This reflects the vehicle's dynamic extension: delayed actuation, a longer horizon, and error propagation through feedback-sensitive dynamics. The result does not claim universal PPO superiority; rather, it shows that this fixture exposes a failure mode where delayed actuation and long-horizon conditioning challenge off-policy baselines under the reproduced protocol.

Together, Figures~\ref{fig:arm-reward} and~\ref{fig:signed_gap_distributions} show that certified fixtures turn RL evaluation from relative ranking into diagnostic measurement. No single algorithm dominates: TD3 excels on serial-random, while PPO excels on serial-fixed and NVDEx. Distributional plots expose lower-tail failures hidden by means. Because the optimum is known, one can measure not only which method performs best, but also how far each is from optimality and how error is distributed across runs and fixtures.

\begin{figure}[t]
    \centering
    \includegraphics[width=1.0\linewidth]{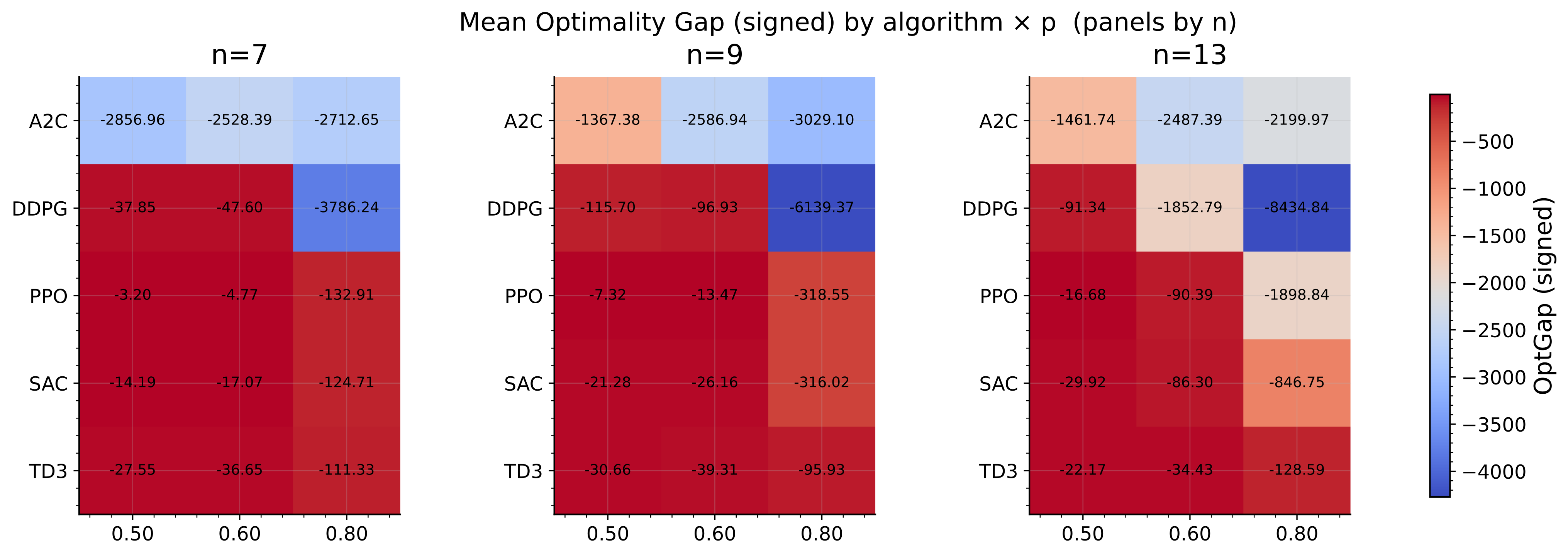}
    \caption{Serial \(n\)-link planar arm (ConverseArm): absolute oracle-referenced performance (optimality gap) across algorithms (rows) and control authority \(p\) (columns), with separate panels for \(n\).}
    \label{fig:arm-gap-heatmap}
\end{figure}

\begin{figure}[t]
    \centering
    \includegraphics[width=1.0\linewidth]{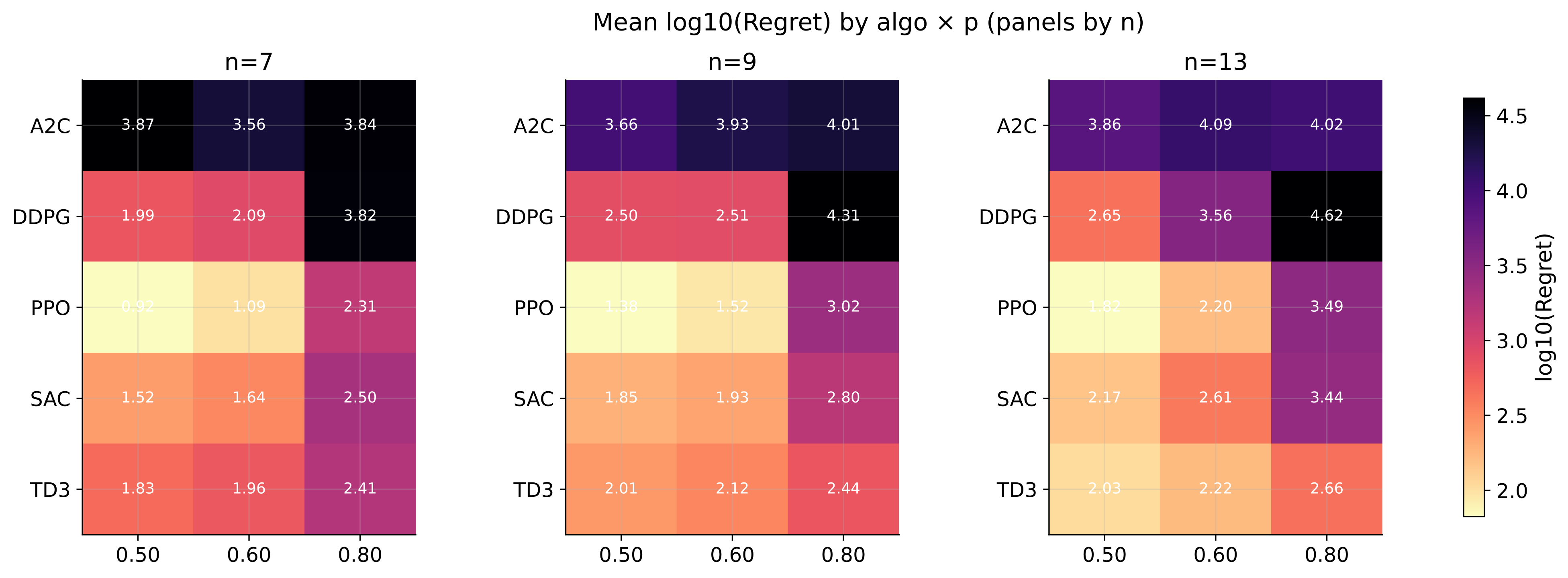}
    \caption{Serial \(n\)-link planar arm (ConverseArm): absolute oracle-referenced performance (regret) across algorithms (rows) and control authority \(p\) (columns), with separate panels for \(n\).}
    \label{fig:arm-regret-heatmap}
\end{figure}

\begin{figure*}[t]
\centering
\includegraphics[width=1.0\textwidth]{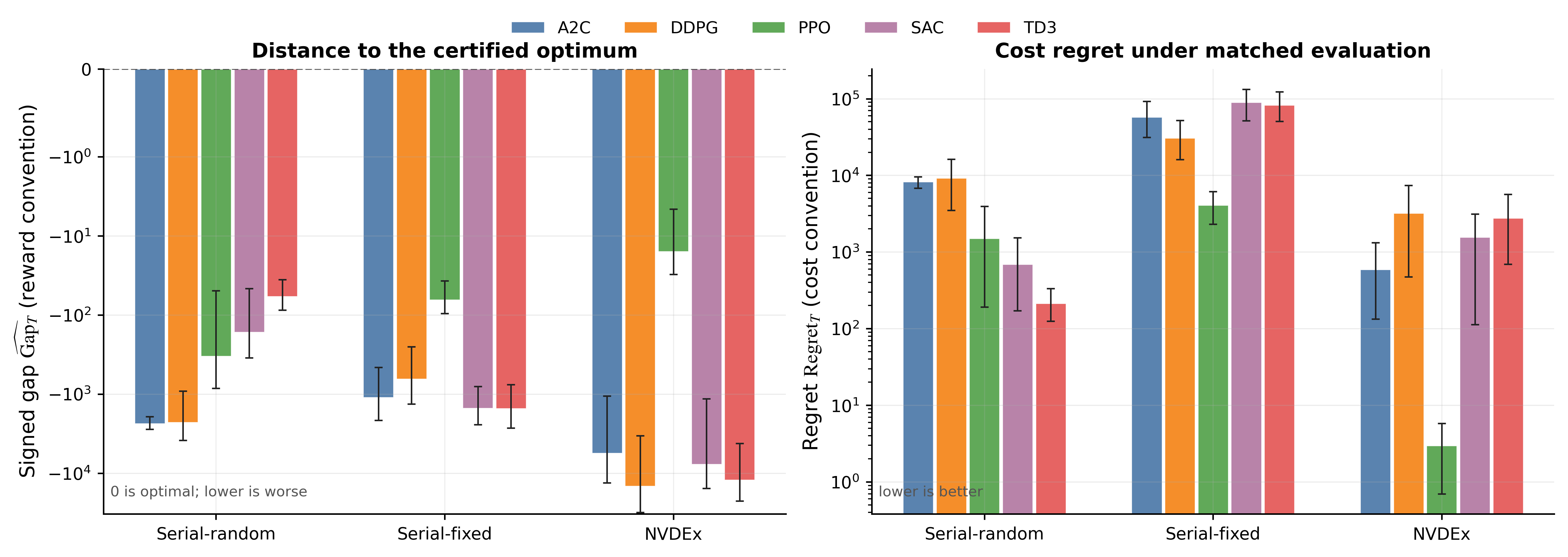}
\caption{Left: reward-convention optimality gap to the certified optimum, where zero is optimal and more negative values are worse. Right: cost regret under matched CRN evaluation, where lower is better. Bars show means and error bars show paired-bootstrap 95\% confidence intervals. The logarithmic vertical scales show order-of-magnitude differences across algorithms and regimes.}
\label{fig:table3_companion}
\end{figure*}

In the serial-random regime, A2C is concentrated far below zero, while DDPG has a wide distribution with a severe negative tail. PPO, SAC, and TD3 are shifted closer to the oracle, with TD3 having the best central performance. This supports the interpretation that randomized initial states favour algorithms that can exploit broad transition coverage, while still revealing instability in DDPG.

In the serial-fixed regime, PPO's distribution is shifted closest to zero, whereas SAC and TD3 exhibit lower-tail failures. Thus, the performance change from serial-random to serial-fixed is not only a change in mean score; it is a change in the shape of the error distribution. The fixed regime appears to favour stable on-policy improvement over replay-based reuse.

The NVDEx panel is the clearest example of why oracle-referenced distributions are useful. PPO remains concentrated near the optimum, while the other baselines show broad distributions and severe negative tails. Since NVDEx contains dynamic extension and delayed actuation, these tails indicate that small policy or value-estimation errors can be amplified over long horizons. The certified value function makes this failure severity measurable in oracle units.

\section{Extension I: ConverseQuad -- Non-Affine Actuation and Non-Gaussian Gusts}
\label{sec:ext-nonaffine}

The fixture methodology is not limited to affine actions or Gaussian disturbances. This section develops \textbf{ConverseQuad}, a certified aerial family in which the RL agent acts through a nonlinear raw-actuator coordinate while the optimum is constructed in a latent affine coordinate and pulled back exactly. The family exposes a failure mode distinct from ConverseArm and NVDEx: policy errors are transformed by the actuator Jacobian before entering the dynamics, so methods that propagate critic-gradient errors through the action map are more exposed, and finite-covariance non-Gaussian gusts increase trajectory variability without changing the certified optimum. The construction is summarized in the following subsections; the state decomposition, input map, and energy identity are derived inline in Sections~\ref{subsec:quad-virtual}--\ref{subsec:quad-noise}. This section summarizes the construction and reports the results available in the source paper~\citep{ibrahim2026benchmarking}.

\paragraph{What this family tests.}
ConverseQuad isolates one interface change and one disturbance change while leaving the rest of the quadratic-Gaussian (QG) construction of Section~\ref{sec:mode2} intact. On the interface side, the agent acts in raw actuator coordinates \(u\), and the environment is non-affine in \(u\); the certified optimum is defined in a virtual coordinate \(\bar u = \phi(u)\) and pulled back exactly. Because the actuator map is smooth but nonlinear, a policy error in raw coordinates is transformed by the Jacobian \(\phi'(u) = 1 + 3\eta u^{2}\) before entering the dynamics, which amplifies critic-gradient errors for methods that backpropagate through the action. On the disturbance side, the noise law is any zero-mean finite-covariance distribution, not necessarily Gaussian; a Gaussian mixture is used in the reported experiments to model gust-like perturbations. This stresses value estimation rather than control, since the Bellman minimizer is unchanged.

\subsection{Virtual Actuator Coordinates}
\label{subsec:quad-virtual}

Let the error state be
\begin{equation}
s=
[x,y,z,\phi,\theta,\psi,v_x,v_y,v_z,\omega_x,\omega_y,\omega_z]^\top
\in\mathbb{R}^{12}.
\label{eq:quad-state}
\end{equation}
The first three components are the translational position error; \(\phi,\theta,\psi\) are the roll, pitch, and yaw errors; \(v_x,v_y,v_z\) are the linear velocity errors; and \(\omega_x,\omega_y,\omega_z\) are the angular velocity errors. Let \(\bar u=[\bar u_1,\bar u_2,\bar u_3,\bar u_4]^\top\in\mathbb{R}^4\) denote a virtual thrust-deviation input around hover; the first component is a collective thrust perturbation, and the remaining three are roll, pitch, and yaw moment inputs. With sampling time \(\tau>0\), mass \(m>0\), and inertias \(J_x,J_y,J_z>0\), define
\begin{equation}
g_Q(s)=\tau g_{\mathrm{cont}}(s),
\label{eq:gQ}
\end{equation}
where
\begin{equation}
g_{\mathrm{cont}}(s)=
\begin{bmatrix}
0_{6\times 4}\\[0.2em]
-\frac{1}{m}\sin\theta & 0 & 0 & 0\\
\frac{1}{m}\sin\phi\cos\theta & 0 & 0 & 0\\
\frac{1}{m}(\cos\phi\cos\theta-1) & 0 & 0 & 0\\
0 & \frac{1}{J_x} & 0 & 0\\
0 & 0 & \frac{1}{J_y} & 0\\
0 & 0 & 0 & \frac{1}{J_z}
\end{bmatrix}.
\label{eq:gQcont}
\end{equation}
The first six rows are zero because position and attitude errors do not change instantaneously; thrust couples to translational velocity through the attitude-dependent terms in rows seven to nine, and moments couple to angular velocity in rows ten to twelve. The virtual-input dynamics are
\begin{equation}
s_{k+1}=f_p(s_k)+p\,g_Q(s_k)\bar u_k+w_k,\qquad p\in(0,1].
\label{eq:virtual-quad}
\end{equation}
The parameter \(p\) is the control-authority homotopy of Section~\ref{sec:mode2}.

\subsection{Raw Action Interface}
\label{subsec:quad-raw}

The RL agent acts through a raw command \(u\in\mathbb{R}^4\), mapped to the virtual input by
\begin{equation}
\bar u=\phi(u),\qquad \phi_i(u_i)=u_i+\eta u_i^3,\qquad \eta>0.
\label{eq:phi-app}
\end{equation}
Since \(\phi_i'(u_i)=1+3\eta u_i^2>0\), \(\phi\) is a component-wise \(C^2\) diffeomorphism. The map models a stiffening actuator: small commands are nearly linear, while large commands are amplified. The environment seen by the RL agent is
\begin{equation}
s_{k+1}=f_p(s_k)+p\,g_Q(s_k)\phi(u_k)+w_k.
\label{eq:raw-quad}
\end{equation}
For \(\eta>0\), this dynamics is non-affine in the raw action because
\begin{equation}
\frac{\partial^2}{\partial u_i^2}\big(g_Q(s)\phi(u)\big)=6\eta u_i\,g_{Q,i}(s),
\label{eq:nonaffine-second}
\end{equation}
where \(g_{Q,i}(s)\) is the \(i\)th column of \(g_Q(s)\). The parameter \(\eta\) is a difficulty knob: \(\eta\to 0\) recovers the affine QG case of Section~\ref{sec:mode2}, and larger \(\eta\) increases the curvature of the actuator map.

\subsection{Cost Pullback and Oracle Preservation}
\label{subsec:quad-pullback}

The virtual-action quadratic cost is
\begin{equation}
c_{\bar u}(s,\bar u)=s^\top Qs+\bar u^\top R\bar u, \qquad Q\succeq0,\;R\succ0.
\label{eq:virtual-cost}
\end{equation}
The raw-action cost is the pullback
\begin{equation}
c_u(s,u)=s^\top Qs+\phi(u)^\top R\phi(u).
\label{eq:raw-cost}
\end{equation}
Therefore,
\begin{align}
&\min_u
\left\{c_u(s,u)+\gamma\mathbb{E} V^\star(f_p(s)+p g_Q(s)\phi(u)+w)\right\}\nonumber\\
&\qquad = \min_{\bar u} \left\{s^\top Qs+\bar u^\top R\bar u +\gamma\mathbb{E} V^\star(f_p(s)+p g_Q(s)\bar u+w)
\right\}.
\label{eq:equiv-min}
\end{align}
Thus the raw-action MDP is action-coordinate equivalent to the virtual affine MDP. If \(\bar u^\star_p(s)\) is the certified virtual optimal action, then
\begin{equation}
u^\star_p(s)=\phi^{-1}(\bar u^\star_p(s)).
\label{eq:raw-oracle}
\end{equation}

For \(V^\star(s)=s^\top Ps+b\), define
\begin{align}
B_Q^{(p)}(s)&=R+\gamma p^2 g_Q(s)^\top P g_Q(s), \label{eq:BQ}\\
H_Q^{(p)}(s)&=\gamma\Big(P-\gamma p^2Pg_Q(s)(B_Q^{(p)}(s))^{-1}g_Q(s)^\top P\Big).
\label{eq:HQ}
\end{align}
These reduce to \(B(s)\) and \(H_\gamma(s)\) of Section~\ref{sec:mode2} when \(g_Q\) is replaced by \(g\) and \(p=1\). Choose
\begin{equation}
f_p(s)=\big(H_Q^{(p)}(s)\big)^{-1/2}S_p(s)(P-Q)^{1/2}s, \qquad S_p(s)^\top S_p(s)=I.
\label{eq:quad-fp}
\end{equation}
Then
\begin{equation}
f_p(s)^\top H_Q^{(p)}(s)f_p(s)=s^\top(P-Q)s,
\label{eq:quad-energy}
\end{equation}
so the Bellman energy identity holds exactly. The optimal virtual input is
\begin{equation}
\bar u_p^\star(s)= -\gamma p\big(B_Q^{(p)}(s)\big)^{-1}g_Q(s)^\top P f_p(s).
\label{eq:virtual-oracle}
\end{equation}

\subsection{Finite-Covariance Non-Gaussian Gusts}
\label{subsec:quad-noise}

Let \(w_k\) be any zero-mean disturbance with finite covariance \(\Sigma\). For example,
\begin{equation}
w_k\sim (1-\epsilon)\mathcal N(0,\Sigma_1)+\epsilon \mathcal N(0,\Sigma_2), \qquad \Sigma=(1-\epsilon)\Sigma_1+\epsilon\Sigma_2.
\label{eq:mixture-noise}
\end{equation}
The Gaussian mixture is only an illustration; the certificate does not require Gaussianity. Since \(V^\star\) is quadratic,
\begin{equation}
\mathbb{E}[V^\star(z+w)] = z^\top Pz+\operatorname{tr}(P\Sigma)+b.
\label{eq:moment-only}
\end{equation}
The Bellman minimizer therefore depends on the disturbance only through its covariance. The value offset remains
\begin{equation}
b=\frac{\gamma}{1-\gamma}\operatorname{tr}(P\Sigma).
\label{eq:b-mixture}
\end{equation}

The certificate requires the disturbance to be zero-mean, of finite covariance, and independent across time. Heteroscedastic, state-dependent, or temporally correlated noise is not covered by the present construction; extending the certificate to such disturbances would require modifying the energy identity and the value offset accordingly.

\subsection{Reported Results}
\label{subsec:quad-results}

The results in this subsection are reproduced from the source paper~\citep{ibrahim2026benchmarking}; no new training runs were performed for this manuscript. Table~\ref{tab:conversequad} reports the oracle-referenced metrics for PPO, A2C, SAC, TD3, and DDPG on ConverseQuad under CRN, using the same optimality-gap and regret definitions as Section~\ref{subsec:oracle_protocol}. The reward-convention optimality gap is reported; zero is optimal and more negative values indicate larger suboptimality.

\begin{table}[t]
\centering
\footnotesize
\begin{tabular}{lcccc}
\toprule
Algorithm & Optimality gap & Optimality gap 95\% CI & Regret & Regret 95\% CI \\
\midrule
A2C  & $-1.44$              & $[-1.76,\,-1.16]$          & $2.13\times 10^{1}$ & $[1.84\times 10^{1},\,2.47\times 10^{1}]$ \\
DDPG & $-2.18$              & $[-2.53,\,-1.85]$          & $3.35\times 10^{1}$ & $[2.92\times 10^{1},\,3.86\times 10^{1}]$ \\
PPO  & $-5.92\times 10^{-1}$ & $[-7.40\times 10^{-1},\,-4.68\times 10^{-1}]$ & $8.52$ & $[7.47,\,9.66]$ \\
SAC  & $-6.09\times 10^{-1}$ & $[-7.64\times 10^{-1},\,-4.76\times 10^{-1}]$ & $9.03$ & $[7.60,\,1.06\times 10^{1}]$ \\
TD3  & $-7.62\times 10^{-1}$ & $[-9.30\times 10^{-1},\,-6.24\times 10^{-1}]$ & $1.23\times 10^{1}$ & $[1.00\times 10^{1},\,1.51\times 10^{1}]$ \\
\bottomrule
\end{tabular}
\caption{ConverseQuad oracle-referenced metrics under CRN. Reproduced from~\citep{ibrahim2026benchmarking}. Optimality gap uses the reward convention: zero is optimal and more negative values are worse; regret is in cost units, where smaller is better. Intervals are the source's reported \(95\%\) paired-bootstrap intervals.}
\label{tab:conversequad}
\end{table}

ConverseQuad exposes a failure mode distinct from ConverseArm and NVDEx. The certificate is constructed in a virtual thrust/moment input, whereas the agent acts through a nonlinear raw-actuator coordinate. Policy errors are therefore transformed by the actuator map before entering the dynamics, and deterministic actor-critic methods become more exposed to critic-gradient and action-coordinate errors. In the reported results, PPO and SAC form the leading group on both metrics, TD3 is intermediate, and A2C and DDPG are least reliable. The finite-covariance non-Gaussian disturbances increase trajectory and return variability while leaving the oracle unchanged, producing noisier value estimates for learned critics. The ordering should be read as a diagnostic of the raw-actuator interface under the reproduced protocol, not as a universal ranking: on ConverseArm, the ranking is regime-dependent, with TD3 leading under diverse initial states and PPO leading under fixed initial states and on NVDEx.

The remaining question---how the ranking changes as the actuator nonlinearity \(\eta\) and the mixture parameter \(\epsilon\) are varied---is not resolved by the reported results and is left to future work. A systematic sweep over these two knobs would separate the contribution of interface nonlinearity from disturbance non-Gaussianity, both of which are present simultaneously in the current ConverseQuad configuration.

\begin{figure*}[h]
\centering
\vspace{-0.5em}
\begin{minipage}[t]{0.64\textwidth}
  \vspace{35pt}
  \centering
  \includegraphics[width=\linewidth]{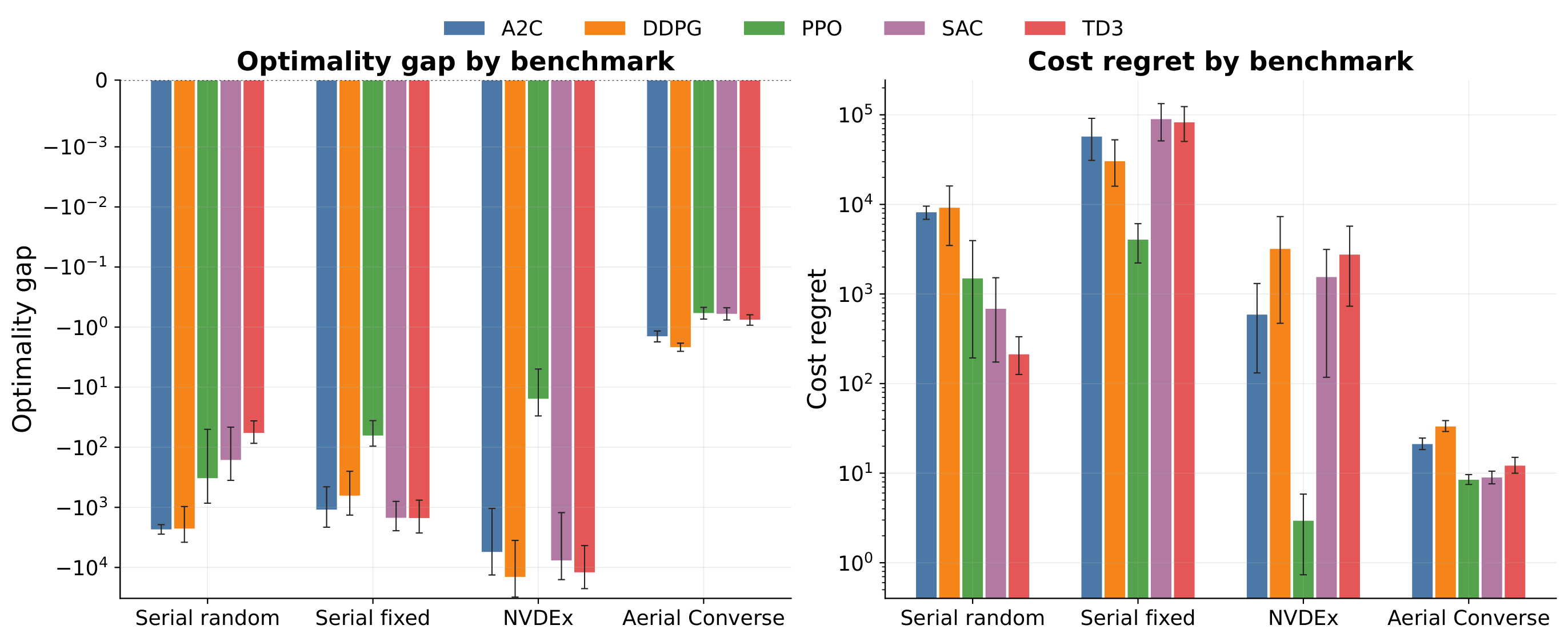}\\[-0.8em]
  {\scriptsize (a)}
\end{minipage}\hfill
\begin{minipage}[t]{0.34\textwidth}
  \vspace{0pt}
  \centering\includegraphics[width=\linewidth]{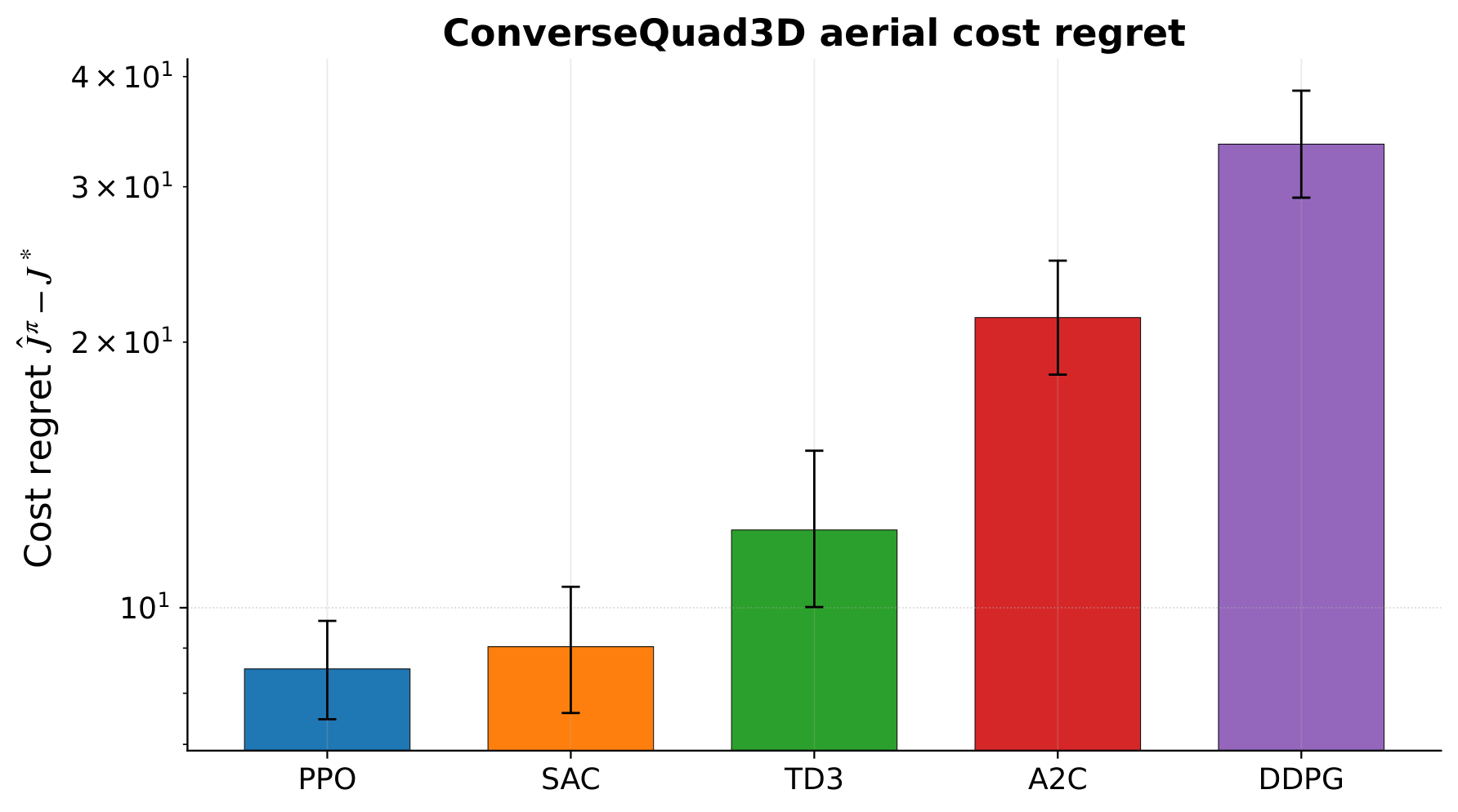}\\[-0.8em]
  {\scriptsize (b)}
  \vspace{0.25em}
  \includegraphics[width=\linewidth]{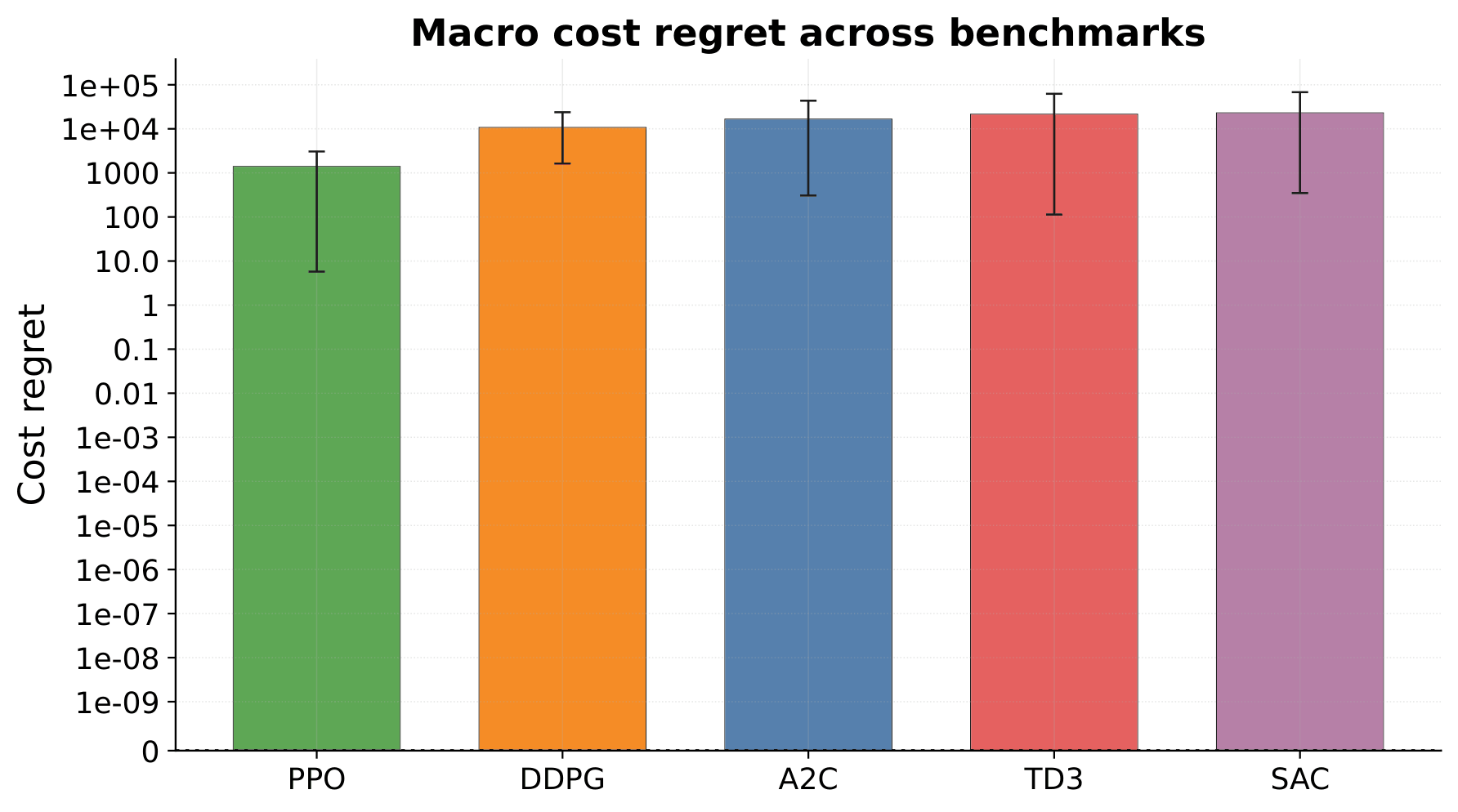}\\[-0.8em]
  {\scriptsize (c)}
\end{minipage}
\vspace{-0.4em}
{\caption{Compact oracle-referenced summaries. 
(a) Per-experiment optimality gap and regret show how algorithm rankings change with occupancy measure, stiffness, and actuation geometry.
(b) \texttt{ConverseQuad} regret under nonlinear raw actuation and non-Gaussian disturbances.
(c) Macro optimality gap; the closer the values are to zero, the better.}}
\label{fig:summary_bars}
\vspace{-0.6em}
\end{figure*}

\paragraph{Forward link.}
The next extension (Section~\ref{sec:ext-nonsmooth}) changes the regularity class of \(V^\star\) itself, rather than the action interface or the disturbance law. The certificate logic is unchanged: the task and its reference pair are constructed together, and the reference is withheld until evaluation.

\section{Extension II: Non-Smooth Converse Optimality}
\label{sec:ext-nonsmooth}

The fixture viewpoint extends to value functions that are merely locally Lipschitz, such as piecewise-quadratic templates with kinks along switching surfaces. This section develops \emph{non-smooth converse optimality}. We replace the smooth HJB certificate of Sections~\ref{sec:mode1}--\ref{sec:ext-nonaffine} with a viscosity--Clarke condition: classical gradients are replaced by Clarke subdifferentials, and the pointwise HJB equation is replaced by its viscosity formulation. The construction is deterministic and continuous-time. A homotopy parameter selects a generalized gradient branch on the switching surface, generating families of systems that share the same value function but differ in the sharpness of the non-smooth behaviour. We work the construction end-to-end for a two-dimensional piecewise-quadratic value function with a straight switching line, and we validate the resulting fixture analytically. 

\subsection{Motivation: why non-smooth values matter}
\label{subsec:nonsmooth-motivation}

Bang-bang control, minimum-time problems, saturated policies, state and control constraints, and hybrid or switching systems all produce optimal value functions that are locally Lipschitz but not differentiable everywhere. When the control enters linearly and the cost does not penalize control quadratically, the optimal policy is typically bang-bang, and the value function has kinks along switching manifolds where the classical gradient does not exist~\citep{clarke2001nonsmooth}. In these regimes, the classical HJB equation is ill-defined, and the smooth converse constructions of Sections~\ref{sec:mode1}--\ref{sec:ext-nonaffine} do not apply.

Without a certified ground truth, a learner's poor performance on such a task cannot be attributed to the non-smoothness of the optimum rather than to approximation or optimization error. Standard RL benchmarks can exhibit non-smooth value functions, but absent a reference they cannot distinguish genuine non-smoothness from overfitting. Non-smooth converse optimality closes this gap: it produces tasks whose optimal value has a known kink and whose optimal policy may be discontinuous, so that a learned policy can be measured against a certified reference in exactly the regime that stresses smoothness assumptions.

\subsection{Background: viscosity solutions and Clarke subdifferentials}
\label{subsec:nonsmooth-background}

\paragraph{Viscosity solutions.}
When the optimal value function is not differentiable, the classical HJB equation is replaced by its viscosity formulation. Consider a first-order Hamilton--Jacobi equation
\[
F(s, V(s), DV(s)) = 0,\quad s\in\Omega,
\]
where \(\Omega\subset\mathbb{R}^{n}\) is open and \(F:\Omega\times\mathbb{R}\times\mathbb{R}^{n}\to\mathbb{R}\) is continuous. A continuous function \(V\) is a \emph{viscosity subsolution} if for every \(C^{1}\) test function \(\phi\) and every local maximum \(s_{0}\) of \(V-\phi\),
\[
F(s_{0}, V(s_{0}), D\phi(s_{0})) \le 0,
\]
and a \emph{viscosity supersolution} if for every local minimum \(s_{0}\) of \(V-\phi\),
\[
F(s_{0}, V(s_{0}), D\phi(s_{0})) \ge 0.
\]
\(V\) is a \emph{viscosity solution} if it is both~\citep{crandall1983viscosity,bardi1997optimal}. For optimal control, the optimal value function is the unique viscosity solution of
\[
\sup_{a\in\mathbb{A}}\big\{-\nabla V^{*}(s)\cdot f(s,a)-c(s,a)\big\}=0,
\]
subject to appropriate boundary conditions. Uniqueness follows from the comparison principle and is what allows the viscosity formulation to certify optimality without classical differentiability.

\paragraph{Clarke subdifferentials.}
For a locally Lipschitz function \(V\), the Clarke generalized directional derivative is
\[
V^{\circ}(s;v) := \limsup_{y\to s,\;t\downarrow 0}\frac{V(y+tv)-V(y)}{t},
\]
and the Clarke subdifferential is
\[
\partial_{C}V(s) := \big\{\zeta\in\mathbb{R}^{n} : V^{\circ}(s;v)\ge\langle\zeta,v\rangle\ \forall v\in\mathbb{R}^{n}\big\}.
\]
The set \(\partial_{C}V(s)\) is nonempty, compact, and convex. It reduces to the singleton \(\{\nabla V(s)\}\) where \(V\) is smooth, and to the convex hull of limiting gradients from adjacent smooth regions at a switching surface. A locally Lipschitz \(V\) is a viscosity solution of the HJB equation if and only if for every \(s\) and every \(p\in\partial_{C}V(s)\) the corresponding HJB inequality holds~\citep{clarke2001nonsmooth}. This equivalence is what makes the viscosity framework constructive: it allows the optimality certificate to be checked pointwise using generalized gradients rather than test functions.

When the optimal policy is discontinuous, the closed-loop dynamics \(\dot{s}=f(s)+g(s)a^{*}(s)\) is interpreted as a differential inclusion in the sense of Filippov~\citep{filippov1988existence}. Solutions exist and are unique under mild regularity, so a discontinuous optimal feedback still defines a well-posed benchmark.

\paragraph{Relation to the smooth case.}
When \(V\) is continuously differentiable, \(\partial_{C}V(s)=\{\nabla V(s)\}\) and the viscosity--Clarke condition collapses to the classical pointwise HJB equation used in Sections~\ref{sec:mode1}--\ref{sec:ext-nonaffine}. The non-smooth formulation is therefore a strict generalization of the smooth converse theory, not a replacement for it.

\subsection{Problem Setup and Assumptions}
\label{subsec:nonsmooth-setup}

Let the state space be $\mathbb{R}^n$ and the control space be $\mathbb{A} \subset \mathbb{R}^m$. We consider control-affine dynamical systems of the form
\begin{equation}
\dot{s}(t) = f(s(t)) + g(s(t))a(t), \quad s(0) = s_0,
\label{eq:nonsmooth-dynamics}
\end{equation}
where $a(\cdot)$ is an admissible control signal taking values in $\mathbb{A}$. We use the cost convention of Section~\ref{sec:fixtures}: the running cost $c : \mathbb{R}^n \times \mathbb{A} \to \mathbb{R}$ is nonnegative, and lower cost is better.

For a stationary policy $\pi : \mathbb{R}^n \to \mathbb{A}$, the infinite-horizon cost functional is
\begin{equation}
V^{\pi}(s_0) = \int_0^{\infty} c\big(s(t), \pi(s(t))\big)\,dt,
\label{eq:nonsmooth-value}
\end{equation}
and the optimal value function is
\begin{equation}
V^{*}(s_0) = \inf_{\pi} V^{\pi}(s_0).
\label{eq:nonsmooth-optimal}
\end{equation}
In the converse optimality setting we instead begin with a prescribed candidate $V : \mathbb{R}^n \to \mathbb{R}$ and aim to construct $(f, g, c)$ such that $V = V^{*}$ and the minimizer of the Hamiltonian is a certified optimal feedback. Throughout this section, $V$ denotes the prescribed candidate, and $V^{*}$ denotes the true optimal value; the goal of the construction is to make them coincide.

\begin{assumption}[System and control structure]\label{as:nonsmooth-system}
\begin{enumerate}[label=(\alph*)]
\item The control set $\mathbb{A} \subset \mathbb{R}^m$ is compact and nonempty.
\item The dynamics $f : \mathbb{R}^n \to \mathbb{R}^n$ and the control matrix field $g : \mathbb{R}^n \to \mathbb{R}^{n \times m}$ are continuous and locally Lipschitz in the state variable: for any compact $S \subset \mathbb{R}^n$ there exist $L_{f,S}, L_{g,S} > 0$ such that for all $s_1, s_2 \in S$,
\[
\|f(s_1) - f(s_2)\| \le L_{f,S} \|s_1 - s_2\|, \qquad
\|g(s_1) - g(s_2)\| \le L_{g,S} \|s_1 - s_2\|.
\]
\item The running cost $c(s,a)$ is continuous, nonnegative, and bounded from below.
\item The control cost $a \mapsto c(s,a)$ is coercive for each $s$, in the sense that $\inf_{a \in \mathbb{A}} c(s,a) \to \infty$ as $\|a\| \to \infty$ whenever $\mathbb{A}$ is unbounded; when $\mathbb{A}$ is compact this condition is automatic.
\end{enumerate}
\end{assumption}

Compactness of the control set guarantees the existence of minimizing controls in the pointwise optimization problems arising in the HJB equation. Lipschitz continuity of the dynamics ensures existence and uniqueness of trajectories for admissible controls. Coercivity of the control cost is used in Lemma~\ref{lem:nonsmooth-first-order} to ensure that the pointwise minimizer is unique and that the stationarity condition is both necessary and sufficient. These assumptions are standard in dynamic programming and HJB theory and are required for the validity of the dynamic programming principle.

\begin{assumption}[Value function regularity]\label{as:nonsmooth-value}
The candidate value function $V : \mathbb{R}^n \to \mathbb{R}$ satisfies:
\begin{enumerate}[label=(\alph*)]
\item $V$ is continuous and locally Lipschitz.
\item $V$ is therefore differentiable almost everywhere (Rademacher), and at points of non-differentiability the Clarke subdifferential $\partial_{C}V(s)$ exists and is nonempty.
\item For the candidate optimal feedback $a^{*}(s)$ constructed in this section, the corresponding closed-loop trajectory $s(\cdot)$ satisfies
\[
\int_0^{\infty} c\big(s(t), a^{*}(s(t))\big)\,dt < +\infty.
\]
\end{enumerate}
\end{assumption}

Local Lipschitz continuity is the natural regularity class for value functions arising in optimal control: it guarantees almost-everywhere differentiability, and it is preserved under the operations that appear in the Bellman equation. The Clarke subdifferential then provides a generalized gradient suitable for non-smooth analysis and yields HJB conditions consistent with viscosity solution theory (Section~\ref{subsec:nonsmooth-background}). We do not require $V$ to be convex or piecewise quadratic; those are special cases used for the worked example in Section~\ref{subsec:nonsmooth-example}.

\begin{definition}[Viscosity HJB condition)]\label{def:nonsmooth-viscosity}
The triple $(f, g, c)$ realizes $V$ as an optimal value function if the following conditions hold:
\begin{enumerate}[label=(\roman*)]
\item \textbf{(Subsolution)} For every $s \in \mathbb{R}^n$ there exists $p \in \partial_{C}V(s)$ such that
\[
\inf_{a \in \mathbb{A}} \big\{ c(s,a) + p \cdot (f(s) + g(s)a) \big\} \le 0.
\]
\item \textbf{(Supersolution)} For every $s \in \mathbb{R}^n$ and every $p \in \partial_{C}V(s)$,
\[
\inf_{a \in \mathbb{A}} \big\{ c(s,a) + p \cdot (f(s) + g(s)a) \big\} \ge 0.
\]
\item \textbf{(Measurable optimal feedback)} There exists a measurable function $a^{*} : \mathbb{R}^n \to \mathbb{A}$ and a measurable selection $p(s) \in \partial_{C}V(s)$ such that for almost every $s$,
\[
a^{*}(s) \in \arg\min_{a \in \mathbb{A}} \big\{ c(s,a) + p(s) \cdot (f(s) + g(s)a) \big\},
\]
and the minimum value equals zero for almost every $s$.
\end{enumerate}
\end{definition}

Definition~\ref{def:nonsmooth-viscosity} is the non-smooth analogue of the Bellman residual certificate of Section~\ref{sec:fixtures}. In the smooth case, the residual $\mathcal{R}_{V}(s,a) \ge 0$ for all $a$ and $\mathcal{R}_{V}(s,\pi(s)) = 0$ are replaced by the subsolution and supersolution inequalities above, evaluated on elements of the Clarke subdifferential. If the viscosity HJB condition holds, then $V$ is a viscosity solution of the Hamilton--Jacobi--Bellman equation associated with the control problem~\citep{bardi1997optimal,fleming2006controlled}, and by uniqueness of viscosity solutions $V = V^{*}$. The feedback $a^{*}$ in condition~(iii) is then an optimal policy. Well-posedness of the closed-loop system when $a^{*}$ is discontinuous is established separately in Lemma~\ref{lem:nonsmooth-measurable}.

Given a non-smooth candidate $V$ satisfying Assumption~2, three questions arise, and they structure the rest of the section: does there exist a triple $(f,g,c)$ meeting the viscosity condition (Section~\ref{subsec:nonsmooth-theory}); can explicit formulas for $(f,g,c)$ be given in terms of $V$ and its Clarke subdifferential (Section~\ref{subsec:nonsmooth-procedure}); and can the construction be parameterized to produce families of systems that all share the same $V$ (Section~\ref{subsec:nonsmooth-example})?

When $V$ is $C^{1}$, $\partial_{C}V(s) = \{\nabla V(s)\}$ and the viscosity condition collapses to the classical pointwise HJB equation. The non-smooth formulation is therefore a strict generalization of the smooth converse theory of Sections~\ref{sec:mode1}--\ref{sec:ext-nonaffine} and of the discrete-time constructions in~\citep{nevistic1996constrained,gohrt2019converse}.

\subsection{Theoretical Results}
\label{subsec:nonsmooth-theory}

We now establish the theoretical foundations of non-smooth converse optimality for deterministic control-affine systems. Throughout, we assume that the candidate value function \(V\) satisfies Assumption~\ref{as:nonsmooth-value} (locally Lipschitz, with Clarke subdifferential \(\partial_{C}V(s)\) defined everywhere), and that the system components \((f,g,c)\) satisfy Assumption~\ref{as:nonsmooth-system}. Full proofs are given in Appendix~\ref{app:nonsmooth-proofs}; we state the results and give proof sketches here.

\begin{lemma}[First-order optimality condition]
\label{lem:nonsmooth-first-order}
Let Assumptions~\ref{as:nonsmooth-system}--\ref{as:nonsmooth-value} hold, and let the stage cost be \(c(s,a)=q(s)+\rho(a)\) with \(\rho\in C^{2}\) strictly convex. Assume that the action space is unconstrained, \(\mathbb{A}=\mathbb{R}^{m}\), and that \(V\) is a viscosity solution of the HJB equation. Suppose there exist measurable functions \(a^{*}:\mathbb{R}^{n}\to\mathbb{R}^{m}\) and \(p:\mathbb{R}^{n}\to\mathbb{R}^{n}\) with \(p(s)\in\partial_{C}V(s)\) such that the minimum in the viscosity HJB condition of Definition~\ref{def:nonsmooth-viscosity} is attained and equals zero for almost every \(s\). Then, for almost every \(s\), the first-order condition for the pointwise minimization over \(a\) holds:
\begin{equation}
\nabla_{a}\rho(a)\big|_{a=a^{*}(s)} + g(s)^{\top} p(s) = 0.
\label{eq:nonsmooth-first-order}
\end{equation}
\end{lemma}

\emph{Proof sketch.} For a fixed \(s\) where the minimum is attained, the function \(a\mapsto \rho(a)+p(s)\cdot g(s)a\) is strictly convex and coercive on \(\mathbb{R}^{m}\), so its unique minimizer \(a^{*}(s)\) satisfies the unconstrained first-order condition. The state cost \(q(s)\) and the drift term \(p(s)\cdot f(s)\) do not depend on \(a\). The full argument is in Appendix~\ref{app:nonsmooth-proofs}. \(\square\)

\paragraph{Remark.} When \(\mathbb{A}\) is a compact subset of \(\mathbb{R}^{m}\) and the minimizer lies on the boundary, condition~\eqref{eq:nonsmooth-first-order} is replaced by a variational inequality. The stationarity conditions in Lemma~\ref{lem:nonsmooth-sufficient} and Theorem~\ref{thm:nonsmooth-main} below are therefore stated under the convention that (i) is either satisfied by an interior minimizer or holds with equality of directional derivatives on the feasible cone at \(a^{*}(s)\).

\begin{lemma}[Sufficient condition for optimality]
\label{lem:nonsmooth-sufficient}
Let Assumptions~\ref{as:nonsmooth-system}--\ref{as:nonsmooth-value} hold. Let \(V\) be locally Lipschitz, and suppose there exist measurable functions \(a^{*}:\mathbb{R}^{n}\to\mathbb{A}\) and \(p:\mathbb{R}^{n}\to\mathbb{R}^{n}\) with \(p(s)\in\partial_{C}V(s)\) such that for almost every \(s\in\mathbb{R}^{n}\):
\begin{enumerate}[label=(\roman*)]
\item \textbf{(Stationarity)} \(\nabla_{a}\rho(a)|_{a=a^{*}(s)} + g(s)^{\top} p(s) = 0\), understood in the variational sense of the remark above when the minimizer is on the boundary;
\item \textbf{(Hamiltonian optimality)} \(a^{*}(s)\) attains the infimum
\[
\inf_{a\in\mathbb{A}}\big\{c(s,a)+p(s)\cdot(f(s)+g(s)a)\big\}=0;
\]
\item \textbf{(Viscosity supersolution)} for every \(\tilde{p}\in\partial_{C}V(s)\),
\[
\inf_{a\in\mathbb{A}}\big\{c(s,a)+\tilde{p}\cdot(f(s)+g(s)a)\big\}\ge 0.
\]
\end{enumerate}
Then \(V\) is the optimal value function for the control problem with dynamics \(\dot{s}=f(s)+g(s)a\) and running cost \(c\), and \(a^{*}\) is an optimal feedback policy.
\end{lemma}

\emph{Proof sketch.} Conditions (ii) and (iii) state that \(V\) satisfies the viscosity HJB condition of Definition~\ref{def:nonsmooth-viscosity}. By uniqueness of viscosity solutions, \(V\) coincides with the optimal value function. Condition (i) identifies \(a^{*}(s)\) as the minimizer of the Hamiltonian for the selected generalized gradient. The full argument is in Appendix~\ref{app:nonsmooth-proofs}. \(\square\)

\paragraph{Remark (Recovering the smooth case).} When \(V\) is continuously differentiable, \(\partial_{C}V(s)=\{\nabla V(s)\}\) and the viscosity conditions reduce to the classical HJB equation. Lemma~\ref{lem:nonsmooth-first-order} then recovers the standard first-order optimality condition, and Lemma~\ref{lem:nonsmooth-sufficient} becomes the classical verification theorem. The non-smooth formulation is therefore a strict generalization of the smooth case, not a parallel theory.

\begin{lemma}[Measurability of the optimal feedback]
\label{lem:nonsmooth-measurable}
Let Assumptions~\ref{as:nonsmooth-system}--\ref{as:nonsmooth-value} hold, and suppose a measurable selection \(p(s)\in\partial_{C}V(s)\) is given.
\begin{enumerate}[label=(\roman*)]
\item The set-valued map
\[
\Phi(s) := \arg\min_{a\in\mathbb{A}}\big\{c(s,a)+p(s)\cdot(f(s)+g(s)a)\big\}
\]
is nonempty, compact-valued, and weakly measurable; equivalently, its graph is a Borel subset of \(\mathbb{R}^{n}\times\mathbb{A}\).
\item \(\Phi\) admits a measurable selection \(a^{*}:\mathbb{R}^{n}\to\mathbb{A}\), i.e., a Lebesgue-measurable function with \(a^{*}(s)\in\Phi(s)\) for almost every \(s\).
\item The closed-loop vector field \(\dot{s}=f(s)+g(s)a^{*}(s)\) is measurable and locally bounded, and therefore admits Filippov solutions~\citep{filippov1988existence} from any initial condition.
\end{enumerate}
\end{lemma}

\emph{Proof sketch.} Since \(c\), \(f\), \(g\), and \(p\) are measurable in \(s\) and continuous in \(a\), the integrand is a Carathéodory function. The control set \(\mathbb{A}\) is compact, so the argmin correspondence is nonempty and compact-valued, and its graph is Borel. The Kuratowski--Ryll-Nardzewski selection theorem then gives a Borel-measurable selector, and Filippov's existence theorem gives solutions of the closed loop. The full argument is in Appendix~\ref{app:nonsmooth-proofs}. \(\square\)

\begin{theorem}[Non-smooth continuous-time converse optimality]
\label{thm:nonsmooth-main}
Let Assumptions~\ref{as:nonsmooth-system}--\ref{as:nonsmooth-value} hold, and let an input-coupling matrix \(g:\mathbb{R}^{n}\to\mathbb{R}^{n\times m}\), a stage cost \(c(s,a)=q(s)+\rho(a)\) with \(q\) continuous and \(\rho\in C^{2}\) strictly convex, and a candidate locally Lipschitz value function \(V:\mathbb{R}^{n}\to\mathbb{R}\) be given. The following statements are equivalent:
\begin{enumerate}[label=(\alph*)]
\item There exists a drift \(f:\mathbb{R}^{n}\to\mathbb{R}^{n}\) such that the triple \((f,g,c)\) realizes \(V\) as the optimal value function for the control-affine system \(\dot{s}=f(s)+g(s)a\).
\item There exist a drift \(f:\mathbb{R}^{n}\to\mathbb{R}^{n}\), a measurable policy \(a^{*}:\mathbb{R}^{n}\to\mathbb{A}\), and a measurable selection \(p:\mathbb{R}^{n}\to\mathbb{R}^{n}\) with \(p(s)\in\partial_{C}V(s)\) such that for almost every \(s\in\mathbb{R}^{n}\):
\begin{enumerate}[label=(\roman*)]
\item \(\nabla_{a}\rho(a)|_{a=a^{*}(s)} + g(s)^{\top} p(s)=0\);
\item \(\inf_{a\in\mathbb{A}}\{c(s,a)+p(s)\cdot(f(s)+g(s)a)\}=0\), attained at \(a^{*}(s)\);
\item for every \(\tilde{p}\in\partial_{C}V(s)\), \(\inf_{a\in\mathbb{A}}\{c(s,a)+\tilde{p}\cdot(f(s)+g(s)a)\}\ge 0\).
\end{enumerate}
\end{enumerate}
Whenever (b) holds, the policy \(a^{*}\) is optimal and \(V\) is the optimal value function. Conversely, whenever (a) holds, a triple satisfying (b) exists.
\end{theorem}

\emph{Proof sketch.} (b) \(\Rightarrow\) (a) is Lemma~\ref{lem:nonsmooth-sufficient}. (a) \(\Rightarrow\) (b) uses the fact that the true optimal value function \(V^{*}=V\) is a viscosity solution of the HJB equation, so the viscosity conditions hold with some measurable selection \(p(s)\in\partial_{C}V(s)\); the stationarity condition is then Lemma~\ref{lem:nonsmooth-first-order}, and measurability of the selection is Lemma~\ref{lem:nonsmooth-measurable}. The full proof is in Appendix~\ref{app:nonsmooth-proofs}. \(\square\)

\paragraph{Role of Lemma~\ref{lem:nonsmooth-measurable}.} Although Lemma~\ref{lem:nonsmooth-measurable} does not appear explicitly in the statement of Theorem~\ref{thm:nonsmooth-main}, it is what makes the converse framework well-posed. Definition~\ref{def:nonsmooth-viscosity} and Lemma~\ref{lem:nonsmooth-sufficient} both require a measurable optimal feedback, and Lemma~\ref{lem:nonsmooth-measurable} guarantees that one exists and that the resulting closed loop is a Filippov system. This is what allows the generated benchmark to be simulated even when the optimal policy is discontinuous.

The theorem is constructive. Section~\ref{subsec:nonsmooth-procedure} turns it into a six-step procedure for a given candidate \(V\), and Section~\ref{subsec:nonsmooth-example} instantiates the procedure for a two-dimensional piecewise-quadratic value function.

\subsection{General Constructive Procedure}
\label{subsec:nonsmooth-procedure}

We now turn Theorem~\ref{thm:nonsmooth-main} into a systematic procedure that, given a candidate locally Lipschitz value function \(V\), produces a triple \((f,g,c)\) for which \(V\) is optimal in the viscosity sense. The procedure makes no assumption about the specific form of the control cost \(\rho\) beyond strict convexity and \(C^{2}\) regularity; the quadratic case is treated in Section~\ref{subsec:nonsmooth-quadratic} and instantiated in Section~\ref{subsec:nonsmooth-example}.

The procedure uses Assumptions~\ref{as:nonsmooth-system}--\ref{as:nonsmooth-value}. The input-coupling matrix \(g:\mathbb{R}^{n}\to\mathbb{R}^{n\times m}\) is prescribed and satisfies Assumption~\ref{as:nonsmooth-system}(b). The control set \(\mathbb{A}\) is compact and nonempty, or \(\mathbb{A}=\mathbb{R}^{m}\) in the unconstrained case. The candidate \(V:\mathbb{R}^{n}\to\mathbb{R}\) is locally Lipschitz and its Clarke subdifferential \(\partial_{C}V(s)\) is explicitly characterized. Among the six steps below, Steps~1--3 are design choices: they depend on \(V\), on the geometry of the switching surfaces, and on a gain matrix \(K\) that shapes the uncontrolled dynamics. Steps~4--6 are then determined: the state cost, the supersolution check, and the closed-loop well-posedness follow from the design choices.

\begin{enumerate}[label=Step~\arabic*., ref=Step~\arabic*, leftmargin=*]

\item \textbf{Subdifferential analysis.}
Compute the Clarke subdifferential \(\partial_{C}V(s)\) for every \(s\in\mathbb{R}^{n}\). Identify:
\begin{itemize}
\item \emph{smooth regions}---open sets on which \(\partial_{C}V(s)=\{\nabla V(s)\}\) is a singleton;
\item \emph{switching surfaces}---lower-dimensional manifolds on which \(V\) is non-differentiable and \(\partial_{C}V(s)\) has more than one element.
\end{itemize}
On each switching surface, express \(\partial_{C}V(s)\) as the convex hull of finitely many limiting gradients from adjacent smooth regions. Introduce local parameters \(\alpha\in[0,1]^{k}\) (a single scalar \(\alpha\) for a codimension-one surface) that parametrize this convex hull. These parameters are homotopy knobs: different values of \(\alpha\) produce different admissible systems that share the same value function \(V\).

\item \textbf{Selection of a generalized gradient.}
Choose a measurable selection \(p:\mathbb{R}^{n}\to\mathbb{R}^{n}\) with \(p(s)\in\partial_{C}V(s)\) for every \(s\). On smooth regions the selection is unique, \(p(s)=\nabla V(s)\). On switching surfaces the selection may depend on the parameters \(\alpha\) introduced in Step~1. Different selections yield different admissible systems but leave \(V\) unchanged.

\item \textbf{Construction of the drift.}
Choose a locally Lipschitz, positive definite gain matrix \(K(s)\succ 0\), and set
\begin{equation}
f(s) := -K(s)\,p(s).
\label{eq:driftK}
\end{equation}
The negative sign makes \(V\) decrease along the uncontrolled dynamics: since \(p(s)\cdot f(s)=-p(s)^{\top}K(s)p(s)\le 0\), the drift performs generalized steepest descent in the \(K\)-weighted norm. Other choices of \(f\) satisfying the Hamiltonian optimality condition are possible; the discrete-time constructions of Section~\ref{sec:mode2} use square-root and metric-normalized parameterizations that also satisfy the corresponding energy identity.

\item \textbf{Determination of the state cost.}
With the drift fixed by \eqref{eq:driftK}, the optimal feedback \(a^{*}\) is determined by the stationarity condition of Lemma~\ref{lem:nonsmooth-first-order}. For general strictly convex \(\rho\in C^{2}\), the equation
\begin{equation}
\nabla_{a}\rho(a)\big|_{a=a^{*}(s)} + g(s)^{\top} p(s) = 0
\label{eq:stat_general}
\end{equation}
implicitly defines \(a^{*}(s)\). Because \(\rho\) is strictly convex and coercive, the solution exists, is unique, and is measurable in \(s\) by Lemma~\ref{lem:nonsmooth-measurable}. The state cost \(q(s)\) is then determined by the Hamiltonian optimality condition (condition~(ii) of Theorem~\ref{thm:nonsmooth-main}): for the chosen \(f\) and \(a^{*}\),
\begin{equation}
0 = q(s) + \rho(a^{*}(s)) + p(s)\cdot\big(f(s)+g(s)a^{*}(s)\big),
\label{eq:ham_zero_general}
\end{equation}
which gives
\begin{equation}
q(s) = -\rho(a^{*}(s)) - p(s)\cdot\big(f(s)+g(s)a^{*}(s)\big).
\label{eq:q_general}
\end{equation}
With \(f(s)=-K(s)p(s)\), the drift term \(p(s)\cdot f(s)\) is non-positive. The sign of the remaining term depends on the control cost; \(q(s)\) is computed from \eqref{eq:q_general} and checked for nonnegativity. In the quadratic case this check is automatic, as shown in Section~\ref{subsec:nonsmooth-quadratic}.

\item \textbf{Verification of the viscosity conditions.}
With \(f\), \(a^{*}\), and \(q\) determined as above, verify the viscosity HJB condition of Definition~\ref{def:nonsmooth-viscosity}.
\begin{itemize}
\item \emph{Subsolution.} By construction, for the selected \(p(s)\in\partial_{C}V(s)\),
\[
\inf_{a\in\mathbb{A}}\big\{c(s,a)+p(s)\cdot(f(s)+g(s)a)\big\}=0,
\]
and \(a^{*}(s)\) attains the infimum. This follows directly from \eqref{eq:stat_general} and the definition of \(q(s)\).

\item \emph{Supersolution.} For every \(\tilde{p}\in\partial_{C}V(s)\),
\[
\inf_{a\in\mathbb{A}}\big\{c(s,a)+\tilde{p}\cdot(f(s)+g(s)a)\big\}\ge 0.
\]
This inequality is not automatic. It couples the choice of \(K(s)\) to the geometry of \(\partial_{C}V(s)\) and to the input map \(g(s)\). It must be checked for the specific \(V\) and the chosen \(K(s)\). In the quadratic case, the check reduces to a lower bound on \(K\); this is derived in Section~\ref{subsec:nonsmooth-quadratic} and used in the worked example of Section~\ref{subsec:nonsmooth-example}. If the supersolution condition fails, one may enlarge \(K(s)\) and recompute \(q(s)\); Theorem~\ref{thm:nonsmooth-main} guarantees that a sufficiently large \(K(s)\) exists whenever the coefficient of \(K\) in the supersolution inequality is positive.
\end{itemize}

\item \textbf{Well-posedness of the closed loop.}
The optimal feedback \(a^{*}(s)\) may be discontinuous on switching surfaces. The closed-loop dynamics \(\dot{s}=f(s)+g(s)a^{*}(s)\) is therefore interpreted as a differential inclusion in the Filippov sense~\citep{filippov1988existence}. By Lemma~\ref{lem:nonsmooth-measurable}, the right-hand side is measurable and locally bounded, so absolutely continuous Filippov solutions exist from any initial condition. This is what allows the generated fixture to be simulated even when the optimal policy jumps.

\end{enumerate}

The procedure yields a triple \((f,g,c)\) satisfying all hypotheses of Theorem~\ref{thm:nonsmooth-main}, which certifies that \(V\) is the optimal value function and \(a^{*}\) is an optimal feedback. Sections~\ref{subsec:nonsmooth-quadratic}--\ref{subsec:nonsmooth-example} specialize the procedure to quadratic control costs and work the construction end-to-end for a two-dimensional piecewise-quadratic value function.

\subsection{Construction with a Quadratic Control Cost}
\label{subsec:nonsmooth-quadratic}

We now specialize the general procedure of Section~\ref{subsec:nonsmooth-procedure} to the quadratic control cost used in the benchmark examples. Let
\begin{equation}
\rho(a) = \tfrac12\,a^{\top} R\, a,\qquad R\succ 0,
\label{eq:quadratic_rho}
\end{equation}
and assume the control space is unconstrained, \(\mathbb{A}=\mathbb{R}^{m}\). The compact case is discussed in the remark at the end of this subsection. The drift is chosen as in the general procedure,
\begin{equation}
f(s) = -K(s)\,p(s),\qquad K(s)\succ 0,
\label{eq:f_quad}
\end{equation}
where \(p(s)\in\partial_{C}V(s)\) is the measurable selection of Step~2 in Section~\ref{subsec:nonsmooth-procedure}. Because \(\rho\) is quadratic, the stationarity condition \(Ra+g(s)^{\top}p(s)=0\) solves explicitly for the optimal action:
\begin{equation}
a^{*}(s) = -R^{-1}g(s)^{\top}p(s).
\label{eq:astar_quad}
\end{equation}
Substituting \eqref{eq:f_quad} and \eqref{eq:astar_quad} into the Hamiltonian optimality condition \eqref{eq:ham_zero_general} and solving for \(q(s)\) gives
\begin{align}
q(s) &= -\tfrac12\,a^{*}(s)^{\top}R\,a^{*}(s) - p(s)\cdot\bigl(f(s)+g(s)a^{*}(s)\bigr)\notag\\
&= \tfrac12\,\bigl\|g(s)^{\top}p(s)\bigr\|_{R^{-1}}^{2} + \bigl\|p(s)\bigr\|_{K(s)}^{2}.
\label{eq:q_quad}
\end{align}
The right-hand side is manifestly nonnegative because \(R^{-1}\succ 0\) and \(K(s)\succ 0\). The state cost therefore satisfies the positivity requirement automatically.

\paragraph{Automatic satisfaction of the first two conditions.}
With the choices \eqref{eq:f_quad}, \eqref{eq:astar_quad}, and \eqref{eq:q_quad}, the stationarity condition (condition~(i) of Lemma~\ref{lem:nonsmooth-sufficient}) holds by construction, and the Hamiltonian optimality condition (condition~(ii)) holds because, for the selected \(p(s)\),
\[
\inf_{a\in\mathbb{R}^{m}}\bigl\{c(s,a)+p(s)\cdot(f(s)+g(s)a)\bigr\}
= c(s,a^{*}(s))+p(s)\cdot(f(s)+g(s)a^{*}(s))=0,
\]
as a direct algebraic verification confirms. The only remaining requirement is the supersolution inequality (condition~(iii) of Lemma~\ref{lem:nonsmooth-sufficient}).

\paragraph{Supersolution condition.}
For fixed \(s\) and \(\tilde p\in\partial_{C}V(s)\), define
\[
\kappa_{\tilde p}(a) := q(s) + \tfrac12\,a^{\top}R\,a + \tilde p\cdot(f(s)+g(s)a).
\]
This expression is strictly convex and coercive in \(a\), so it attains a unique minimum at
\[
a_{\tilde p}^{*}(s) := -R^{-1}g(s)^{\top}\tilde p.
\]
Substituting this minimizer and using \eqref{eq:f_quad} and \eqref{eq:q_quad} gives
\begin{equation}
\inf_{a\in\mathbb{R}^{m}}\kappa_{\tilde p}(a)
= \tfrac12\,\bigl\|g(s)^{\top}p(s)\bigr\|_{R^{-1}}^{2}
-\tfrac12\,\bigl\|g(s)^{\top}\tilde p\bigr\|_{R^{-1}}^{2}
+\bigl\|p(s)\bigr\|_{K(s)}^{2}
-\langle \tilde p, p(s)\rangle_{K(s)}.
\label{eq:supersolution-explicit}
\end{equation}
The supersolution condition is \eqref{eq:supersolution-explicit} \(\ge 0\) for every \(\tilde p\in\partial_{C}V(s)\).

\paragraph{The case \(\tilde p = p(s)\).}
At \(\tilde p = p(s)\), the right-hand side of \eqref{eq:supersolution-explicit} equals
\[
\tfrac12\|g^{\top}p\|_{R^{-1}}^{2} - \tfrac12\|g^{\top}p\|_{R^{-1}}^{2} + \|p\|_{K}^{2} - \langle p,p\rangle_{K} = 0.
\]
The condition is therefore trivially satisfied at the selected gradient. The nontrivial constraint comes from the \emph{other} elements of the Clarke subdifferential, i.e., from \(\tilde p\in\partial_{C}V(s)\setminus\{p(s)\}\). This is the analogue of the fact that, in the smooth case, the Bellman residual is identically zero at the optimal action and nonnegative at all others.

\paragraph{Scalar gain.}
Let \(K(s)=kI\) with \(k>0\). Then \eqref{eq:supersolution-explicit} becomes
\begin{equation}
k\bigl(\|p\|^{2}-\tilde p\cdot p\bigr)
\;\ge\;
\tfrac12\,\bigl\|g^{\top}\tilde p\bigr\|_{R^{-1}}^{2}
-\tfrac12\,\bigl\|g^{\top}p\bigr\|_{R^{-1}}^{2}
\qquad \forall\,\tilde p\in\partial_{C}V(s)\setminus\{p\}.
\label{eq:k-inequality}
\end{equation}
The right-hand side does not depend on \(k\). Suppose
\begin{equation}
c_{\min}(s) := \min_{\tilde p\in\partial_{C}V(s)\setminus\{p(s)\}}\bigl\{\|p(s)\|^{2}-\tilde p\cdot p(s)\bigr\} > 0.
\label{eq:cmin-positive}
\end{equation}
Then the coefficient of \(k\) is positive on the nontrivial part of the subdifferential, and \eqref{eq:k-inequality} holds if and only if
\begin{equation}
k \;\ge\; k_{\min}(s) :=
\frac{\displaystyle
\max_{\tilde p\in\partial_{C}V(s)\setminus\{p(s)\}}\bigl\{\tfrac12\|g(s)^{\top}\tilde p\|_{R^{-1}}^{2}\bigr\}
-\tfrac12\|g(s)^{\top}p(s)\|_{R^{-1}}^{2}}
{\displaystyle c_{\min}(s)}.
\label{eq:kmin}
\end{equation}
The numerator is finite because \(\partial_{C}V(s)\) is compact. If \eqref{eq:cmin-positive} fails for the chosen \(s\), then either the selected generalized gradient \(p(s)\) must be replaced by another element of \(\partial_{C}V(s)\) for which the denominator is positive, or a state-dependent gain \(K(s)\) must be used. A state-dependent \(K(s)\) can be chosen so that the quadratic form in \(\langle\tilde p,p\rangle_{K}\) dominates the \(R^{-1}\)-weighted difference in \eqref{eq:supersolution-explicit}; the scalar bound \eqref{eq:kmin} is the worst-case special case.

\paragraph{Summary and forward reference.}
The construction yields a certified triple \((f,g,c)\) once the supersolution lower bound \eqref{eq:kmin} is satisfied. In smooth regions, \(c_{\min}(s)\) is undefined (the subdifferential is a singleton), and the supersolution condition holds trivially. The constraint is therefore active only on switching surfaces. Section~\ref{subsec:nonsmooth-example} instantiates the bound \eqref{eq:kmin} for the two-dimensional piecewise-quadratic family, where \(c_{\min}(s) = 8(1-\alpha^{2})y^{2}>0\) on the switching line for every \(\alpha<1\) and every \(y\neq 0\).

\paragraph{Remark (compact control sets).}
When \(\mathbb{A}\subset\mathbb{R}^{m}\) is compact and the unconstrained minimizer \(a^{*}(s)=-R^{-1}g(s)^{\top}p(s)\) lies outside \(\mathbb{A}\), the optimal action is the projection of \(a^{*}(s)\) onto \(\mathbb{A}\), and the stationarity condition is replaced by a variational inequality on the normal cone of \(\mathbb{A}\) at the projected point. The remaining construction is unchanged: the supersolution inequality must still be checked for every \(\tilde p\in\partial_{C}V(s)\), and the lower bound on \(K\) is obtained from the same computation with the projected action. The compact case is not needed for the benchmark examples in this paper and is left to future work.

\subsection{Remarks on Problem Geometry and Multiple Switching Surfaces}
\label{subsec:nonsmooth-remarks}

\paragraph{Geometric interpretation via Lie derivatives.}
In continuous-time control-affine systems, the quantity \(g(s)^{\top}p(s)\) appearing in the optimal feedback \eqref{eq:astar_quad} admits a natural geometric interpretation: it is a vector of Lie derivatives of \(V\) along the columns of the input-coupling matrix. Let \(g_{1}(s),\ldots,g_{m}(s)\) be the columns of \(g(s)\). Each \(g_{i}\) describes the instantaneous direction in which the state moves when the \(i\)-th control input is applied. For a smooth \(V\), the Lie derivative of \(V\) along \(g_{i}\) is the directional derivative
\[
\mathcal{L}_{g_{i}}V(s) = \nabla V(s)\cdot g_{i}(s),
\]
and the vector \(\mathcal{L}_{g}V(s) = (\mathcal{L}_{g_{1}}V(s),\ldots,\mathcal{L}_{g_{m}}V(s))^{\top}\) collects the rates of change of \(V\) in all actuated directions. In smooth continuous-time converse optimality, the optimal feedback is
\begin{equation}
a^{*}(s) = -R^{-1}\mathcal{L}_{g}V(s) = -R^{-1}g(s)^{\top}\nabla V(s),
\label{eq:smooth-lie-feedback}
\end{equation}
see \citep{nevistic1996constrained,yamerenko2025generating}. The right-hand side is the steepest descent of \(V\) projected onto the subspace spanned by the columns of \(g(s)\), weighted by the inverse of the control-cost matrix \(R^{-1}\). The same structure appears in the discrete-time construction of Section~\ref{sec:mode2}, where the optimal feedback \(a^{*}(s) = -\gamma B(s)^{-1}g(s)^{\top}Pf(s)\) is the \(R^{-1}\)-weighted projection of the value gradient onto the actuated subspace, with the metric \(B(s) = R + \gamma g(s)^{\top}Pg(s)\) replacing \(R\).

\paragraph{The Clarke extension.}
With a measurable selection \(p(s)\in\partial_{C}V(s)\), the scalar \(p(s)\cdot g_{i}(s)\) is a selected Lie derivative drawn from the set
\[
\mathcal{L}_{g_{i}}^{\mathrm{Clarke}} V(s) := \{\zeta \cdot g_{i}(s) : \zeta \in \partial_{C}V(s)\}.
\]
Because \(\partial_{C}V(s)\) is compact and convex and \(\zeta \mapsto \zeta \cdot g_{i}(s)\) is a linear functional, this set is a closed bounded interval in \(\mathbb{R}\). For a fixed selection \(p(s)\), the \(i\)-th coordinate of the vector \(g(s)^{\top}p(s)\) is an element of \(\mathcal{L}_{g_{i}}^{\mathrm{Clarke}} V(s)\). The optimal feedback
\[
a^{*}(s) = -R^{-1}g(s)^{\top}p(s)
\]
therefore steers the system opposite to the chosen projection of the generalized gradient onto the control subspace. On smooth regions the selection is unique and coincides with \(\nabla V(s)\); on a switching surface, \(\partial_{C}V(s)\) is the convex hull of limiting gradients from adjacent smooth regions, so each \(\mathcal{L}_{g_{i}}^{\mathrm{Clarke}} V(s)\) is the interval spanned by the corresponding classical Lie derivatives.\\

\emph{Worked example.}
For the piecewise-quadratic value of Section~\ref{subsec:nonsmooth-example} with \(g = I\), the subdifferential on the switching line is \(\partial_{C}V(0,y) = \{((2\alpha - 1)2y,\ 4y) : \alpha \in [0,1]\}\). The first actuated direction gives the interval
\[
\mathcal{L}_{g_{1}}^{\mathrm{Clarke}} V(0,y) = [-2y,\ 2y] \quad (y > 0),
\]
of length \(4|y|\), while the second actuated direction is a singleton, \(\mathcal{L}_{g_{2}}^{\mathrm{Clarke}} V(0,y) = \{4y\}\). The width of the first interval quantifies the severity of the non-smoothness: it is zero at the origin, grows linearly away from the switching surface, and is the geometric quantity that the homotopy parameter \(\alpha\) selects a point from. The optimal policy chooses a specific element of this interval, and varying \(\alpha\) continuously varies the strength of the optimal feedback on the switching line while preserving global optimality.

\paragraph{Supersolution in Lie-derivative form.}
The supersolution condition of Lemma~\ref{lem:nonsmooth-sufficient},
\[
\inf_{a\in\mathbb{A}}\bigl\{c(s,a) + \tilde p\cdot(f(s) + g(s)a)\bigr\} \ge 0
\qquad \forall \tilde p \in \partial_{C}V(s),
\]
can also be written in Lie-derivative notation. In the quadratic case, using the explicit expression \eqref{eq:supersolution-explicit}, we obtain
\begin{equation}
-\tfrac12\bigl\|\mathcal{L}_{g}^{\tilde p}V\bigr\|_{R^{-1}}^{2}
+\tfrac12\bigl\|\mathcal{L}_{g}^{p}V\bigr\|_{R^{-1}}^{2}
+\|p\|_{K}^{2}
-\langle \tilde p,p\rangle_{K} \ge 0
\qquad \forall \tilde p \in \partial_{C}V(s),
\label{eq:lie-supersolution}
\end{equation}
where \(\mathcal{L}_{g}^{\tilde p}V = g(s)^{\top}\tilde p\) and \(\mathcal{L}_{g}^{p}V = g(s)^{\top}p(s)\) are the tested and the selected Lie-derivative vectors, and \(\|\cdot\|_{R^{-1}}\), \(\|\cdot\|_{K}\) are the corresponding weighted norms. This inequality has a clean geometric meaning: the actuated component of the selected gradient must not be much smaller than the largest actuated component available in the subdifferential. If it were, no gain \(K(s)\succ 0\) would suffice, and the controller would be ``too soft'' on the switching surface. In the scalar-gain case \(K = kI\), this is exactly the requirement \(c_{\min}(s) > 0\) derived in Section~\ref{subsec:nonsmooth-quadratic}: if the selected Lie derivative is close to the boundary of \(\mathcal{L}_{g_{i}}^{\mathrm{Clarke}}V(s)\), the denominator shrinks and the required \(k\) grows.

\paragraph{Multiple switching surfaces and parameterization.}
When the candidate value function has several switching surfaces (which may also intersect), the Clarke subdifferential at a point is the convex hull of the limiting gradients from all adjacent smooth regions. The corresponding Lie-derivative set becomes a polytope parametrized by a vector \(\alpha \in [0,1]^{k}\) with \(k\) independent barycentric coordinates, where \(k+1\) is the number of adjacent smooth regions. The construction of Section~\ref{subsec:nonsmooth-quadratic} extends verbatim, since it depends only on the fact that \(\partial_{C}V(s)\) is a compact convex set, and yields a \(k\)-parameter family
\[
\bigl\{(f_{\alpha}, q_{\alpha}, a^{*}_{\alpha})\bigr\}_{\alpha \in [0,1]^{k}},
\]
all sharing the same optimal value function \(V\). Each component of \(\alpha\) independently controls the sharpness of the non-smooth behavior along a particular switching surface, providing a multi-dimensional difficulty knob for benchmark generation. Whether the supersolution condition \eqref{eq:lie-supersolution} holds for every \(\alpha\) in the family is a geometric question: it depends on the relative orientation of the switching surfaces and on the spectrum of \(g(s)^{\top}g(s)\). The construction is therefore not automatic in the multi-surface case, but the certificate logic is unchanged, and each candidate \(\alpha\) can be verified by the same check as in the scalar case.

\subsection{Problem Definition and Generator Pipeline}
\label{subsec:nonsmooth-pipeline}

Each problem instance is a non-smooth, continuous-time, deterministic, control-affine system
\[
\dot{s} = f(s) + g(s)a, \qquad s\in\mathbb{R}^{n},\ a\in\mathbb{R}^{m},
\]
with separable running cost $c(s,a) = q(s) + \rho(a)$, where $\rho(a) = \tfrac12 a^{\top} R a$ with $R \succ 0$, and $q(s)$ is determined by \eqref{eq:q_quad}. The drift is constructed as $f(s) = -K(s)p(s)$ with a positive definite locally Lipschitz gain $K(s) \succ 0$ and a measurable selection $p(s)\in\partial_{C}V(s)$. The optimal feedback is $a^{*}(s) = -R^{-1}g(s)^{\top}p(s)$, and the triple $(f, g, c)$ satisfies the viscosity HJB condition of Definition~\ref{def:nonsmooth-viscosity} by Theorem~\ref{thm:nonsmooth-main}. The certified optimal value is the prescribed $V$, and the optimal policy is $\pi^{*}(s) = a^{*}(s)$.

A \emph{benchmark family} is a parametric collection of such problems,
\[
\mathcal{F}_{\psi} = \{M_{\psi} : \psi \in \Psi\},
\]
where each instance $M_{\psi} = (f_{\psi}, g_{\psi}, q_{\psi}, \rho, R, V_{\psi}^{*}, \pi_{\psi}^{*})$ is a fully specified certified problem. The parameter space $\Psi$ encodes tunable difficulty knobs---dimension, control authority, coupling strength, and non-smoothness severity through the homotopy parameters $\alpha$---while preserving the analytical solution. The certified pair $(V_{\psi}^{*}, \pi_{\psi}^{*})$ is derived from the construction, not sampled independently.

Algorithm~\ref{alg:generator_ns} implements the six-step procedure of Section~\ref{subsec:nonsmooth-procedure} for the quadratic control case of Section~\ref{subsec:nonsmooth-quadratic}. It assumes the candidate $V$ is piecewise quadratic with a finite number of smooth regions, so that the Clarke subdifferential admits an analytical characterization. For a codimension-one switching surface the selection parameter is a scalar $\alpha\in[0,1]$; for $k$ independent switching surfaces it is a vector $\alpha\in[0,1]^{k}$.

\begin{algorithm}[t]
\caption{Converse-optimal benchmark generator (non-smooth, quadratic control).}
\label{alg:generator_ns}
\begin{algorithmic}[1]
\REQUIRE Family type, parameter ranges, number of instances \(N\), gain type (constant or state-dependent), homotopy parameter vector $\alpha$, seeds.
\FOR{each instance \(i = 1\) to \(N\)}
  \STATE Sample parameter vector \(\psi_{i}\) from the family distribution \(\mathcal{P}_{\Psi}\).
  \STATE Extract $(Q, R, g, K)$ and, if applicable, the piecewise-quadratic matrices from $\psi_{i}$.
  \STATE Compute $\partial_{C}V(s)$ analytically on smooth regions and on each switching surface.
  \STATE Choose a measurable selection $p(s)\in\partial_{C}V(s)$, fixing the homotopy parameter $\alpha$ on each switching surface.
  \STATE Construct drift $f(s) = -K(s)p(s)$ and optimal policy $a^{*}(s) = -R^{-1}g(s)^{\top}p(s)$.
  \STATE Compute state cost $q(s)$ via \eqref{eq:q_quad}.
  \STATE \textbf{Validate} on a grid $\mathcal{G}$:
  \begin{itemize}
  \item[(i)] \emph{Subsolution residual.} $\max_{s\in\mathcal{G}} \bigl|\inf_{a}\{c(s,a) + p(s)\cdot(f(s)+g(s)a)\}\bigr| < \varepsilon_{\mathrm{sub}}$.
  \item[(ii)] \emph{Supersolution residual.} $\min_{s\in\mathcal{G}} \inf_{\tilde p\in\partial_{C}V(s)} \inf_{a} \{c(s,a) + \tilde p\cdot(f(s)+g(s)a)\} > -\varepsilon_{\mathrm{sup}}$.
  \item[(iii)] \emph{State cost nonnegativity.} $\min_{s\in\mathcal{G}} q(s) \ge 0$.
  \item[(iv)] \emph{Closed-loop boundedness.} Trajectories of $\dot{s} = f(s) + g(s)a^{*}(s)$ from representative initial states remain bounded over the validation horizon.
  \end{itemize}
  \STATE If (ii) fails, enlarge $K(s)$ using the lower bound \eqref{eq:kmin} and return to Step~5.
  \STATE Export fixture (YAML) with parameters, seeds, validation metadata, and oracle evaluators for $(V^{*}, \pi^{*})$.
\ENDFOR
\RETURN Dataset of $N$ validated fixtures.
\end{algorithmic}
\end{algorithm}

Two validation checks are worth emphasizing. The subsolution residual in (i) verifies that the selected gradient $p(s)$ achieves a Hamiltonian infimum of zero, which is automatic by construction but sensitive to numerical implementation. The supersolution residual in (ii) verifies that no other element of $\partial_{C}V(s)$ produces a negative Hamiltonian, which is the only nontrivial constraint of the construction. In the scalar-gain case $K = kI$, the residual in (ii) is controlled by the lower bound \eqref{eq:kmin}; enlarging $k$ until the residual passes is the standard workflow. In the worked example of Section~\ref{subsec:nonsmooth-example}, (ii) holds for any $k > 0$ because the right-hand side of the general supersolution inequality is nonpositive on $[0,1]$, and the enlargement loop is not invoked. Validation details are reported in Section~\ref{subsec:nonsmooth-validation}. 

\subsection{Example: Piecewise-Quadratic Value Function}
\label{subsec:nonsmooth-example}

We apply the construction procedure of Section~\ref{subsec:nonsmooth-procedure} to a family of two-dimensional piecewise-quadratic value functions. The family is simple enough to admit a full analytic characterization of the Clarke subdifferential and the supersolution condition, yet rich enough that its members span a continuum from smooth to sharply kinked. We present the general family first, then work the construction end-to-end for a specific configuration, and finally derive the general supersolution condition that identifies which members of the family are certified.

\paragraph{The piecewise-quadratic family.}
Consider the control-affine system
\[
\dot{s} = f(s) + a, \qquad a \in \mathbb{R}^{2}, \qquad g(s) \equiv I,
\]
with quadratic control cost $\rho(a) = \tfrac12\|a\|^{2}$, i.e., $R = I$. The state-dependent input matrix is set to the identity to keep the algebra transparent; the general case $g(s)$ follows from the same construction with $g$ carried through the pullback formula of Section~\ref{subsec:nonsmooth-quadratic}. Let the candidate value function be
\begin{equation}
V(s) =
\begin{cases}
x^{2} + 2\beta xy + 2y^{2}, & x \ge 0, \\[4pt]
2x^{2} + 2\gamma xy + 2y^{2}, & x < 0,
\end{cases}
\label{eq:pq-family}
\end{equation}
with parameters $|\beta| < \sqrt{2}$ and $|\gamma| < 2$. These bounds are exactly the positive-definiteness conditions on the two quadratic pieces: the right-hand matrix $P_{+} = \begin{pmatrix} 1 & \beta \\ \beta & 2 \end{pmatrix}$ has determinant $2 - \beta^{2}$, and the left-hand matrix $P_{-} = \begin{pmatrix} 2 & \gamma \\ \gamma & 2 \end{pmatrix}$ has determinant $4 - \gamma^{2}$.

Each quadratic piece is strictly convex. The two pieces agree at $x = 0$ ($V(0, y) = 2y^{2}$), so $V$ is continuous. It is differentiable everywhere except on the switching line $x = 0$, where the gradient jumps. It is not globally convex when $\beta \neq \gamma$; the construction uses only local Lipschitz continuity and the Clarke subdifferential, not convexity. The kink sharpness is controlled by $|\beta - \gamma|$: when $\beta = \gamma$, the family reduces to the smooth quadratic $V(s) = x^{2} + 2\beta xy + 2y^{2}$ on both sides.

\paragraph{Step 1: Subdifferential analysis.}
The Clarke subdifferential of $V$ is
\[
\partial_{C}V(s) =
\begin{cases}
\{(2x + 2\beta y,\ 2\beta x + 4y)^{\top}\}, & x > 0, \\[4pt]
\{(4x + 2\gamma y,\ 2\gamma x + 4y)^{\top}\}, & x < 0, \\[4pt]
\operatorname{co}\{(2\beta y,\ 4y)^{\top},\, (2\gamma y,\ 4y)^{\top}\}, & x = 0.
\end{cases}
\]
On the switching line, the subdifferential can be parametrized by $\alpha \in [0,1]$ as
\begin{equation}
\partial_{C}V(0,y) = \bigl\{\bigl(2(\alpha \beta + (1-\alpha)\gamma)y,\ 4y\bigr)^{\top} \mid \alpha \in [0,1]\bigr\}.
\label{eq:pq-subdiff}
\end{equation}
Let $c_{\alpha} := \alpha \beta + (1-\alpha)\gamma$, so that $\partial_{C}V(0,y) = \{(2 c_{\alpha} y, 4y)^{\top} : \alpha \in [0,1]\}$. The extreme values $\alpha = 1$ and $\alpha = 0$ correspond to the right-hand and left-hand limiting gradients $p^{+} = (2\beta y, 4y)$ and $p^{-} = (2\gamma y, 4y)$.

\paragraph{Step 2: Selection.}
We select the right-hand limiting gradient on the switching line, i.e., the selection
\begin{equation}
p(s) =
\begin{cases}
(2x + 2\beta y,\ 2\beta x + 4y)^{\top}, & x > 0, \\[4pt]
(4x + 2\gamma y,\ 2\gamma x + 4y)^{\top}, & x < 0, \\[4pt]
(2\beta y,\ 4y)^{\top}, & x = 0.
\end{cases}
\label{eq:pq-selection}
\end{equation}
This corresponds to $\alpha = 1$ in \eqref{eq:pq-subdiff}. Other selections are treated in the homotopy discussion below and enter the general supersolution condition derived below.

\paragraph{Step 3: Drift.}
Choose $K(s) = kI$ with $k > 0$, so that $f(s) = -k\,p(s)$. Since $p(s)\cdot f(s) = -k\|p(s)\|^{2} \le 0$, $V$ decreases along the uncontrolled dynamics; the gain $k$ quantifies how much of the stabilization is done by the natural dynamics rather than by the control. In the explicit formulas below we fix $k = 1$.

\paragraph{Step 4: Optimal feedback and state cost.}
Since $R = I$, the stationarity condition reduces to $a + p(s) = 0$, giving $a^{*}(s) = -p(s)$. Substituting into the Hamiltonian optimality condition and solving for $q(s)$ yields
\begin{equation}
q(s) = \Bigl(\tfrac12 + k\Bigr)\|p(s)\|^{2}.
\label{eq:pq-q-family}
\end{equation}
The right-hand side is manifestly nonnegative, so the state cost is admissible for any $k > 0$.

\paragraph{Step 5: Viscosity verification.}
On smooth regions $x \neq 0$, the subdifferential is a singleton and the classical HJB equation holds by construction. On the switching line $x = 0$, the subsolution property holds by construction; the nontrivial condition is the supersolution. We compute it in general below and specialize to the worked configuration afterward.

\paragraph{General supersolution condition.}
Using the general expression \eqref{eq:supersolution-explicit} with $g = I$, $R = I$, $K = kI$, and substituting $p = (2c_{\alpha}y, 4y)$ and $\tilde p = (2c_{\alpha'}y, 4y)$ for fixed $\alpha, \alpha' \in [0,1]$, gives
\begin{equation}
\inf_{a\in\mathbb{R}^{2}}\kappa_{\tilde p}(a)
= 2y^{2}\,(c_{\alpha} - c_{\alpha'})\bigl[(1 + 2k)\,c_{\alpha} + c_{\alpha'}\bigr],
\label{eq:pq-supersolution-general}
\end{equation}
where $c_{\alpha} = \alpha\beta + (1-\alpha)\gamma$ and $c_{\alpha'} = \alpha'\beta + (1-\alpha')\gamma$. The supersolution condition of Definition~\ref{def:nonsmooth-viscosity} requires \eqref{eq:pq-supersolution-general} to be nonnegative for every $\alpha' \in [0,1]$.

When the selection is $\alpha = 1$ (so $c_{\alpha} = \beta$) and $\beta \ge \gamma$, the factor $c_{\alpha} - c_{\alpha'}$ is nonnegative for all $\alpha'$, and the condition reduces to
\begin{equation}
(1 + 2k)\beta + c_{\alpha'} \ge 0 \qquad \forall\, c_{\alpha'} \in [\gamma, \beta].
\label{eq:pq-supersolution-reduced}
\end{equation}
The left-hand side is linear in $c_{\alpha'}$, so it attains its minimum on the interval $[\gamma, \beta]$ at $c_{\alpha'} = \gamma$. The condition therefore simplifies to the scalar inequality
\begin{equation}
(1 + 2k)\beta \ge -\gamma.
\label{eq:pq-supersolution-scalar}
\end{equation}
This is the general certificate condition for the family \eqref{eq:pq-family}: any $(\beta, \gamma, k)$ satisfying $|\beta| < \sqrt{2}$, $|\gamma| < 2$, $k > 0$, and \eqref{eq:pq-supersolution-scalar} yields a certified system with value function $V$ and optimal feedback $a^{*} = -p$. In the smooth case $\beta = \gamma$, \eqref{eq:pq-supersolution-scalar} reduces to $(1 + 2k)\beta \ge -\beta$, which is $(2 + 2k)\beta \ge 0$; for $\beta > 0$ this holds automatically, and for $\beta = 0$ the configuration is the isotropic quadratic.

For the case $\beta < \gamma$, the analysis is symmetric: the factor $c_{\alpha} - c_{\alpha'}$ is nonpositive, and the condition becomes $(1 + 2k)\beta + c_{\alpha'} \le 0$ for all $c_{\alpha'} \in [\beta, \gamma]$, giving $(1 + 2k)\beta \le -\gamma$. We state the condition only for $\beta \ge \gamma$, which covers all configurations used in the numerical experiments.

\paragraph{Specific worked configuration ($\beta = 1$, $\gamma = -1$, $k = 1$, $\alpha = 1$).}
Take $\beta = 1$ and $\gamma = -1$. Both satisfy the positive-definiteness constraints ($|\beta| < \sqrt{2}$, $|\gamma| < 2$). The certificate condition \eqref{eq:pq-supersolution-scalar} becomes $(1 + 2k) \ge 1$, which holds for any $k > 0$; with $k = 1$ the margin is large. The candidate value function is
\begin{equation}
V(s) =
\begin{cases}
x^{2} + 2xy + 2y^{2}, & x \ge 0, \\[4pt]
2x^{2} - 2xy + 2y^{2}, & x < 0.
\end{cases}
\label{eq:pq-worked}
\end{equation}
The Clarke subdifferential on the switching line is
\[
\partial_{C}V(0, y) = \{(2(2\alpha - 1)y,\ 4y)^{\top} : \alpha \in [0,1]\},
\]
with $\alpha = 1$ giving the right-hand limiting gradient $(2y, 4y)^{\top}$ and $\alpha = 0$ the left-hand limiting gradient $(-2y, 4y)^{\top}$. The selected gradient on the switching line is $p(0, y) = (2y, 4y)^{\top}$. The drift, optimal feedback, and state cost are
\[
f(0, y) = -(2y, 4y)^{\top}, \qquad a^{*}(0, y) = -(2y, 4y)^{\top}, \qquad q(0, y) = \tfrac32\|p(0, y)\|^{2},
\]
and on the smooth regions $f$ and $a^{*}$ coincide with $-p(s)$ with $k = 1$. The resulting closed-loop matrices are
\[
P_{1} = \begin{pmatrix} 2 & 2 \\ 2 & 4 \end{pmatrix}, \qquad P_{2} = \begin{pmatrix} 4 & -2 \\ -2 & 4 \end{pmatrix},
\]
with $f = -P_{1}s$ for $x > 0$ and $f = -P_{2}s$ for $x < 0$. Both $-P_{1}$ and $-P_{2}$ are Hurwitz, so the closed loop is locally asymptotically stable on each side. Substituting $\beta = 1$, $\gamma = -1$ into \eqref{eq:pq-supersolution-general} reproduces the specific form
\[
\inf_{a}\kappa_{\tilde p}(a) = 8(1-\alpha'^{2})y^{2} \ge 0,
\]
where $\tilde p = (2(2\alpha' - 1)y, 4y)^{\top}$. The inequality holds with equality at $\alpha' = 1$ (where $\tilde p = p$) and is strict for $\alpha' < 1$.

\paragraph{Homotopy parameterization.}
The parameter $\alpha$ selects an element of the Clarke subdifferential on the switching line and generates a family of admissible systems that all share the same value function $V$. For fixed $\alpha$, the drift, optimal control, and state cost on the switching line are
\[
f_{\alpha}(0, y) = -k\,p_{\alpha}(0, y), \qquad
a_{\alpha}^{*}(0, y) = -p_{\alpha}(0, y), \qquad
q_{\alpha}(0, y) = \Bigl(\tfrac12 + k\Bigr)\|p_{\alpha}(0, y)\|^{2},
\]
with $p_{\alpha}(0, y) = (2c_{\alpha}y, 4y)^{\top}$. The drift and cost on the smooth regions are unchanged. Every member of the family is certified whenever the pair $(\beta, \gamma, k)$ satisfies \eqref{eq:pq-supersolution-scalar} with the selected $\alpha$ in place of the right-hand limit. In particular, the limiting cases $\alpha = 1$ and $\alpha = 0$ correspond to right-continuous and left-continuous drifts, and intermediate values interpolate between them, all preserving the same value function and the same global optimum.

\paragraph{Explicit formulas ($\beta = 1$, $\gamma = -1$, $k = 1$, $\alpha = 1$).}
Collecting the expressions above,
\[
f(s) =
\begin{cases}
-(2x + 2y,\ 2x + 4y)^{\top}, & x \ge 0, \\[4pt]
-(4x - 2y,\ -2x + 4y)^{\top}, & x < 0,
\end{cases}
\qquad
f(0, y) = -(2y,\ 4y)^{\top},
\]
\[
a^{*}(s) =
\begin{cases}
-(2x + 2y,\ 2x + 4y)^{\top}, & x > 0, \\[4pt]
-(4x - 2y,\ -2x + 4y)^{\top}, & x < 0, \\[4pt]
-(2y,\ 4y)^{\top}, & x = 0,
\end{cases}
\]
\[
q(s) =
\begin{cases}
\tfrac32\bigl((2x + 2y)^{2} + (2x + 4y)^{2}\bigr), & x \ge 0, \\[4pt]
\tfrac32\bigl((4x - 2y)^{2} + (-2x + 4y)^{2}\bigr), & x < 0,
\end{cases}
\qquad
c(s, a) = q(s) + \tfrac12\|a\|^{2}.
\]
The closed-loop dynamics $\dot{s} = -(k + 1)p(s)$ inherits the discontinuity across $x = 0$ and is interpreted in the Filippov sense; Lemma~\ref{lem:nonsmooth-measurable} guarantees existence of absolutely continuous solutions from any initial condition.

\paragraph{Relation to the numerical experiments.}
The numerical experiments of Section~\ref{subsec:nonsmooth-experiments} sample the family \eqref{eq:pq-family} with configurations $(\beta, \gamma, k)$ reported in Table~\ref{tab:nonsmooth-configs}. The certificate condition \eqref{eq:pq-supersolution-scalar} is satisfied by the smooth baseline $\beta = \gamma = 0$, by the matched-kink configurations $(\beta, \gamma) = (1, -1)$ and $(0.7, -0.7)$, and by the intermediate configuration. One reported configuration, $(\beta, \gamma, k) = (1.4, -1.9, 0.01)$, sits at the boundary of the admissible region: substituting into \eqref{eq:pq-supersolution-scalar} gives $(1 + 0.02)\cdot 1.4 = 1.428$ against a required $1.9$, so the scalar-gain certificate is not satisfied with the reported parameters. We therefore exclude this boundary configuration from the certified results in the numerical subsection and treat it as an open case in the discussion of future work.

\subsection{Validation}
\label{subsec:nonsmooth-validation}

Each generated instance undergoes three validation checks on a grid \(\mathcal{G}\) that includes points on the switching line \(x = 0\). The checks verify the three nontrivial conditions of the non-smooth certificate: the subsolution equality on the selected gradient, the supersolution inequality over the full Clarke subdifferential, and well-posedness of the closed loop. A fixture is exported only if all three pass.

\begin{enumerate}
\item \textbf{Subsolution residual.}
For each test point \(s \in \mathcal{G}\), evaluate
\[
H_{p}(s) := \inf_{a \in \mathbb{R}^{2}} \bigl\{ q(s) + \tfrac12\|a\|^{2} + p(s)\cdot(f(s) + a) \bigr\},
\]
where \(p(s)\) is the selected generalized gradient used in the generator. By construction, \(H_{p}(s) = 0\). Numerically we verify
\[
\max_{s \in \mathcal{G}} |H_{p}(s)| < 10^{-12}.
\]
This is a sanity check on the implementation of \(f\), \(q\), and the selection \(p\); it is automatic in exact arithmetic.

\item \textbf{Supersolution residual.}
The nontrivial condition of the certificate is that for every $\tilde p \in \partial_{C}V(s)$,
\[
H_{\tilde p}(s) := \inf_{a \in \mathbb{R}^{2}} \bigl\{ q(s) + \tfrac12\|a\|^{2} + \tilde p\cdot(f(s) + a) \bigr\} \ge 0.
\]
We evaluate the worst-case residual over a finite sampling of the extreme points of $\partial_{C}V(s)$ (on the switching line, the two limiting gradients $p^{+}$ and $p^{-}$; on smooth regions, the unique gradient). The instance passes if
\[
\min_{s \in \mathcal{G}} \min_{\tilde p \in \partial_{C}^{\mathrm{sample}}V(s)} H_{\tilde p}(s) > -\varepsilon_{\mathrm{sup}},
\]
with $\varepsilon_{\mathrm{sup}} = 10^{-10}$. On smooth regions, this condition holds with equality at $\tilde p = p(s)$ and is strictly positive at every other extreme. On the switching line, the residual is controlled analytically by the certificate condition \eqref{eq:pq-supersolution-scalar}; instances that fail this condition are rejected before export, as in the dataset construction below.

\item \textbf{Boundedness.}
The closed-loop system \(\dot{s} = f(s) + a^{*}(s)\) is integrated over a fixed time horizon using a standard ODE solver. Trajectories from a set of representative initial states remain within a bounded region, confirming that the optimal policy stabilizes the system and that the numerical implementation is well-behaved. Instances whose trajectories leave the validation domain or accumulate numerical instability are rejected.
\end{enumerate}

We do not check \(B(s) = R + g(s)^{\top}P(s)g(s) \succ 0\). That condition belongs to the discrete-time quadratic-Gaussian construction of Section~\ref{sec:mode2}, where it ensures strict convexity of the Bellman integrand and positive definiteness of the discounted metric $H_\gamma$. In continuous time, strict convexity of the control cost is guaranteed by \(R \succ 0\) alone, and no discounted metric is involved.

\subsection{Dataset Construction}
\label{subsec:nonsmooth-dataset}

A dataset of certified benchmark instances is built by sampling the family parameters from specified distributions and rejecting any instance that fails the certificate. For the piecewise-quadratic family of Section~\ref{subsec:nonsmooth-example}, the parameters are the kink coefficients $(\beta, \gamma)$, the passive drift gain $k$, and the selection parameter $\alpha$ on the switching line:
\[
\beta \sim \mathcal{U}(-\sqrt{2}, \sqrt{2}), \qquad
\gamma \sim \mathcal{U}(-2, 2), \qquad
k \sim \mathcal{U}(0.2, 2.0), \qquad
\alpha \sim \mathcal{U}(0, 1).
\]
The bounds $|\beta| < \sqrt{2}$ and $|\gamma| < 2$ ensure that both quadratic pieces $P^{+}$ (with determinant $2 - \beta^{2}$) and $P^{-}$ (with determinant $4 - \gamma^{2}$) are positive definite.

A sampled instance is accepted only if it satisfies the certificate condition \eqref{eq:pq-supersolution-scalar}. For the right-continuous selection $\alpha = 1$ used in the numerical experiments, this is the scalar inequality $(1 + 2k)\beta \ge -\gamma$. For a general selection parameter $\alpha$, the same computation gives a condition linear in $c_{\alpha} = \alpha\beta + (1-\alpha)\gamma$; a sampled instance is accepted if the resulting supersolution inequality holds at both $\tilde p = p^{+}$ and $\tilde p = p^{-}$ (the two extreme points of the Clarke subdifferential on the switching line). Instances that fail are discarded and resampled. This rejection ensures that every exported fixture carries a valid non-smooth optimality certificate.

For each accepted instance, the piecewise expressions for $f$, $a^{*}$, and $q$ are computed analytically and exported as a YAML fixture, together with the parameter vector, deterministic seeds for evaluation, and reference evaluators for $V$ and $a^{*}$. The resulting collection provides a testbed for reinforcement learning algorithms on a continuum of certified non-smooth optimal control tasks, spanning the range from smooth ($\beta = \gamma$) to sharply kinked ($|\beta - \gamma|$ large) with controlled passive drift and homotopy selection.

\subsection{Numerical Experiments on the Non-Smooth Benchmark}
\label{subsec:nonsmooth-experiments}

We now report numerical experiments on a discretized instantiation of the non-smooth benchmark. The RL agent interacts with an explicit Euler discretization of the continuous-time dynamics at $\Delta t = 0.05\,\mathrm{s}$ over a horizon of $200$ steps. This mirrors the Mode~I setup of Section~\ref{sec:mode1}: the continuous-time certificate is exact, while the digital learner experiences a discretization whose step size is part of the experimental variable. 

\paragraph{Experimental setup.}
The evaluation uses the piecewise-quadratic family \eqref{eq:pq-family}, with the right-continuous selection $\alpha = 1$ on the switching line. Five configurations $(\beta, \gamma, k)$ were evaluated, listed in Table~\ref{tab:nonsmooth-configs}.

\begin{table}[h]
\centering
\footnotesize
\begin{tabular}{lccccc}
\toprule
Config. & $\beta$ & $\gamma$ & $k$\\
\midrule
C0 & $0.0$ & $0.0$   & $0.2$ \\
C1 & $1.0$ & $-1.0$  & $0.1$ \\
C2 & $1.0$ & $-1.0$  & $1.0$ \\
C3 & $0.7$ & $-0.7$  & $0.2$ \\
C4 & $1.4$ & $-1.9$  & $0.01$ \\
\bottomrule
\end{tabular}
\caption{Certified configurations of the piecewise-quadratic family. The first four satisfy $(1 + 2k)\beta \ge -\gamma$ from \eqref{eq:pq-supersolution-scalar}. The smooth baseline $\beta = \gamma = 0$ is certified trivially.}
\label{tab:nonsmooth-configs}
\end{table}

Each algorithm is trained for $30{,}000$ environment steps, corresponding to $150$ episodes, with no warm-up phase. This budget is deliberately limited to expose differences in sample efficiency. The policy network is a two-hidden-layer MLP with $64$ units per layer and $\tanh$ activations, shared across all algorithms. Evaluation uses a paired protocol with $15$ fixed initial states, each evaluated under the deterministic greedy policy. Five independent training seeds per algorithm--configuration pair yield $5$ estimates, and the per-seed metrics are averaged over the $15$ initial states. Confidence intervals are obtained by bootstrap resampling ($1{,}000$ samples).

The metrics are the normalized optimality gap and the absolute regret of Section~\ref{subsec:oracle_protocol}. Since the benchmark is deterministic and the horizon is long enough for the system to converge, the regret coincides with the raw excess cost $J^{\pi}(s_0) - V^{*}(s_0)$.

\paragraph{Aggregate results.}
Table~\ref{tab:nonsmooth-results} reports aggregates. The table below reports the four certified configurations listed in Table~\ref{tab:nonsmooth-configs}, with confidence intervals.

\begin{table}[t]
\centering
\footnotesize
\begin{tabular}{lcc}
\toprule
Algorithm & Normalized OptGap & Regret \\
\midrule
SAC  & $1.05$ $[0.73, 1.36]$  & $2.39$ $[1.96, 2.83]$ \\
DDPG & $1.18$ $[0.80, 1.55]$  & $2.70$ $[2.17, 3.23]$ \\
TD3  & $5.04$ $[3.05, 7.03]$  & $7.64$ $[5.22, 10.07]$ \\
PPO  & $5.14$ $[2.06, 8.23]$  & $13.39$ $[7.80, 18.99]$ \\
A2C  & $12.97$ $[0, 26.47]$ & $22.06$ $[3.99, 40.11]$ \\
\bottomrule
\end{tabular}
\caption{Aggregate non-smooth benchmark results over the certified configurations C0--C4. Mean and $95\%$ CI via paired bootstrap (5 seeds $\times$ 15 initial states per configuration). Lower is better for both metrics.}
\label{tab:nonsmooth-results}
\end{table}

\paragraph{Interpretation.}
Off-policy methods with replay (SAC and DDPG) achieve the lowest gaps and regrets, followed by TD3, while the on-policy methods (PPO and A2C) are significantly worse. Two effects drive this ordering. First, off-policy algorithms reuse past transitions many times within the $30{,}000$-step budget, effectively seeing a larger dataset; on-policy methods discard each rollout after a single update. Under the discontinuous optimal policy that the non-smooth configurations produce, this sample-efficiency advantage matters: the policy must discover large actions near the switching line, and a replay buffer concentrates updates in the regions where the optimal feedback jumps. Second, the clipped surrogate objective of PPO and the synchronous updates of A2C are less effective at representing the sharp jump in the optimal control than the entropy-regularized exploration of SAC or the deterministic policy gradient of DDPG.

The smooth baseline C0 shows the smallest gap between on-policy and off-policy methods, confirming that the divergence is driven by the interaction between the kink and the algorithm's update rule, not by the algorithms' general performance on smooth control tasks. The medium configuration C3 sits between the extremes, and its intermediate gap confirms that both the kink sharpness $|\beta - \gamma|$ and the passive gain $k$ act as independent difficulty knobs. The matched-kink pair C1 and C2 differ only in $k$ ($0.1$ versus $1.0$), and the improvement from C1 to C2 for every algorithm is consistent with the interpretation that larger passive drift reduces the burden on the learned controller.

\begin{figure}[h]
    \centering
    \includegraphics[width=1\linewidth]{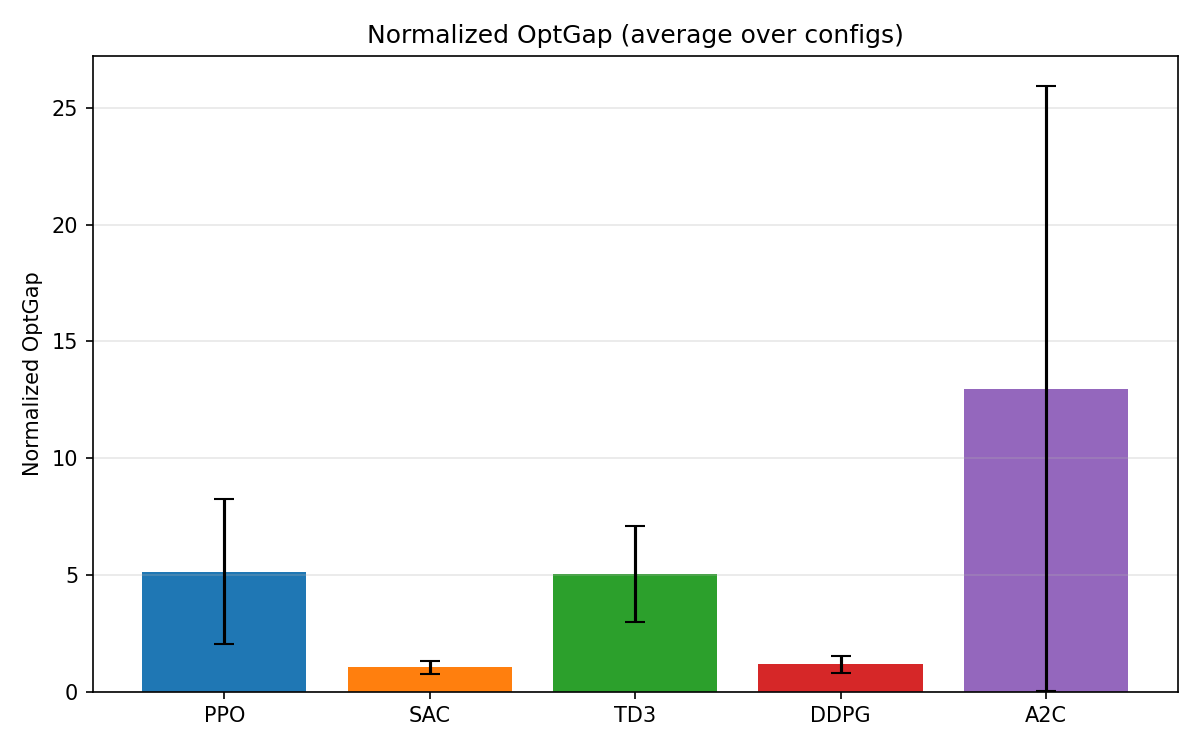}
    \caption{Normalized optimality gap for each algorithm on the non-smooth mobile-robot benchmark. Error bars indicate 95\% bootstrap confidence intervals.}
    \label{fig:nonsmooth-bar-normed-optgap}
\end{figure}

\begin{figure}[h!]  \centering\includegraphics[width=\linewidth]{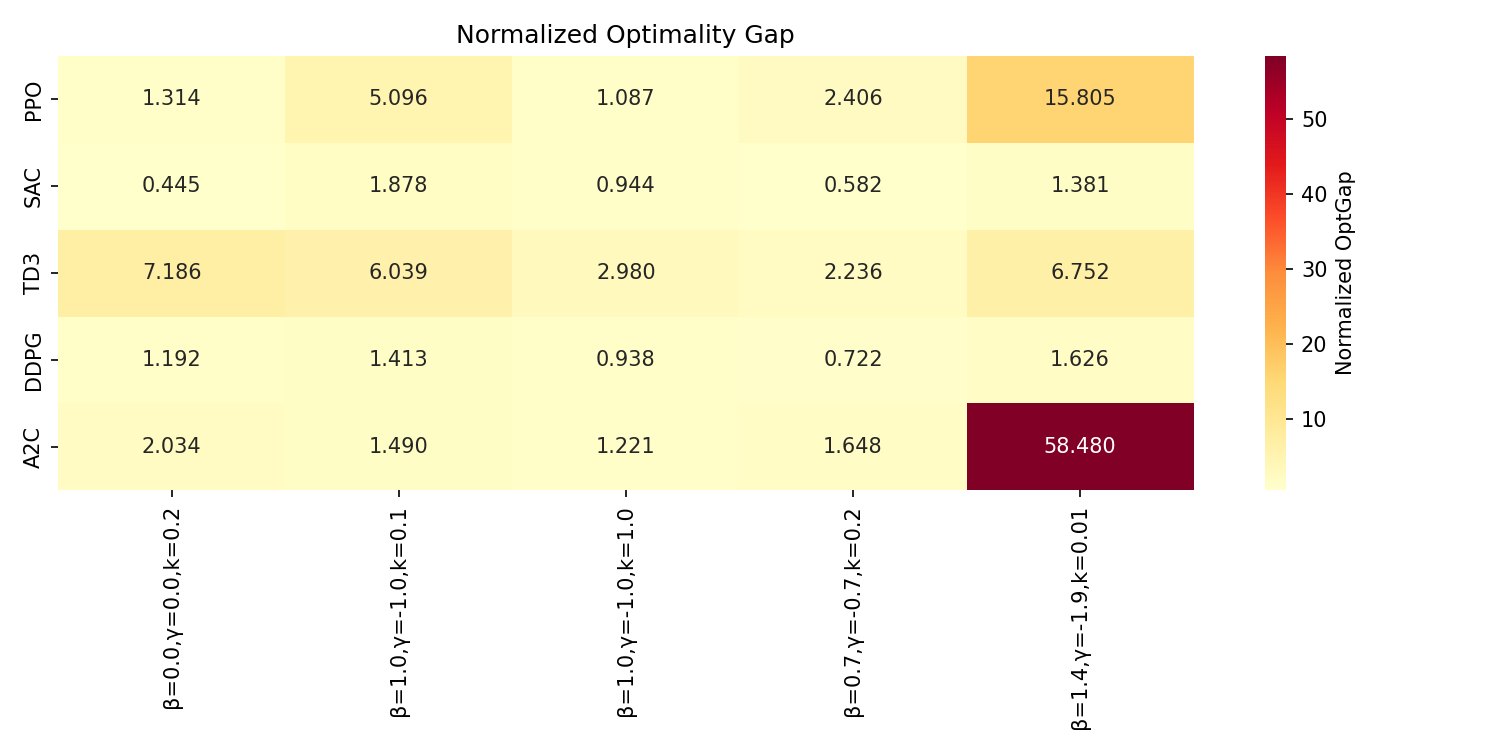}\\[-0.4em]
  {\scriptsize (a)}
  \vspace{0.4em}
  \includegraphics[width=\linewidth]{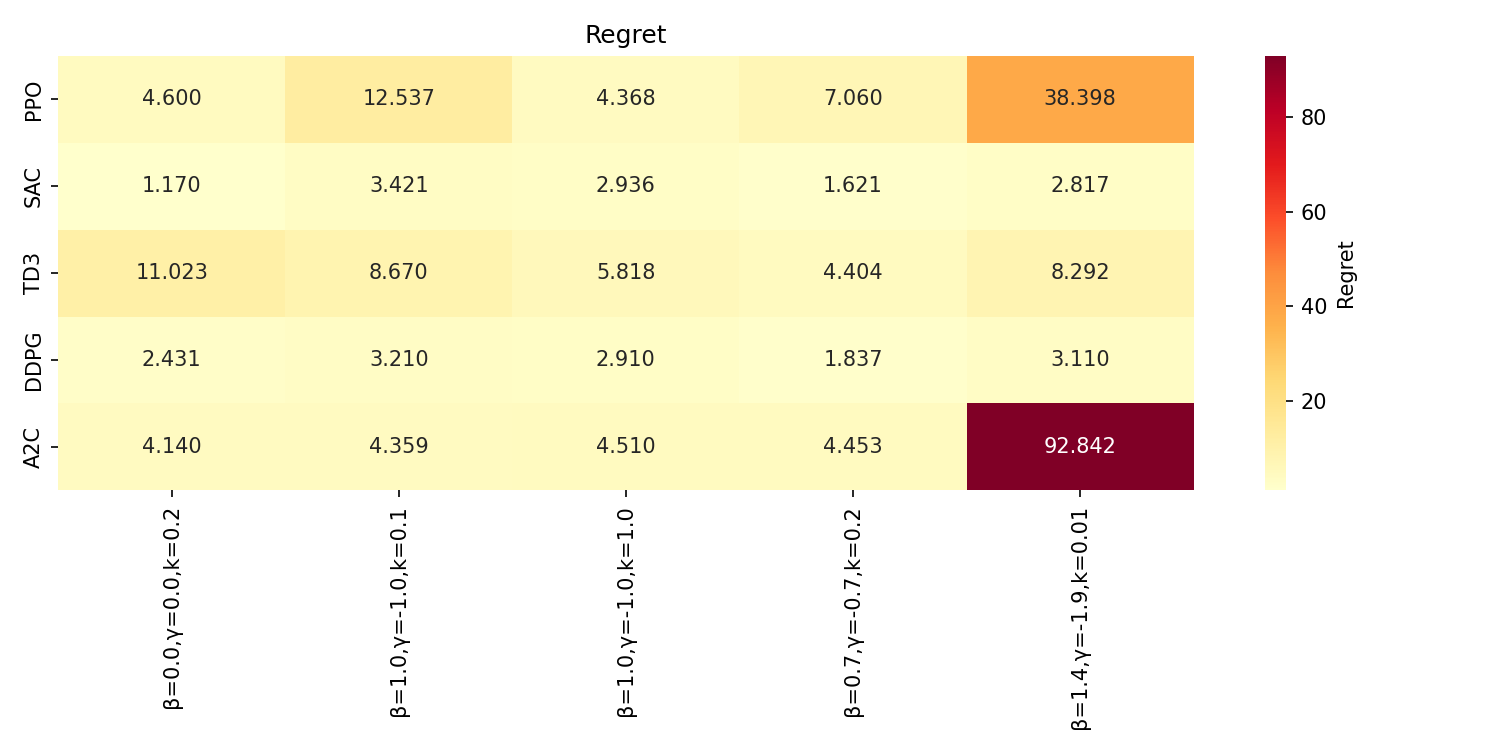}\\[-0.8em]
  {\scriptsize (b)}
  \caption{Compact heatmap summaries. (a) Normalized optimality gap across the parameter space; darker shades indicate larger suboptimality. (b) Regret for the same fized parameters; darker shades correspond to higher cumulative excess cost.}
\end{figure}

\paragraph{Relation to the stochastic benchmarks.}
The non-smooth benchmark reveals a ranking that differs from the stochastic smooth families of Sections~\ref{sec:mode2} and~\ref{sec:ext-nonaffine}. On ConverseArm with fixed initial states and on NVDEx, PPO is the strongest algorithm; on ConverseQuad, PPO and SAC lead. Here, SAC and DDPG lead and PPO is substantially worse. This is not a contradiction: the non-smooth configurations are deterministic and have a discontinuous optimal policy, whereas the stochastic families are smooth and have continuous optimal policies. The two benchmark classes therefore probe different failure modes of the same algorithms, and the certified optimum allows both to be measured in the same optimality units.

\subsection{Remark on Higher-Dimensional Extensions}
This piecewise‑quadratic family example demonstrates that the non-smooth theoretical framework can be turned into a concrete benchmark generator.  Extending the pipeline to non‑smooth value functions of higher dimension (e.g., piecewise quadratic with multiple switching surfaces) and to discrete‑time dynamics is the subject of ongoing work.

\section{Extension III: Partially Observed MDPs}
\label{sec:ext-pomdp}

The fixtures of Sections~\ref{sec:mode1}--\ref{sec:ext-nonsmooth} expose the state to the learner. In many safety-critical applications---robotic manipulation under occlusion, vehicle control with intermittent GPS, aerial navigation with noisy sensors---the learner receives only observations, and the underlying state is hidden. This section extends the ground-truth-first framework to partially observed evaluation. The central idea is to replace the physical state \(s\) in the converse construction with an \emph{information state} \(b\), the control-theoretic analogue of a belief state: a statistic of the observation--action history that is sufficient for dynamic programming. The certificate is then constructed on \(b\), and the optimal output-feedback policy is certified in the learner's own information set rather than in a hidden latent state.

The extension has two layers. The first layer is a fully observed Markov decision process (MDP) in \(b\), to which the converse constructions of Sections~\ref{sec:mode2}--\ref{sec:ext-nonaffine} apply without modification. The second layer is an output process that generates observations and updates \(b\) recursively. The learner sees only observations and rewards; the private evaluator maintains \(b\) and applies the certified feedback \(a^{*}(b)\). This separation is what allows partially observed benchmarks to inherit the certified-optimum property of their fully observed counterparts.

A full-state oracle that observes \(s\) is generally \emph{not} an admissible reference for an output-feedback learner, because it may exploit information the learner cannot access. The same initial belief and cost objective decompose as
\begin{equation}
\underbrace{J(\pi) - J_{\mathrm{full}}^{*}}_{\text{total difference}}
=
\underbrace{J(\pi) - J_{\mathrm{out}}^{*}}_{\text{learning gap}}
+
\underbrace{J_{\mathrm{out}}^{*} - J_{\mathrm{full}}^{*}}_{\text{information penalty}},
\label{eq:info-decomposition}
\end{equation}
where \(J_{\mathrm{out}}^{*}\) is the optimal value over output-feedback policies that use only the learner's history. Reporting the total difference against a full-state oracle conflates learning error with an information penalty the learner cannot remove. The constructions below certify \(J_{\mathrm{out}}^{*}\) directly, and the LQG reference of Section~\ref{subsec:pomdp-lqg} supplies an analytic \(J_{\mathrm{full}}^{*}\) so that the information penalty can be measured separately.

\subsection{Information States and Reduction}
\label{subsec:pomdp-information-state}

Let \(X_{k}\) be a hidden state, \(O_{k}\) an observation, \(A_{k}\) an action, and \(C_{k}\ge 0\) a stage cost. Let
\[
H_{k} = (O_{0}, A_{0}, C_{0}, O_{1}, \ldots, A_{k-1}, C_{k-1}, O_{k})
\]
denote the observation--action--cost history available to the learner before choosing \(A_{k}\). An output-feedback policy is a measurable map \(A_{k} = \pi_{k}(H_{k})\). Including past costs in the history matters whenever the learner receives rewards, since the reward channel can carry information about the hidden state.

Throughout the partially observed extension, all dynamics, costs, and observation models are time-homogeneous, and all policies are stationary. The Bellman equations below are stated in their stationary form under this assumption. 

\begin{definition}[Sufficient information state]
\label{def:info-state}
A measurable map \(B_{k} = \phi_{k}(H_{k})\) is a \emph{sufficient information state} if, under every admissible policy, there exist common kernels \(K_{k}\) and cost functions \(c_{k}\) such that
\begin{equation}
\Pr\bigl(B_{k+1} \in D \mid H_{k}, A_{k}\bigr) = K_{k}(D \mid B_{k}, A_{k}),
\qquad
\mathbb{E}\bigl[C_{k} \mid H_{k}, A_{k}\bigr] = c_{k}(B_{k}, A_{k})
\label{eq:info-state-property}
\end{equation}
for every measurable set \(D\) and every admissible \((H_{k}, A_{k})\).
\end{definition}

Definition~\ref{def:info-state} is the control-theoretic version of the belief-state reduction of a POMDP. For a classical POMDP with hidden state \(X_{k}\), the canonical information state is the conditional density \(b_{k}(x) = p_{X_{k} \mid H_{k}}(x \mid h_{k})\). In finite-state POMDPs, \(b_{k}\) is a probability vector on the simplex; in linear-Gaussian problems, it is the Kalman filter's conditional mean and covariance; in general nonlinear continuous-state models, it is an infinite-dimensional object. The construction below is exact whenever \(B_{k}\) is finite-dimensional and Markov; when it is not, the benchmark must be built on a deliberately designed finite-dimensional information state, as discussed in Section~\ref{subsec:pomdp-realization}.

\begin{theorem}[Information-state reduction]
\label{thm:info-reduction}
Let \(B_{k}\) be a sufficient information state in the sense of Definition~\ref{def:info-state}. Then the optimal output-feedback control problem is equivalent to a fully observed MDP with state \(B_{k}\). In particular, the optimal value can be written as a function \(v^{*}(b)\) satisfying
\begin{equation}
v^{*}(b) = \min_{a \in \mathbb{R}^{m}}\Bigl\{c(b, a) + \gamma \int_{\mathbb{B}} v^{*}(b_{+})\, p_{b}(b_{+} \mid b, a)\, db_{+}\Bigr\}.
\label{eq:info-bellman}
\end{equation}
If a measurable minimizer \(\pi^{*}(b)\) exists, then the output-feedback policy \(a_{k} = \pi^{*}(\phi_{k}(h_{k}))\) is optimal among all nonanticipative output-feedback policies.
\end{theorem}

The proof is in Appendix~\ref{app:pomdp-proofs}. The essential point is that no new dynamic-programming principle is needed: Definition~\ref{def:info-state} says that the transition kernel and the expected stage cost depend on the history only through \(B_{k}\), so the backward-induction argument for the fully observed Bellman equation applies verbatim once \(s\) is replaced by \(b\).

\subsection{Converse Optimality on the Information State}
\label{subsec:pomdp-converse}

Theorem~\ref{thm:info-reduction} reduces the output-feedback problem to a fully observed Markov problem on the information state. We can therefore apply the converse constructions of Sections~\ref{sec:mode2}--\ref{sec:ext-nonaffine} directly to \(b\). Throughout this subsection, we assume that the information-state dynamics is affine in the virtual input,
\begin{equation}
B_{k+1} = f(B_{k}) + g(B_{k})\, A_{k} + \Omega_{k+1},
\label{eq:info-dynamics}
\end{equation}
where \(\{\Omega_{k}\}\) is i.i.d., zero-mean, and has finite covariance \(\Sigma_{b}\), and that the stage cost is separable,
\begin{equation}
c(b, a) = q(b) + \rho(a),
\label{eq:info-cost}
\end{equation}
with \(q \ge 0\) and \(\rho\) strictly convex. These are the same structural assumptions used for the fully observed stochastic construction; they are now stated on \(b\) rather than on \(s\).

\begin{theorem}[Converse optimality on the information state]
\label{thm:info-converse}
Let the information-state dynamics \eqref{eq:info-dynamics} and cost \eqref{eq:info-cost} satisfy the assumptions above, and let \(v^{*} : \mathbb{R}^{d} \to \mathbb{R}\) be a candidate value function satisfying the same regularity and integrability conditions as in the fully observed stochastic construction (Section~\ref{sec:mode2}). If for every \(b \in \mathbb{R}^{d}\) there exists a policy \(\pi^{*}(b)\) such that
\begin{align}
\pi^{*}(b) \text{ is a Bellman argmin}: \quad \pi^{*}(b) &\in \arg\min_{a \in \mathbb{R}^{m}} \Bigl\{\rho(a) + \gamma \mathbb{E}\bigl[v^{*}(f(b) + g(b)a + \Omega)\bigr]\Bigr\},
\label{eq:info-stationarity} \\[2pt]
v^{*}(b) &= q(b) + \rho(\pi^{*}(b)) + \gamma \mathbb{E}\bigl[v^{*}(f(b) + g(b)\pi^{*}(b) + \Omega)\bigr],
\label{eq:info-bellman-identity}
\end{align}
then \(v^{*}\) is the optimal information-state value function, and the output-feedback policy \(A_{k} = \pi^{*}(B_{k})\) is optimal among all output-feedback policies.
\end{theorem}

The proof applies the Bellman-operator contraction argument of the fully observed stochastic theorem (Appendix~\ref{app:smooth-proofs}) on the information state. Because \(B_{k}\) is a sufficient statistic by Definition~\ref{def:info-state}, the value function of the output-feedback problem coincides with the value function of the Markov problem on \(b\), and optimality in the latter transfers to the former.

\subsection{Quadratic Information-State Construction}
\label{subsec:pomdp-quadratic}

The most useful specialization is again quadratic. Choose
\begin{equation}
q(b) = b^{\top} Q b, \qquad \rho(a) = a^{\top} R a, \qquad v^{*}(b) = b^{\top} P b + v_{0},
\label{eq:info-quadratic}
\end{equation}
where \(P \succ 0\), \(R \succ 0\), \(Q \succeq 0\), and \(P - Q \succeq 0\). Let
\begin{equation}
B_{\gamma}(b) = R + \gamma\, g(b)^{\top} P g(b), \qquad
H_{\gamma}(b) = \gamma\bigl(P^{-1} + \gamma\, g(b) R^{-1} g(b)^{\top}\bigr)^{-1}.
\label{eq:info-BH}
\end{equation}
Then, by the same completion-of-squares calculation as in the fully observed stochastic case, the optimal information-state feedback is
\begin{equation}
a^{*}(b) = -\gamma\, B_{\gamma}(b)^{-1} g(b)^{\top} P f(b),
\label{eq:info-oracle}
\end{equation}
and the Bellman identity reduces to the energy equation
\begin{equation}
b^{\top}(P - Q) b = f(b)^{\top} H_{\gamma}(b) f(b).
\label{eq:info-energy}
\end{equation}
The value offset is \(v_{0} = \frac{\gamma}{1-\gamma}\operatorname{tr}(P \Sigma_{b})\). Any drift of the metric-normalized form
\begin{equation}
f(b) = H_{\gamma}(b)^{-1/2} S(b) (P - Q)^{1/2} b, \qquad S(b)^{\top} S(b) = I,
\label{eq:info-drift}
\end{equation}
satisfies \eqref{eq:info-energy}, so the same parameterizations as in Section~\ref{sec:mode2} apply with \(s\) replaced by \(b\).

\begin{corollary}[Quadratic information-state benchmark]
\label{cor:info-benchmark}
Suppose \(B_{k}\) is a finite-dimensional sufficient information state satisfying \eqref{eq:info-dynamics}, and the cost is quadratic as in \eqref{eq:info-quadratic}. Choose any input matrix field \(g(b)\) and any orthogonal field \(S(b)\), and define \(f\) by \eqref{eq:info-drift}. Then \(v^{*}(b) = b^{\top} P b + v_{0}\) with \(v_{0} = \frac{\gamma}{1-\gamma}\operatorname{tr}(P \Sigma_{b})\) is the optimal information-state value function, and the optimal output-feedback controller is \(A_{k} = a^{*}(B_{k})\) with \(a^{*}\) given by \eqref{eq:info-oracle}.
\end{corollary}

The corollary is the POMDP analogue of the QG specialization used for ConverseArm, NVDEx, and ConverseQuad. Every construction of Section~\ref{subsec:families} and Section~\ref{subsec:quad-virtual} carries over once the physical state is replaced by a sufficient information state, and the certificate logic is unchanged.

\subsection{Realizing the Information State as an Output System}
\label{subsec:pomdp-realization}

The construction above is stated in terms of \(b\). For a learner-facing benchmark, \(b\) must be reconstructed from the observation--action--reward history rather than observed directly. This requires an output realization: a pair \((\eta, \tau)\) where \(\eta(o \mid b, a)\) is an observation density and
\begin{equation}
b_{+} = \tau(b, a, o_{+})
\label{eq:filter-update}
\end{equation}
is a deterministic filter recursion. The induced transition kernel on \(b\) is
\begin{equation}
\kappa(D \mid b, a) = \int_{\mathbb{O}} \mathbbm{1}\{\tau(b, a, o) \in D\}\, \eta(o \mid b, a)\, do,
\label{eq:induced-kernel}
\end{equation}
and the realization is compatible with \eqref{eq:info-dynamics} if this kernel equals the one induced by the noise density \(p_{\Omega}\).

\begin{proposition}[Innovation realization]
\label{prop:innovation}
Let the target information-state dynamics be \eqref{eq:info-dynamics} with noise density \(p_{\Omega}\). Take the observation space \(\mathbb{O} = \mathbb{R}^{d}\), reveal the innovation \(O_{k+1} = \Omega_{k+1}\), and define the filter update
\begin{equation}
\tau(b, a, o) = f(b) + g(b) a + o.
\label{eq:innovation-update}
\end{equation}
Then the pair \((\eta, \tau)\) with \(\eta(o \mid b, a) = p_{\Omega}(o)\) realizes \eqref{eq:info-dynamics}.
\end{proposition}

The innovation realization is mathematically clean and gives an exact output process whose information state is finite-dimensional by design. It does not, however, model a physically natural hidden-state process: it is a deliberately designed benchmark whose sufficient statistic is known in closed form. This is the design principle of the POMDP extension. The construction is not the most general nonlinear POMDP, and it does not claim a nonlinear separation theorem; it constructs a benchmark for which the optimal output-feedback policy is certified in the learner's own information set.

\paragraph{What the learner sees.}
The environment maintains \(b_{k}\) internally. At reset, the learner receives \(o_{0}\) and the environment returns rewards \(r_{k} = -c(b_{k}, a_{k})\). The field \texttt{info} returned by the environment must not contain \(b_{k}\), the drift \(f\), the input map \(g\), the matrices \(P, Q, R\), the orthogonal field \(S\), or the filter update \(\tau\). If the source code of the environment is available to the learner, a strong student or an LLM can recover the information-state recursion; a serious benchmark therefore keeps the final evaluation seeds and the private filter parameters outside the learner's notebook.

\paragraph{Making the interface less transparent.}
The innovation realization is exact but transparent if \(f\), \(g\), and \(\tau\) are revealed. A less transparent variant composes the observation with a smooth bijection \(\chi : \mathbb{R}^{d} \to \mathbb{R}^{d}\):
\begin{equation}
o_{k+1} = \chi(\Omega_{k+1}), \qquad \tau(b, a, o) = f(b) + g(b)a + \chi^{-1}(o).
\label{eq:bijection-observation}
\end{equation}
The induced kernel on \(b\) is unchanged, so the certificate is preserved.

\subsection{Benchmark Protocol}
\label{subsec:pomdp-protocol}

A student-facing POMDP benchmark is implemented as follows. The environment stores \(b_{k}\) internally and returns only \(o_{k}\), the reward \(r_{k} = -c(b_{k}, a_{k})\), and termination flags. It never returns \(b_{k}\). The private evaluator computes the certified action \(a^{*}(b_{k})\) from \eqref{eq:info-oracle} and the certified value \(v^{*}(b_{k}) = b_{k}^{\top} P b_{k} + v_{0}\).

The learner-facing API is
\begin{equation}
o_{0} = \mathrm{reset}(), \qquad (o_{k+1}, r_{k}, \mathrm{done}, \mathrm{info}) = \mathrm{step}(a_{k}),
\label{eq:pomdp-api}
\end{equation}
where the \texttt{info} dictionary may contain diagnostic signals but must not contain \(b_{k}\) or any of the certificate parameters. The evaluation objective is the discounted cost
\begin{equation}
J^{\pi}(b_{0}) = \mathbb{E}_{\pi}\Bigl[\sum_{k=0}^{T-1} \gamma^{k} c(B_{k}, A_{k}) \,\Big|\, B_{0} = b_{0}\Bigr],
\label{eq:pomdp-objective}
\end{equation}
truncated at horizon \(T\). The certified reference is
\begin{equation}
J^{*}(b_{0}) = v^{*}(b_{0}) - \gamma^{T} \mathbb{E}\bigl[v^{*}(B_{T}) \mid B_{0} = b_{0}, A_{k} = a^{*}(B_{k})\bigr],
\label{eq:pomdp-reference}
\end{equation}
and the normalized relative suboptimality is
\begin{equation}
\mathrm{RS}(b_{0}, \pi) = \frac{J^{\pi}(b_{0}) - J^{*}(b_{0})}{|J^{*}(b_{0})| + \epsilon}.
\label{eq:pomdp-RS}
\end{equation}
CRN are used as in Section~\ref{subsec:oracle_protocol}: all evaluated policies see the same start beliefs and the same innovation streams.

A useful decomposition follows directly from \eqref{eq:info-decomposition}. When a full-state reference is available (for example, from a linear-Gaussian hidden model), the evaluator reports both \(J^{\pi} - J_{\mathrm{out}}^{*}\) and the analytic information penalty \(J_{\mathrm{out}}^{*} - J_{\mathrm{full}}^{*}\) separately. This separates learning error from information loss and is what distinguishes an information-consistent POMDP benchmark from one whose oracle exploits hidden information.

\subsection{LQG Reference and Information Penalty}
\label{subsec:pomdp-lqg}

To make the information-penalty decomposition quantitative, consider the noisy-sensor LQG reference
\begin{equation}
X_{k+1} = A X_{k} + G A_{k} + W_{k+1}, \qquad Y_{k} = C X_{k} + V_{k},
\label{eq:lqg-dynamics}
\end{equation}
with process covariance \(W = 10^{-4} I\), sensor variance \(\sigma_{y}^{2}\), quadratic stage cost \(X^{\top} Q X + R A^{2}\), discount \(\gamma \in (0,1)\), and parameters
\[
A = \begin{pmatrix} 1 & 0.2 \\ 0 & 1 \end{pmatrix}, \quad
G = \begin{pmatrix} 0.02 \\ 0.2 \end{pmatrix}, \quad
C = I, \quad Q = \mathrm{diag}(1, 0.1), \quad R = 0.1.
\]
The Kalman filter is the finite-dimensional information state \(b_{k} = (\mu_{k}, E)\), where \(\mu_{k} = \mathbb{E}[X_{k} \mid H_{k}]\) and \(E\) is the stationary posterior covariance solving
\[
\Xi = A E A^{\top} + W, \qquad E = \Xi - \Xi C^{\top}(C \Xi C^{\top} + \sigma_{y}^{2})^{-1} C \Xi.
\]
The controller Riccati equation \(P = Q + \gamma A^{\top} P A - \gamma^{2} A^{\top} P G (R + \gamma G^{\top} P G)^{-1} G^{\top} P A\) determines the optimal full-state value
\[
V_{\mathrm{full}}^{*}(x) = x^{\top} P x + \frac{\gamma}{1-\gamma}\operatorname{tr}(P W).
\]
Averaged over the same initial belief, \(\mathbb{E} V_{\mathrm{full}}^{*}(X_{0}) = \mu_{0}^{\top} P \mu_{0} + \operatorname{tr}(P E) + \frac{\gamma}{1-\gamma}\operatorname{tr}(P W)\). The optimal output-feedback value satisfies
\[
v^{*}(\mu) = \mu^{\top} P \mu + \operatorname{tr}(P E) + \frac{\gamma}{1-\gamma}\operatorname{tr}(P W) + \frac{\gamma}{1-\gamma}\operatorname{tr}(P K^{\top} M K E),
\]
where \(K\) is the LQG gain and \(M = R + \gamma G^{\top} P G\). The information penalty is therefore
\begin{equation}
J_{\mathrm{out}}^{*} - J_{\mathrm{full}}^{*} = \frac{\gamma}{1-\gamma}\operatorname{tr}\bigl(P K^{\top} M K E\bigr),
\label{eq:lqg-penalty}
\end{equation}
a closed-form expression that scales with the posterior covariance \(E\), i.e., with sensor uncertainty. 

The exact result is Theorem~\ref{thm:info-converse}: if \(B_{k}\) is a finite-dimensional sufficient information state with affine dynamics \eqref{eq:info-dynamics}, the quadratic construction of Section~\ref{subsec:pomdp-quadratic} certifies \(v^{*}\) and \(a^{*}\) exactly, and the output-feedback policy is optimal in the learner's information class. What the construction does not do is solve arbitrary nonlinear continuous-state POMDPs. For a general hidden-state model, the belief density \(x \mapsto p_{X_{k} \mid H_{k}}(x \mid h_{k})\) is infinite-dimensional, and the value function is a functional on densities rather than a matrix on \(\mathbb{R}^{d}\). The matrix formulas of this section apply only after a finite-dimensional information state has been specified. This is a design principle, not a defect: a POMDP benchmark is useful precisely when its sufficient statistic is known, and the construction makes the choice explicit.

\paragraph{Forward link.}
The three extensions of Sections~\ref{sec:ext-nonaffine}--\ref{sec:ext-pomdp} each change one element of the framework: the action interface and disturbance law in Extension~I, the regularity class of the value function in Extension~II, and the observation interface in Extension~III. The certificate logic is unchanged.

\section{Discussion}
\label{sec:discussion}

The framework trades broad ecological coverage for calibration: converse-generated fixtures provide known-answer control problems with measurable optimality gaps, not replacements for open-ended simulators or real-world platforms. The families studied here exhibit nonlinear drift, state-dependent actuation, delayed input channels, and stiff closed-loop behaviour, with difficulty tunable via dimension, dynamic extension, coupling strength, control authority, and disturbance structure. They remain constructed tasks; strong performance indicates diagnostic capability under certified dynamics, not deployment readiness.

\paragraph{What each extension covers, and what it does not.}
Extension~I (Section~\ref{sec:ext-nonaffine}) shows that the certificate survives a nonlinear raw-action coordinate and a non-Gaussian finite-covariance disturbance law. What it does not cover is state-dependent, heteroscedastic, or temporally correlated noise, all of which would break the moment-only argument behind \eqref{eq:moment-only}. The reported ConverseQuad results are a single configuration in $\eta$ and $\epsilon$; a systematic sweep over actuator nonlinearity and mixture weight is left to future work.

Extension~II (Section~\ref{sec:ext-nonsmooth}) gives a viscosity--Clarke certificate for value functions that are locally Lipschitz but not differentiable. The theory is continuous-time; the numerical experiments of Section~\ref{subsec:nonsmooth-experiments} use an explicit Euler discretization, which inherits the discretization confound of Mode~I. A native discrete-time non-smooth construction would eliminate this confound, but the discrete-time analogue of the viscosity inequality is not standard and is left open. The supersolution condition couples the gain matrix $K$ to the geometry of the Clarke subdifferential, and the scalar-gain version requires the denominator \eqref{eq:cmin-positive} to be positive.

Extension~III (Section~\ref{sec:ext-pomdp}) constructs a finite-dimensional information-state benchmark whose optimal output-feedback policy is certified in the learner's own information set. The construction does not solve arbitrary nonlinear continuous-state POMDPs. For a general hidden-state model, the belief density is infinite-dimensional, and the value function is a functional on densities rather than a matrix on $\mathbb{R}^{d}$. The innovation realization of Proposition~\ref{prop:innovation} is exact but not physically natural: it is a deliberately designed benchmark whose sufficient statistic is known by construction, not a solution to the general partial-observability problem.

\section{Conclusion}
\label{sec:conclusion}

We presented a ground-truth-first framework for reinforcement learning benchmarking based on converse optimality. The central idea is to make the optimum part of the benchmark specification rather than an unknown target outside it. This supports two complementary modes: continuous-time generators for sensitivity and robustness studies, and native discrete-time fixtures for oracle-referenced comparison by optimality gap, regret, Bellman residual, and policy error. Continuous-time and discrete-time generators play distinct roles: the former provide physically interpretable probes but inherit simulator-step dependence, while the latter remove timestep ambiguity for absolute comparison.

The paper combines a review of the converse-optimality lineage with three extensions of the certificate logic. Extension~I extends the construction to non-affine raw actions and finite-covariance non-Gaussian disturbances, with the ConverseQuad aerial family as the concrete instantiation. Extension~II replaces the smooth HJB certificate with a viscosity--Clarke condition, enabling certified benchmarks with kinked value functions and switching surfaces, and reports a discrete-time instantiation on a piecewise-quadratic family. Extension~III reduces partially observed evaluation to converse optimality on a sufficient information state, certifies the optimal output-feedback policy in the learner's information class, and supplies an analytic LQG information penalty that separates learning error from information loss.

\bibliography{references}
\bibliographystyle{unsrt}

\appendix

\section{Reproducibility and Released Assets}
\label{app:released-assets}

The framework is intended to be reproducible as a seedable benchmark generator rather than only as a fixed list of hand-authored environments. The public project page: \url{https://gtfrl.github.io/Benchmarking-of-Reinforcement-Learning-via-Converse-Optimality/} hosts the paper, code links, fixtures, experimental summaries, and visual diagnostics. A complete release contains generator code, random seeds for sampled fixtures, validation grids, accepted and rejected instance metadata, CRN evaluation schedules, baseline configurations, software versions, and scripts for reproducing the tables.

For the experiments summarized in the main text, the relevant reproducibility objects are the certified fixture parameters, the validation tolerance, the horizon, the start-state distribution, the disturbance streams, and the baseline hyperparameter files. Since the oracle is explicit, independent reproduction can verify both the reported returns and the Bellman consistency of each exported environment.

\section{Impact Statement}
\label{app:impact}

This work aims to improve the reliability and interpretability of RL experimentation. Its main positive impact is methodological: it can help researchers distinguish relative performance from distance to optimality and design more diagnostic experiments. Potential negative impact is limited but not absent. If known-answer fixtures are overused as narrow leaderboards, they could encourage overfitting to generated families or create unwarranted confidence about deployment readiness. We therefore frame the method as a complement to broader benchmarks and real-world validation, not a substitute for them.

\section{Proofs for the Smooth Theory}
\label{app:smooth-proofs}

A benchmark fixture is useful only if we can verify that a given pair \((V,\pi)\) is indeed optimal.  The natural verification tool is the \emph{Bellman residual}. For any candidate value function \(V\),
\[
  \mathcal{R}_{V}(s,a) \; \coloneqq \; c(s,a) + \gamma \int_{\mathcal{S}} V(s') \, P(\mathrm{d}s'\mid s,a) - V(s).
  \label{eq:residual}
\]
The residual measures how much \(V\) fails to satisfy the Bellman equation at a given state‑action pair.  If \(V = V^{\star}\), then \(\inf_{a}\mathcal{R}_{V}(s,a) = 0\) for all \(s\) and the infimum is attained by any optimal action.  More generally, the residual provides a pointwise certificate that can be checked without solving the full MDP.  Proposition 1 makes this precise.

\begin{definition}[Bellman operators] For a given MDP $(\mathcal{S},\mathcal{A},P,c,\gamma)$, we define the \emph{optimal Bellman operator} $T$ and the \emph{policy evaluation operator} $T^{\pi}$ as follows. For any function $V:\mathcal{S}\to\mathbb{R}$,
\[
    (T V)(s)\!\coloneqq\!\inf_{a\in\mathcal{A}}\!\Big\{\!c(s,a)\!\!+\!\gamma \!\!\!\int_{\mathcal{S}} \!\!\!V(s') \!\, P(\mathrm{d}s'\!\!\mid\!\!s,\!a) \!\Big\}, \quad(T^{\pi} V)(s)\! \coloneqq \!c(s,\pi(s)) \!+\!\gamma\!\!\!\int_{\mathcal{S}} \!\!\!V(s') \!\, P(\mathrm{d}s' \!\!\mid\!\! s,\!\pi(s)).
\]
\end{definition}

The Bellman operators formalize two fundamental questions in dynamic programming. The \emph{optimal Bellman operator} $T$ asks: ``What is the smallest total cost that can be achieved from state $s$, assuming that the future cost from the next state onward is given by $V$?''. The true optimal value function $V^{\star}$ is the unique fixed point of this operator, i.e.\ $V^{\star} = T V^{\star}$; this is the Bellman equation. The \emph{policy evaluation operator} $T^{\pi}$ asks: ``If I follow policy $\pi$ forever, what total cost will I incur from state $s$ if the future cost is evaluated according to $V$?''. The value function of $\pi$ is its unique fixed point: $V^{\pi} = T^{\pi} V^{\pi}$.

\begin{proposition}[Optimality certificate] Let $\mathbb{B}_w$ be the weighted sup-norm Banach space defined by $\mathbb{B}_w \coloneqq \{V:\mathcal{S} \to \mathbb{R} \text{ measurable } | \|V\|_w < \infty\}$, with $\|V\|_w = \sup_{s\in\mathcal S}\frac{|V(s)|}{w(s)}$ and $w$ some suitable weight function. Let $T$ and $T^\pi$ be the Bellman operators from Definition 1, well-defined on $\mathbb{B}_w$, and admitting unique fixed points $V^\star$ and $V^\pi$. For any locally Lipschitz function $V \in \mathbb{B}_w$ and any admissible measurable policy $\pi$, the following are equivalent:
  \begin{enumerate}
    \item $V = V^{\star} = V^{\pi}$; i.e., $V$ is the true optimal value function and $\pi$ is an optimal policy.
    \item For all $s\in\mathcal{S}$ and all $a\in\mathcal{A}$, we have $\mathcal{R}_{V}(s,\pi(s)) = 0$ and $\mathcal{R}_{V}(s,a) \ge 0$.
  \end{enumerate}
\end{proposition}

\begin{proof} We prove the equivalence.

\noindent{$(1)\Rightarrow(2)$.} Assume that $V = V^\star = V^\pi$, and that $\pi$ is an optimal policy. Since $V^\star$ is the optimal value function, it satisfies the Bellman equation
\[
V^\star(s) = \inf_{a\in\mathcal A}\left\{c(s,a) + \gamma \int_{\mathcal S} V^\star(s') \,P(\mathrm ds'\mid s,a) \right\}.
\]

Therefore, for every $(s,a)\in\mathcal S\times\mathcal A$,
\[ c(s,a) + \gamma\int_{\mathcal S}V^\star(s')\,P(\mathrm ds'\mid s,a)-V^\star(s)\ge 0,
\]
that is, $\mathcal R_V(s,a)\ge0$. Since $\pi$ is optimal and $V^\pi=V^\star$, the policy $\pi$ attains the Bellman infimum at every state:
\[
V^\star(s)=c(s,\pi(s))+\gamma\int_{\mathcal S}V^\star(s')\,P(\mathrm ds'\mid s,\pi(s)).
\]

Hence, $\mathcal R_V(s,\pi(s))=0$. Thus condition (2) holds.\\

\noindent{$(2)\Rightarrow(1)$.} Assume now that for all $s\in\mathcal S$ and all $a\in\mathcal A$, $\mathcal R_V(s,a)\ge0$, and $\mathcal R_V(s,\pi(s))=0$. By definition of the residual, the inequality $\mathcal R_V(s,a)\ge0$ means that
\[
V(s)\le c(s,a)+\gamma\int_{\mathcal S}V(s')\,P(\mathrm ds'\mid s,a)\qquad \forall a\in\mathcal A.
\]
Taking the infimum over $a$ gives $V(s)\le (TV)(s)$. On the other hand, evaluating the residual at the action $\pi(s)$ and using $\mathcal R_V(s,\pi(s))=0$, we obtain
\[
V(s)=c(s,\pi(s))+\gamma\int_{\mathcal S}V(s')\,P(\mathrm ds'\mid s,\pi(s))=(T^\pi V)(s).
\]

Since
\[
(TV)(s)=\inf_{a\in\mathcal A}\left\{c(s,a)+\gamma\int_{\mathcal S}V(s')\,P(\mathrm ds'\mid s,a)\right\},
\]
and the action $a=\pi(s)$ achieves equality, we conclude that $(TV)(s)=(T^\pi V)(s)=V(s)$. Therefore, $V=TV$ and $V=T^\pi V$.

The Bellman operators $T$ and $T^\pi$ each admit unique fixed points. Hence $V$ must coincide simultaneously with the optimal value function and the value function associated with $\pi$, so that $V=V^\star=V^\pi$. Therefore $\pi$ is an optimal policy and condition (1) holds. This completes the proof.
\end{proof}

Condition 2 of Proposition 1 requires the residual to vanish on the candidate policy and to be non‑negative everywhere else. Together they force \(V\) to be the fixed point of the optimal Bellman operator and \(\pi\) to attain the minimum.  When \(V\) is smooth and the action space is unconstrained, the residual inequality reduces to the classical Hamilton–Jacobi–Bellman equation together with a first‑order stationarity condition; when \(V\) is only locally Lipschitz, the same logic applies with Clarke subdifferentials replacing classical gradients (see Appendix~\ref{app:nonsmooth-proofs}). With the certificate in hand, we can now define the central object of the benchmarking framework.

\subsection{Results for the Stochastic Discrete-Time Setting}

To guarantee that the stochastic Bellman equation is well-posed and that the first-order optimality condition is valid, we consider the following set of assumptions \citep{ibrahim2026benchmarking}.

\begin{assumption}[Well-posedness and regularity]\label{ass:wellposed} The following holds on $\mathbb{R}^n$:
\begin{enumerate}[leftmargin=1.35em,itemsep=2pt]
\item $q:\mathbb{R}^n\to\mathbb{R}_+$ is continuous.
\item $\rho:\mathbb{R}^m\to\mathbb{R}_+$ is $C^2$ and $\mu$-strongly convex: $\nabla^2\rho(a)\succeq \mu I$ for all $a$.
\item $f,g$ are locally Lipschitz with at most polynomial growth.
\item The candidate value function $V:\mathbb{R}^n\to\mathbb{R}$ is $C^2$ and convex.
\end{enumerate}
\end{assumption}

\begin{assumption}[Growth and integrability]\label{ass:integrability} There exist constants $\eta>0$ and an integer $\ell\ge 1$ such that for all $s\in\mathbb{R}^n$,
\[
|V(s)|+\|\nabla V(s)\|+\|\nabla^2 V(s)\|\le \eta\big(1+\|s\|^\ell\big),
\]
and, for each fixed $(s,a)\in\mathbb{R}^n\times\mathbb{R}^m$, the expectations $\mathbb{E}[\|\nabla V(f(s)+g(s)a+w)\|]$ and $\mathbb{E}[\|\nabla^2 V(f(s)+g(s)a+w)\|]$ are finite under $w\sim\mathcal{N}(0,\Sigma)$.
\end{assumption}

\begin{lemma}[Expectation-gradient interchange]\label{lem:interchange}
Under Assumptions 1 and 2, for all $s \in \mathbb{R}^n$ and $a \in \mathbb{R}^m$,
\[
\nabla_a\mathbb{E}\!\left[V(f(s)+g(s)a+w)\right]
=\mathbb{E}\!\left[g(s)^\top\nabla V(f(s)+g(s)a+w)\right].
\]
\end{lemma}

\begin{proof} Fix $s \in \mathbb{R}^n$ and define $\phi(a,w) = V(f(s) + g(s)a + w)$. By the chain rule, we have
\[
\nabla_a \phi(a,w) = g(s)^{\top} \nabla V(f(s) + g(s)a + w)
\]
To interchange the derivative and the expectation, we verify the conditions of the dominated convergence theorem. Let $e_1,\dots, e_m$ be the canonical basis vectors of $\mathbb{R}^m$. Without loss of generality, we choose a sequence $\{t_{\nu}\}$ converging to $0$, such that we can define sequences $\{a_\nu^{(i)}\}$ through $a^{(i)}_{\nu} = a + t_{\nu}e_i$ converging to $a$ along each coordinate axis $i$.By the mean value theorem, for each axis $i$, there exists $\xi^{(i)}_{\nu}$ on the line segment between $a^{(i)}_{\nu}$ and $a$ such that 
\begin{gather*}
\phi(a^{(i)}_{\nu}, w) - \phi(a,w) = \nabla_a \phi(\xi^{(i)}_{\nu},w)\cdot(a_{\nu}^{(i)} -a),\\ \text{and }
 |\phi(a^{(i)}_{\nu},w) - \phi(a,w)| \leq \|\nabla_a \phi(\xi^{(i)}_{\nu},w)\| \cdot \|a^{(i)}_{\nu} - a\|.
\end{gather*}

Now, let us bound $\|\nabla_a \phi(\xi^{(i)}_{\nu},w)\|$ for each axis $i$. From Assumption 2, we obtain:
\[
    \|\nabla_a \phi(\xi^{(i)}_{\nu},w)\| \leq \|g(s)\| \cdot \|\nabla V(f(s) + g(s)\xi^{(i)}_{\nu} + w)\|,
\]
and hence,
\[
\|\nabla_a \phi(\xi^{(i)}_{\nu},w)\| \le \|g(s)\| \cdot \eta(1 + \|f(s) + g(s)\xi^{(i)}_{\nu} + w\|^{\ell}).
\]

Since $\xi^{(i)}_{\nu}$ lies in a bounded neighborhood of $a$ (for $\nu$ sufficiently large), and $f,g$ have at most polynomial growth (Assumption 1.3), there exist constants $C_1(s)$ and $C_2(s)$ such that
\begin{equation*}
\|f(s) + g(s)\xi^{(i)}_{\nu} + w\| \leq C_1(s) + C_2(s)\|\xi^{(i)}_{\nu}\| + \|w\| \leq C(s,a) + \|w\|,
\end{equation*}
where $C(s,a)$ is independent of $\nu$ and $\xi_{\nu}^{(i)}$ (i.e., uniform over the neighborhood of $a$). This uniformity ensures the subsequent growth bound 
\begin{equation*}
\|\nabla_a \phi(\xi^{(i)}_{\nu},w)\| 
\leq \eta\|g(s)\|(1 + (C(s,a) + \|w\|)^{\ell})
\leq K(s,a)(1 +\|w\|^{\ell}),
\end{equation*} 
where $K(s,a)(1 +\|w\|^{\ell})$ is independent of $\nu$. This provides a single bound applicable to the entire sequence. The growth bound $\beta :w \mapsto K(s,a)(1+\|w\|^l)$ is measurable. 

Let $\mu_\Sigma$ be the Gaussian measure. By definition of the expectation, the Lebesgue integral of $\beta$ weighted by the Gaussian density is 
\[\int_{\mathbb{R}^n}\beta(w) d\mu_{\Sigma}(w)=\mathbb{E}[K(s,a)(1+\|w\|^l)].\] 
The linearity of the expectation yields 
\[
\mathbb{E}[K(s,a)(1+\|w\|^l)] = K(s,a)(1+\mathbb{E}[\|w\|^l]),
\]
and hence,
\[\int_{\mathbb{R}^n}\beta(w) d\mu_{\Sigma}(w) = K(s,a)
\left(
1+\int_{\mathbb{R}^n}\|w\|^{\ell} d\mu_{\Sigma}(w)
\right).
\]

As all Gaussian moments are finite, we have:
\[
\int_{\mathbb{R}^n}\|w\|^{\ell} d\mu_{\Sigma}(w) = \mathbb{E}[\|w\|^{\ell}] < \infty
\]
for all $\ell>0$. Hence, 
\[
\int_{\mathbb{R}^n}\beta(w) d\mu_{\Sigma}(w) < \infty,
\]
and the growth bound $\beta$ is Lebesgue-integrable with respect to the Gaussian measure. Since $\xi^{(i)}_{\nu} \to a$, continuity gives $\nabla_a \phi(\xi^{(i)}_{\nu},w) \to \nabla_a \phi(a,w)$ pointwise. Hence, the bound 
\[\|\nabla_a \phi(\xi^{(i)}_{\nu},w)\| \le K(s,a)(1+\|w\|^\ell)\]
provides a suitable dominating function, and by the dominated convergence theorem, we eventually obtain
\[
\lim_{\nu\to\infty} \mathbb{E}[\nabla_a \phi(\xi^{(i)}_{\nu},w)] = \mathbb{E}[\nabla_a \phi(a,w)].\quad (*)
\]

On the other hand, by definition of the gradient, we have 
\[
\nabla_a \mathbb{E}[\phi(a, w)] = \Big(\frac{\partial}{\partial a_1}\mathbb{E}[\phi(a, w)],\dots,\frac{\partial}{\partial a_m}\mathbb{E}[\phi(a, w)]\Big).
\]
The directional derivative along the $i$-th axis is 
\[
\frac{\partial}{\partial a_i}\mathbb{E}[\phi(a,w)] = \lim_{\nu\to\infty} \mathbb{E}\Big[\frac{\phi(a+t_{\nu}e_i,w) - \phi(a,w)}{t_{\nu}}\Big].
\]
Along the $i$-th axis, we also have 
\[
\frac{\phi(a+t_{\nu}e_i,w) - \phi(a,w)}{t_{\nu}} = \nabla_a \phi(\xi^{(i)}_{\nu}, w)\cdot e_i.
\] 
Combining these expressions yields 
\begin{equation*}
\lim_{\nu\to\infty} \mathbb{E}\Big[\frac{\phi(a+t_{\nu}e_i,w) - \phi(a,w)}{t_{\nu}}\Big] =\lim_{\nu\to\infty} \mathbb{E}[\nabla_a \phi(\xi^{(i)}_{\nu}, w)\cdot e_i] = \mathbb{E}[\nabla_a \phi(a,w)\cdot e_i],
\end{equation*}
where the last equality follows from the dominated convergence argument $(*)$ above. But, $\nabla_a \phi(a, w) \cdot e_i = \frac{\partial}{\partial a_i}\phi(a, w)$. Therefore, $\frac{\partial}{\partial a_i}\mathbb{E}[\phi(a,w)] = \mathbb{E}[\frac{\partial}{\partial a_i}\phi(a, w)]$. Finally, we get 
\begin{gather*}
    \nabla_a \mathbb{E}[\phi(a,w)] = \Big(\mathbb{E}[\frac{\partial}{\partial a_1}\phi(a,w)], \dots, \mathbb{E}[\frac{\partial}{\partial a_m}\phi(a,w)]\Big) =\\= \mathbb{E}\Big(\frac{\partial \phi}{\partial a_1}(a,w), \dots,\frac{\partial \phi}{\partial a_m}(a,w)\Big)= \mathbb{E}[\nabla_a \phi(a,w)].
\end{gather*}
\end{proof}

\begin{theorem}[Stochastic converse optimality]\label{thm:main} Fix a stage cost $c(s,a) = q(s) + \rho(a)$ with $q,\rho$ satisfying Assumption 1, discount factor $\gamma \in (0,1)$, noise covariance $\Sigma \succ 0$, input coupling $g: \mathbb{R}^n \to \mathbb{R}^{n\times m}$, and a candidate value function $V^* \in C^2(\mathbb{R}^n)$ satisfying Assumption 2. Then there exists a drift $f: \mathbb{R}^n \to \mathbb{R}^n$ (i.e. a system $s_{k+1} = f(s_k) + g(s_k)a_k + w_k$ with $w_k \sim \mathcal{N}(0,\Sigma)$) for which $V^*$ is the optimal value function if and only if there exists a policy $\pi^*: \mathbb{R}^n \to \mathbb{R}^m$ such that for every $s \in \mathbb{R}^n$:
\begin{gather}
\nabla_a\Big[\rho(a)
+\gamma\,\mathbb{E}\!\left[V^*\big(f(s)+g(s)a+w\big)\right]\Big]\Big|_{a=a^*}=0,\label{eq:fooc}\\
V^*(s)=q(s)+\rho(a^*)+\gamma\mathbb{E}\!\left[V^*\big(f(s)+g(s)a^*+w\big)\right].\label{eq:bellman_id}
\end{gather}
When such $\pi^*$ exists, it is an optimal policy and $V^*$ satisfies the stochastic Bellman equation given in Subsection 5.2 for all $s \in \mathbb{R}^n$.
\end{theorem}

\begin{proof}\textbf{Sufficiency.} Suppose there exist a policy $\pi^*: \mathbb{R}^n \to \mathbb{R}^m$ and a drift $f: \mathbb{R}^n \to \mathbb{R}^n$ such that for every $s \in \mathbb{R}^n$, equations \eqref{eq:fooc} and \eqref{eq:bellman_id} hold, and $\pi^*(s_k) = a_k^*$ for all $s_k \in \mathbb{R}^n$. For a fixed $s\in \mathbb{R}^n$ we define the Bellman integrand:
\[
\Phi_s(a) := q(s) + \rho(a) + \gamma \mathbb{E}\big[V^*(f(s) + g(s)a + w)\big]
\]
By the Lemma, we have:
\[
\nabla_a \Phi_s(a) = \nabla_a \rho(a) + \gamma \mathbb{E}\big[g(s)^\top \nabla V^*(f(s) + g(s)a + w)\big].
\] 
Equations (4)-(5) give $\nabla_a \Phi_s(a)|_{a=a^*} = 0$ and imply that the Bellman identity holds with action $a^*$. In particular, 
\[
V^*(s) = q(s)+\rho(a^*)+\gamma \mathbb{E}\big[V^*(f(s)+g(s)a^*+w)\big] = \Phi_s(a^*). 
\]
Since the Bellman equation involves the minimization of $\Phi_s(a)$ with respect to $a$, it follows that this minimum is attained at $a^*$. Hence: 
\begin{align*} 
\min_{a \in \mathbb{R}^m} \Phi_s(a) &= \Phi_s(a^*) \\
&= q(s) + \rho(a^*) + \gamma \mathbb{E}\big[V^*(f(s) + g(s)a^* + w)\big]\\ 
&= V^*(s). 
\end{align*} 
Thus, $\pi^*$ is an optimal policy and $V^*$ is the optimal value function for the drift $f$.

Now fix any stationary policy $\pi$ and define its Bellman operator:
\[
(T^\pi V)(s):=q(s)+\rho(\pi(s))+\gamma\mathbb{E}\!\left[V\big(f(s)+g(s)\pi(s)+w\big)\right].
\]
By definition of the minimum, $\min_a \Phi_s(a)\le \Phi_s(\pi(s))$, hence $V^*(s)\le (T^\pi V^*)(s)$ for all $s$. Applying $T^\pi$ to both sides of $V^* \le T^\pi V^*$ and using the monotonicity of $T^\pi$ (which follows from the positivity of costs), we obtain $T^\pi V^* \le (T^\pi)^2 V^*$. By induction, $V^* \le (T^\pi)^k V^*$ for all $k \in \mathbb{N}$. Since $T^\pi$ is a $\gamma$-contraction in an appropriate weighted norm (e.g. the sup norm under bounded costs, or a weighted norm under the discounted criterion), then by the Banach fixed-point theorem, $(T^\pi)^k V^* \to V^\pi$, the value function of policy $\pi$. Taking limits preserves the inequality, yielding $V^* \le V^\pi$ on $\mathbb{R}^n$, proving the optimality of $V^*$ and $\pi^*$.

\textbf{Necessity.} Assume $V^*$ optimal, so for each $s$ the minimizer $\pi^*(s)$ of the Bellman equation exists. By the measurable selection theorem \citep{shreve1979universally}, $\pi^*$ is chosen to be measurable, and hence admissible as a feedback policy. At a minimum, we have $\nabla_a\Phi_s(\pi^*(s))=0$. Lemma 1 justifies interchanging the derivative and the expectation, yielding \eqref{eq:fooc}. The Bellman identity \eqref{eq:bellman_id} holds by the definition of $\pi^*(s) = a^*$ as the minimizer of the Bellman equation.
\end{proof}

\subsection{Continuous‑Time Converse Derivation}
\label{app:ct_derivation}
Recall the continuous‑time control‑affine system \(\dot s = f(s) + g(s)a\) with running cost \(c(s,a)\). For a candidate value function \(V\in C^{1}\), the Hamiltonian is
\[
  H(s,a) = c(s,a) + \nabla V(s)^{\!\top}\!\bigl(f(s) + g(s)a\bigr).
\]

The generator~\citep{yamerenko2025generating} prescribes two chosen smooth fields \(H:\mathbb R^{n}\to\mathbb R^{n}\) and \(G:\mathbb R^{n}\to\mathbb R^{n\times m}\) and defines
\begin{align}
  g(s)   &= G(s), \label{eq:ct_g}\\
  f(s)   &= \frac14\,G(s)\,\bigl(\mathcal L_{G}V(s)\bigr)^{\!\top} - H(s), \label{eq:ct_f}\\
  c(s,a) &= \mathcal L_{H}V(s) + \|a\|_{2}^{2}. \label{eq:ct_c}
\end{align}
Here \(\mathcal L_{G}V(s) = \nabla V(s)^{\!\top} G(s)\) is a row vector in \(\mathbb R^{1\times m}\), and \(\mathcal L_{H}V(s) = \nabla V(s)^{\!\top} H(s)\) is a scalar. Substituting \eqref{eq:ct_g}--\eqref{eq:ct_c} into the Hamiltonian,
\[
\begin{aligned}
  H(s,a) &=
  \mathcal L_{H}V(s) + \|a\|_{2}^{2}
  + \nabla V(s)^{\!\top}\!\Bigl(
      \frac14\,G(s)\bigl(\mathcal L_{G}V(s)\bigr)^{\!\top} - H(s)
      + G(s)\,a
    \Bigr) \\[2mm]
  &= \mathcal L_{H}V(s) + \|a\|_{2}^{2}
     + \frac14\,\nabla V(s)^{\!\top} G(s)\bigl(\mathcal L_{G}V(s)\bigr)^{\!\top}
     - \nabla V(s)^{\!\top} H(s)
     + \nabla V(s)^{\!\top} G(s)\,a .
\end{aligned}
\]

Because \(\mathcal L_{G}V(s) = \nabla V(s)^{\!\top} G(s)\), we have
\[
  \nabla V(s)^{\!\top} G(s)\bigl(\mathcal L_{G}V(s)\bigr)^{\!\top}
  = \|\mathcal L_{G}V(s)\|_{2}^{2},
\qquad
  \nabla V(s)^{\!\top} H(s) = \mathcal L_{H}V(s).
\]

Therefore the terms \(\mathcal L_{H}V(s)\) cancel, leaving
\[
  H(s,a) = \|a\|_{2}^{2}
           + \frac14\,\|\mathcal L_{G}V(s)\|_{2}^{2}
           + \bigl(\mathcal L_{G}V(s)\bigr)\,a .
\]

The right‑hand side is a quadratic form in \(a\). Completing the square,
\[
  H(s,a) = \bigl\|\, a + \tfrac12\bigl(\mathcal L_{G}V(s)\bigr)^{\!\top} \bigr\|_{2}^{2}.
\]

The unique minimizer of this expression is obtained by making the squared norm vanish, i.e.
\[
  a^{\star}(s) = -\tfrac12\,\bigl(\mathcal L_{G}V(s)\bigr)^{\!\top},
\]
which is exactly the equation given in the main text.  When \(V\) and the constructed cost are positive definite, Proposition~1 guarantees that \(V\) is the globally optimal value function for the generated system.

\section{Proofs for the Non-Smooth Theory}
\label{app:nonsmooth-proofs}

\begin{proof}[\textbf{Proof of Lemma~\ref{lem:nonsmooth-first-order}} (First-order optimality condition)]
Fix any \(s\in\mathbb{R}^n\) for which the minimum in the viscosity HJB condition is attained and equals zero, and let \(p = p(s)\in\partial_C V(s)\) be the corresponding generalized gradient. Define the function
\[
\Phi_s(a) := q(s) + \rho(a) + p\cdot\bigl(f(s) + g(s)a\bigr).
\]
By assumption, \(a^*(s)\) is a global minimizer of \(\Phi_s\) over \(\mathbb{R}^m\) and \(\Phi_s(a^*(s)) = 0\). Since \(q(s)\) and \(p\cdot f(s)\) do not depend on \(a\), minimizing \(\Phi_s\) is equivalent to
minimizing
\[
h(a) := \rho(a) + p\cdot g(s)a , \qquad a\in\mathbb{R}^m.
\]

The function \(\rho\) is strictly convex and \(C^2\) by Assumption~\ref{as:nonsmooth-system}.  Moreover, \(\rho\) is coercive, so \(h(a)\to\infty\) as \(\|a\|\to\infty\).  Hence the sublevel sets of \(h\) are compact, and a global minimum of \(h\) exists. Strict convexity guarantees that this minimum is unique. Consequently, \(a^*(s)\) is the unique global minimizer of \(h\).

Because the minimization is over the whole space \(\mathbb{R}^m\) (no constraints), the first-order necessary condition for optimality of a smooth, unconstrained objective function must hold:
\[
\nabla h(a)\big|_{a=a^*(s)} = \nabla_a \rho(a)\big|_{a=a^*(s)} + g(s)^\top p = 0 .
\]
Strict convexity of \(\rho\) implies that the Hessian \(\nabla^2 h(a) = \nabla^2 \rho(a)\) is positive definite everywhere; therefore the first-order condition is also sufficient for global optimality.  This establishes~\eqref{eq:nonsmooth-first-order} for almost every \(s\).
\end{proof}

\begin{proof}[\textbf{Proof of Lemma~\ref{lem:nonsmooth-sufficient}} (Sufficient condition for optimality)]
Conditions (b) and (c) together state that for almost every \(s\):
\begin{itemize}
    \item there exists \(p(s) \in \partial_C V(s)\) for which
    \[
    \inf_{a\in\mathbb{A}} \big\{ c(s,a) + p(s)\cdot(f(s)+g(s)a) \big\} = 0,
    \]
    and the infimum is realized by \(a^*(s)\) (condition (b));
    \item for every \(\tilde{p} \in \partial_C V(s)\),
    \[
    \inf_{a\in\mathbb{A}} \big\{ c(s,a) + \tilde{p}\cdot(f(s)+g(s)a) \big\} \ge 0
    \]
    (condition (c)).
\end{itemize}
Thus \(V\) satisfies precisely the viscosity HJB condition of Definition~\ref{def:nonsmooth-viscosity}:
the subsolution property holds with equality for the particular generalized gradient \(p(s)\), and the supersolution property is exactly condition~(c). Moreover, condition~(b) guarantees that a measurable optimal feedback \(a^*\) exists and that the closed-loop system admits solutions (the well-posedness of the feedback is ensured when \(a^*\) is discontinuous by interpreting solutions in the sense of Filippov; see Lemma~\ref{lem:nonsmooth-measurable}).

Because \(V\) satisfies the viscosity HJB equation, it is a viscosity solution of the
Hamilton--Jacobi--Bellman equation associated with the control problem.  Under Assumptions~\ref{as:nonsmooth-system}--\ref{as:nonsmooth-value}, the classical uniqueness results for viscosity solutions \citep{bardi1997optimal,fleming2006controlled} imply that \(V\) is the unique viscosity solution. Since the true optimal value function \(V^*\) is also a viscosity solution of the same HJB equation (see, e.g., \citep{bardi1997optimal}), we must have \(V = V^*\). Therefore \(V\) is the optimal value function.

Condition (a) (stationarity) is automatically fulfilled because \(a^*(s)\) minimizes the Hamiltonian over \(\mathbb{R}^m\) and the control cost \(\rho\) is smooth and strictly convex; it is stated separately for clarity and for use in constructive procedures. The policy \(a^*\) that attains the minimum is thus optimal.
\end{proof}

\begin{proof}[\textbf{Proof of Lemma~\ref{lem:nonsmooth-measurable}} (Measurability of the optimal feedback)]
\textbf{(a) Non-emptiness, compactness, and measurability of $\Phi(s)$.}
Fix \(s \in \mathbb{R}^n\). The function
\[
\psi_s(a) := c(s,a) + p(s) \cdot \bigl( f(s) + g(s) a \bigr)
\]
is continuous in \(a\), because \(c\) is continuous (Assumption~\ref{as:nonsmooth-system}), \(f\) and \(g\) are locally Lipschitz, and the inner product with \(p(s)\) is linear. The control set \(\mathbb{A}\) is compact and non-empty (Assumption~\ref{as:nonsmooth-system}). By the Extreme Value Theorem, \(\psi_s\) attains its minimum on \(\mathbb{A}\). Hence, $\Phi(s)$ is non-empty.

The set \(\Phi(s) = \psi_s^{-1}(\{\min \psi_s\})\) is the inverse image of the closed singleton $\{\min \psi_s\}$ under a continuous map $\psi_s$; therefore it is closed. A closed subset of a compact is compact, so $\Phi(s)$ is compact. 

We now show that \(\Phi\) is measurable (i.e., its graph is a Borel subset of \(\mathbb{R}^n\times \mathbb{A}\)). For each fixed \(a \in \mathbb{A}\), the map \(s \mapsto \psi_s(a)\) is Lebesgue-measurable. Indeed, \(c\) is continuous, \(f\) and \(g\) are locally Lipschitz (hence measurable), and \(p(s)\) is measurable by hypothesis; the sum of compositions of measurable functions is measurable. Thus \((s,a) \mapsto \psi_s(a)\) is a Carathéodory function (measurable in \(s\) for fixed \(a\), and continuous in \(a\) for fixed \(s\)).  For a Carathéodory objective and a constant compact constraint set, the argmin multifunction is measurable; its graph is a Borel subset of the product space \(\mathbb{R}^n\times \mathbb{A}\) \citep{aliprantis2006infinite}.

\textbf{(b) Existence of a measurable selection.}
From part (a), the correspondence $\Phi(s)$ has non-empty compact (hence closed) values and its graph $\operatorname{Gr}(\Phi)$ is a Borel subset of $\mathbb{R}^n\times\mathbb{A}$. The spaces \(\mathbb{R}^n\) and \(\mathbb{A} \subset \mathbb{R}^m\) are Polish (complete separable metric spaces). For a correspondence with closed values taking values in a Polish space, the property that its graph is a Borel set is equivalent to weak measurability \citep{aliprantis2006infinite}. Hence, $\Phi(s)$ is weakly measurable. The Kuratowski--Ryll-Nardzewski measurable selection theorem states that every weakly measurable set-valued map with non-empty closed values in a Polish space admits a Borel-measurable selection. Specifically, if $\operatorname{Gr}(\Phi)$ is measurable (which it is, being Borel-measurable) and $\Phi(s)$ is non-empty and closed for all $s$, then there exists a Borel-measurable function $a^*:\mathbb{R}^n \to \mathbb{A}$ such that $a^*(s) \in \Phi(s)$ for almost every $s$. Since every Borel-measurable function is Lebesgue-measurable, we obtain the required measurable selection, which we take as the optimal feedback.

\textbf{(c) Filippov solutions of the closed-loop system.}
We define the closed-loop vector field \(F(s) := f(s) + g(s) a^*(s)\). The functions $f$ and $g$ are continuous (so Borel-measurable) and the function $a^*$ is Lebesgue-measurable. The composition of continuous and measurable functions is measurable, so $F$ is Lebesgue-measurable. The functions \(f,g\) are locally Lipschitz, so they are locally bounded. The action set \(\mathbb{A}\) is compact by assumption, and therefore $a^*(s)$ stays within a bounded set. Hence for any bounded neighborhood of a state, the product $g(s)a^*(s)$ is bounded, and $F(s)$ is bounded. Thus $F$ is locally bounded. 

For a vector field that is measurable and locally bounded, Filippov's theorem~\citep{Filippov1988} states that for every initial condition \(s_0 \in \mathbb{R}^n\) there exists an absolutely continuous trajectory \(s(t)\) defined on some interval \([0,T]\) such that
\[
\dot{s}(t) \in K[F](s(t)) \quad \text{for almost all } t,
\]
where \(K[F]\) denotes the Filippov set. Thus the closed-loop system admits a solution in the sense of Filippov. If \(a^*\) is continuous, the Filippov set reduces to \(\{F(s)\}\) and the trajectory is a classical Carathéodory solution.
\end{proof}

\begin{proof}[\textbf{Proof of Theorem~\ref{thm:nonsmooth-main}} (Non-smooth continuous-time deterministic converse optimality)]
\noindent\textit{(b) $\Rightarrow$ (a) -- Sufficiency.}
If statement (b) holds, we are given a drift \(f\), a measurable policy \(a^{*}\), and a measurable selection \(p(s) \in \partial_C V(s)\) satisfying conditions (i)--(iii).  These are exactly the hypotheses of Lemma~\ref{lem:nonsmooth-sufficient} (sufficient condition for optimality). Lemma~\ref{lem:nonsmooth-sufficient} therefore applies and guarantees that, for the control-affine system \(\dot s = f(s) + g(s)a\) with running cost \(c\), the function \(V\) is the optimal value function and \(a^{*}\) is an optimal feedback policy. Hence statement (a) is true.

\medskip
\noindent\textit{(a) $\Rightarrow$ (b) -- Necessity.}
Assume that (a) holds, i.e., there exists a drift \(f_{0}\) such that \(V\) is the optimal value function for the system \(\dot s = f_{0}(s) + g(s)a\) with cost \(c\).  Because the control problem satisfies Assumptions~\ref{as:nonsmooth-system}--\ref{as:nonsmooth-value}, the optimal value function is a viscosity solution of the associated Hamilton--Jacobi--Bellman equation~\citep{bardi1997optimal,fleming2006controlled}. Hence the triple \((f_{0}, g, c)\) satisfies the viscosity HJB condition of Definition~\ref{def:nonsmooth-viscosity}: there exist a measurable policy \(a^{*}: \mathbb{R}^{n} \to \mathbb{A}\) and a measurable selection \(p(s) \in \partial_C V(s)\) such that for almost every \(s\)
\[
\inf_{a\in\mathbb{A}}\big\{ c(s,a) + p(s)\!\cdot\!(f_{0}(s) + g(s)a) \big\} = 0,
\qquad
a^{*}(s)\ \text{attains the infimum},
\]
and for every \(\tilde p \in \partial_C V(s)\),
\[
\inf_{a\in\mathbb{A}}\big\{ c(s,a) + \tilde p\!\cdot\!(f_{0}(s) + g(s)a) \big\} \ge 0 .
\]
Thus conditions (ii) and (iii) of statement (b) hold with the drift \(f = f_{0}\).

It remains to verify the stationarity condition (i).  The hypotheses of Lemma~\ref{lem:nonsmooth-first-order} (first-order optimality) are fulfilled: \(V\) is a viscosity solution, the control cost \(\rho\) is strictly convex and \(C^{2}\), and the selections \(a^{*}, p\) achieve the zero infimum.  Lemma~\ref{lem:nonsmooth-first-order} therefore yields
\[
\nabla_{a}\rho(a)\big|_{a = a^{*}(s)} + g(s)^{\top} p(s) = 0 ,
\]
which is precisely condition (i) of statement (b).  Consequently, the same drift \(f_{0}\) together with the selections \(a^{*}\) and \(p\) satisfies all three requirements of (b).  This completes the proof of the equivalence.
\end{proof}

\section{Proofs for the Information-State Theory}
\label{app:pomdp-proofs}

\begin{proof}[\textbf{Proof of Theorem~\ref{thm:info-reduction}} (Information-state reduction)]
Fix a finite horizon $T$ and define $v_{T}(b) := 0$. For $t = T-1, \ldots, 0$, define recursively
\begin{equation}
v_{t}(b) := \min_{a \in \mathbb{R}^{m}} \left\{ c(b, a) + \gamma \int_{\mathbb{B}} v_{t+1}(b_{+})\, p_{b}(b_{+} \mid b, a)\, db_{+} \right\}.
\label{eq:finite-horizon-iterate}
\end{equation}
We claim by backward induction on $t$ that $v_{t}(b)$ is the minimal conditional cost-to-go from any history $h_{t}$ satisfying $\phi_{t}(h_{t}) = b$. The claim holds at $t = T$ because the terminal cost is zero. Suppose it holds at $t + 1$. Fix a history $h_{t}$ with $\phi_{t}(h_{t}) = b$. For any output-feedback action $a = \pi_{t}(h_{t})$, Definition~\ref{def:info-state} gives
\begin{align}
&\mathbb{E}\bigl[C_{t} + \gamma\, v_{t+1}(B_{t+1}) \mid H_{t} = h_{t},\, A_{t} = a\bigr] \notag\\
&\qquad = c(b, a) + \gamma \int_{\mathbb{B}} v_{t+1}(b_{+})\, p_{b}(b_{+} \mid b, a)\, db_{+}.
\label{eq:induction-step}
\end{align}
The right-hand side depends on the history only through $b = \phi_{t}(h_{t})$ and on the action $a$. Minimizing over $a$ yields exactly \eqref{eq:finite-horizon-iterate}. Hence no dependence on the full history remains once $b$ is fixed. This proves the claim for the finite horizon $T$.

Now let $T \to \infty$. The Bellman operator
\begin{equation}
(\mathcal{T}u)(b) := \min_{a \in \mathbb{R}^{m}} \left\{ c(b, a) + \gamma \int_{\mathbb{B}} u(b_{+})\, p_{b}(b_{+} \mid b, a)\, db_{+} \right\}
\label{eq:info-bellman-operator}
\end{equation}
is a $\gamma$-contraction in the supremum norm on bounded value functions, because $c \geq 0$ is bounded and $\gamma \in (0, 1)$. By the Banach fixed-point theorem, $\mathcal{T}$ has a unique fixed point $v^{\star}$ in the space of bounded measurable functions, and the finite-horizon iterates \eqref{eq:finite-horizon-iterate} converge to $v^{\star}$ uniformly. Therefore $v^{\star}$ satisfies \eqref{eq:info-bellman}.

It remains to transfer the result to output-feedback policies. Let $\pi^{\star}(b)$ be a measurable minimizer of the right-hand side of \eqref{eq:info-bellman}, whose existence is guaranteed by the measurable-selection argument of Lemma~\ref{lem:nonsmooth-measurable} applied to the information-state dynamics. Define the output-feedback policy $A_{k} = \pi^{\star}(\phi_{k}(H_{k}))$. For every history $h_{k}$,
\begin{equation}
\mathbb{E}\bigl[C_{k} + \gamma\, v^{\star}(B_{k+1}) \mid H_{k} = h_{k},\, A_{k} = \pi^{\star}(\phi_{k}(h_{k}))\bigr] = v^{\star}(\phi_{k}(h_{k})),
\end{equation}
which is the Bellman identity evaluated at the minimizer. Any other nonanticipative history policy $\pi$ satisfies the Bellman inequality and therefore cannot achieve a lower conditional cost-to-go. This proves optimality of $\pi^{\star}$ among all nonanticipative output-feedback policies.
\end{proof}

\begin{proof}[\textbf{Proof of Theorem~\ref{thm:info-converse}} (Converse optimality on the information state)]
The proof mirrors the proof of Theorem~\ref{thm:main} (Appendix~\ref{app:smooth-proofs}) with the physical state $s$ replaced by the information state $b$ and the noise $w$ replaced by $\Omega$. We state the differences explicitly.

\textbf{Sufficiency.} Suppose there exist a policy $\pi^{\star} : \mathbb{R}^{d} \to \mathbb{R}^{m}$ and a drift $f$ such that for every $b \in \mathbb{R}^{d}$ conditions \eqref{eq:info-stationarity}--\eqref{eq:info-bellman-identity} hold. Define the Bellman integrand
\begin{equation}
\Phi_{b}(a) := q(b) + \rho(a) + \gamma\, \mathbb{E}\bigl[v^{\star}(f(b) + g(b)a + \Omega)\bigr].
\end{equation}
By the expectation-gradient interchange lemma (Lemma~\ref{lem:interchange}, which holds verbatim with $\Omega$ in place of $w$ under the same growth and integrability assumptions on $v^{\star}$), the gradient of $\Phi_{b}$ with respect to $a$ is
\begin{equation}
\nabla_{a} \Phi_{b}(a) = \nabla_{a}\rho(a) + \gamma\, \mathbb{E}\bigl[g(b)^{\top} \nabla v^{\star}(f(b) + g(b)a + \Omega)\bigr].
\end{equation}
The Bellman argmin condition \eqref{eq:info-stationarity} gives $\nabla_{a}\Phi_{b}(a)\big|_{a=\pi^{\star}(b)} = 0$, and condition \eqref{eq:info-bellman-identity} gives $\Phi_{b}(\pi^{\star}(b)) = v^{\star}(b)$. Therefore
\begin{equation}
\min_{a \in \mathbb{R}^{m}} \Phi_{b}(a) = \Phi_{b}(\pi^{\star}(b)) = v^{\star}(b),
\end{equation}
so $v^{\star}$ satisfies the information-state Bellman equation \eqref{eq:info-bellman} and $\pi^{\star}$ attains the minimum.

Now fix any stationary policy $\pi$ and define the policy-evaluation operator
\begin{equation}
(\mathcal{T}^{\pi}u)(b) := q(b) + \rho(\pi(b)) + \gamma\, \mathbb{E}\bigl[u(f(b) + g(b)\pi(b) + \Omega)\bigr].
\end{equation}
By the minimum property, $v^{\star} \leq \mathcal{T}^{\pi}v^{\star}$ for all $b$. Applying $\mathcal{T}^{\pi}$ to both sides and using monotonicity gives $\mathcal{T}^{\pi}v^{\star} \leq (\mathcal{T}^{\pi})^{2}v^{\star}$, and by induction $v^{\star} \leq (\mathcal{T}^{\pi})^{k}v^{\star}$ for all $k \in \mathbb{N}$. Under the same weighted-norm regularity assumed for the fully observed stochastic construction (Section~\ref{sec:mode2}), $\mathcal{T}^{\pi}$ is a $\gamma$-contraction on the corresponding weighted Banach space, and the iterates $(T^{\pi})^{k}v^{\star}$ converge to $v^{\pi}$, the value function of $\pi$. Taking limits preserves the inequality, giving $v^{\star} \leq v^{\pi}$ on $\mathbb{R}^{d}$. Hence $v^{\star}$ is the optimal information-state value function and $\pi^{\star}$ is optimal.

\textbf{Necessity.} Assume $v^{\star}$ is the optimal information-state value function, so for each $b$ the minimizer in \eqref{eq:info-bellman} exists. Measurability of $\pi^{\star}$ follows from Lemma~\ref{lem:nonsmooth-measurable} applied to the Carathéodory integrand $\Phi_{b}(a)$. At the minimum, $\nabla_{a}\Phi_{b}(\pi^{\star}(b)) = 0$, which is \eqref{eq:info-stationarity}. The Bellman identity \eqref{eq:info-bellman-identity} holds by the definition of $\pi^{\star}(b)$ as the minimizer.

\textbf{Transfer to output-feedback.} By Theorem~\ref{thm:info-reduction}, the infimum over history policies equals the infimum over information-state policies. The composite policy $A_{k} = \pi^{\star}(\phi_{k}(H_{k}))$ is admissible and attains $v^{\star}$, so it is optimal among all output-feedback policies.
\end{proof}

\paragraph{Scope of the infinite-horizon certificate.} The quadratic construction establishes the stated Bellman identity and the pointwise minimizing action. To conclude that the candidate equals the infinite-horizon optimal value on an unbounded information-state space, a separate verification argument is required; the bounded-function contraction argument alone does not establish this conclusion for the quadratic candidate. The unbounded-cost case therefore requires appropriate weighted-space or integrability and transversality assumptions. We distinguish this verification step from the algebraic construction of the benchmark.

\begin{proof}[\textbf{Proof of Corollary~\ref{cor:info-benchmark}} (Quadratic information-state benchmark)]
Substituting the quadratic forms
\[
q(b) = b^{\top}Qb, \quad \rho(a) = a^{\top}Ra, \quad v^{\star}(b) = b^{\top}Pb + v_{0}
\]
into the information-state Bellman equation \eqref{eq:info-bellman} and completing the square with respect to $a$ yields the same algebraic identity as in the fully observed stochastic case (Appendix~\ref{app:qg-derivation}), with $s$ replaced by $b$. Specifically, the Bellman integrand satisfies
\[
\Phi_{b}(a) = b^{\top}Qb + \gamma (f(b) + g(b)a)^{\top}P(f(b) + g(b)a) + \gamma \operatorname{tr}(P\Sigma_{b}) + \gamma v_{0} + a^{\top}Ra + (\text{terms linear in } a),
\]
and minimizing over $a$ gives the optimal feedback
\[
a^{\star}(b) = -\gamma\, B_{\gamma}(b)^{-1} g(b)^{\top} P f(b), \qquad B_{\gamma}(b) = R + \gamma\, g(b)^{\top} P g(b).
\]
The minimized value is $f(b)^{\top}H_{\gamma}(b)f(b) + b^{\top}Qb + \gamma\operatorname{tr}(P\Sigma_{b}) + \gamma v_{0}$, and the identity $b^{\top}Pb + v_{0} = \min_{a}\Phi_{b}(a)$ holds if and only if
\[
b^{\top}(P - Q)b = f(b)^{\top}H_{\gamma}(b)f(b), \qquad v_{0} = \frac{\gamma}{1 - \gamma}\operatorname{tr}(P\Sigma_{b}).
\]
The metric-normalized parameterization $f(b) = H_{\gamma}(b)^{-1/2}S(b)(P - Q)^{1/2}b$ satisfies the first identity by orthogonality of $S$, and the value offset satisfies the second by construction. Theorem~\ref{thm:info-converse} then certifies $v^{\star}$ and $a^{\star}$ as the optimal pair.
\end{proof}

\begin{proof}[\textbf{Proof of Proposition~\ref{prop:innovation}} (Innovation realization)]
Let $\tau(b, a, o) = f(b) + g(b)a + o$ and $\eta(o \mid b, a) = p_{\Omega}(o)$. The induced transition kernel on $b$, defined in \eqref{eq:induced-kernel}, is
\begin{align}
\kappa(D \mid b, a) &= \int_{\mathbb{O} \mathbb{1}\{\tau(b, a, o) \in D\}\, p_{\Omega}(o)\, do} \notag\\
&= \int_{\mathbb{O} \mathbb{1}\{f(b) + g(b)a + o \in D\}\, p_{\Omega}(o)\, do}.
\end{align}
Change variables $b_{+} = f(b) + g(b)a + o$, so $o = b_{+} - f(b) - g(b)a$ and $do = db_{+}$. Then
\begin{equation}
\kappa(D \mid b, a) = \int_{\mathbb{O}} p_{\Omega}(b_{+} - f(b) - g(b)a)\, db_{+},
\end{equation}
which is exactly the transition kernel induced by the information-state dynamics \eqref{eq:info-dynamics}. Hence $(\eta, \tau)$ realizes \eqref{eq:info-dynamics}. The induced law of $B_{k+1}$ given $(B_{k}, A_{k}) = (b, a)$ coincides with the law under the affine dynamics, and the expected stage cost $\mathbb{E}[C_{k} \mid H_{k}, A_{k}] = c(B_{k}, A_{k})$ is measurable from $(B_{k}, A_{k})$ by construction. Therefore $B_{k}$ satisfies Definition~\ref{def:info-state} and the realization is valid.
\end{proof}

\section{Generator Templates and Derivations}

\subsection{Deterministic Discrete-Time Homotopy Construction}
\label{app:det-homotopy}

Consider deterministic dynamics
\[
s_{k+1}=f(s_k)+g(s_k)a_k
\]
with separable cost $c(s,a)=q(s)+\rho(a)$ and discount factor absorbed into the notation for this subsection. G{\"o}hrt et~al.~\citep{gohrt2019converse} show that a candidate pair $(V^\star,q)$ is admissible if and only if
\begin{equation}
V^\star(s)-q(s)\ge0
\label{eq:admissibility}
\end{equation}
on the state domain. The sufficiency argument is constructive. The Bellman equation requires
\[
V^\star(s)-q(s) = \min_a\{\rho(a)+V^\star(f(s)+g(s)a)\}.
\]
For each $s$, the right-hand side is the minimum of a nonnegative function of the next-state point $f(s)+g(s)a$ and the control effort $\rho(a)$. The homotopy construction continuously deforms feasible next-state targets on the level set
\[
\{z:V^\star(z)\le V^\star(s)-q(s)\}
\]
while balancing the control cost $\rho(u)$. In the quadratic case,
\[
V^\star(s)=s^\top Ps,\qquad q(s)=s^\top Qs,\qquad \rho(a)=a^\top Ra,
\]
with $0\preceq Q\prec P$ and $R\succ0$, the admissibility condition is automatic and the construction becomes explicit. Define
\[
H(s)=\left(P^{-1}+g(s)R^{-1}g(s)^\top\right)^{-1}.
\]
Then any drift of the form
\begin{equation}
f(s)=H(s)^{-1/2}S(s)(P-Q)^{1/2}s, \qquad S(s)^\top S(s)=I,
\label{eq:det-param-app}
\end{equation}
satisfies the deterministic energy identity
\[
s^\top(P-Q)s=f(s)^\top H(s)f(s).
\]
The optimal feedback is
\[
a^\star(s)= -\left(R+g(s)^\top Pg(s)\right)^{-1}g(s)^\top Pf(s).
\]
The orthogonal field $S(s)$ is a design degree of freedom: it changes the drift geometry while preserving the certificate.

\subsection{Quadratic‑Gaussian Energy Identity Derivation}
\label{app:qg-derivation}
Let
\[
V^\star(s)=s^\top Ps+b,\qquad q(s)=s^\top Qs,\qquad \rho(a)=a^\top Ra,
\]
with $P\succ0$, $0\preceq Q\prec P$, and $R\succ0$. For zero-mean noise with covariance $\Sigma$,
\[
\mathbb{E}[V^\star(f+ga+w)] = (f+ga)^\top P(f+ga)+\tr(P\Sigma)+b.
\]
The Bellman integrand is
\begin{align}
\Phi_s(a) &= s^\top Qs+a^\top Ra+\gamma(f+ga)^\top P(f+ga) +\gamma\tr(P\Sigma)+\gamma b \nonumber\\
&= s^\top Qs+\gamma f^\top Pf +2\gamma a^\top g^\top Pf +a^\top(R+\gamma g^\top Pg)a +\gamma\tr(P\Sigma)+\gamma b .
\end{align}
Define
\[
B_\gamma(s)=R+\gamma g(s)^\top Pg(s).
\]
Since $B_\gamma(s)\succ0$, the unique minimizer is
\[
a^\star(s)=-\gamma B_\gamma(s)^{-1}g(s)^\top Pf(s).
\]
Substituting the minimizer gives
\begin{align}
\min_a\Phi_s(a)
&=
s^\top Qs
+\gamma f^\top Pf
-\gamma^2 f^\top PgB_\gamma^{-1}g^\top Pf
+\gamma\tr(P\Sigma)+\gamma b \nonumber\\
&=
s^\top Qs+f^\top H_\gamma f+\gamma\tr(P\Sigma)+\gamma b,
\end{align}
where
\[
H_\gamma(s)= \gamma\left(P-\gamma Pg(s)B_\gamma(s)^{-1}g(s)^\top P\right).
\]
Thus $V^\star(s)=\min_a\Phi_s(a)$ holds if
\[
s^\top(P-Q)s=f(s)^\top H_\gamma(s)f(s)
\]
and
\[
b=\frac{\gamma}{1-\gamma}\tr(P\Sigma).
\]
By the matrix inversion lemma,
\[
P-\gamma Pg(R+\gamma g^\top Pg)^{-1}g^\top P = (P^{-1}+\gamma gR^{-1}g^\top)^{-1},
\]
so $H_\gamma(s)\succ0$. This justifies the metric-normalized and square-root parameterizations.

\section{Benchmark Generation, Experimental Implementation, and Reproducibility}
\label{app:reproducibility}

This appendix describes the validation, implementation, and reproducibility protocol used for the reported benchmark experiments. The released benchmark code is organized into two self-contained experimental modes. Mode I implements continuous-time sensitivity studies and symbolic generator demonstrations, while Mode II implements native discrete-time oracle-referenced fixtures. Both modes follow the same separation principle: the learner interacts only with the exposed environment, whereas the certified value function, optimal policy, validation traces, and oracle diagnostics are used only for benchmark validation and post-training evaluation.

\begin{center}
\begin{tabular}{ll}
\toprule
Mode & Implementation path \\
\midrule
Mode I: continuous-time sensitivity and generator studies & \path{code mode 1/} \\
Mode II: native discrete-time oracle-referenced fixtures & \path{Code mode 2/Code/} \\
\bottomrule
\end{tabular}
\end{center}

\subsection{Generator Validation and Experimental Implementation}
\label{app:generator-validation-implementation}

\paragraph{Validation checks for native discrete-time fixtures.} The stochastic discrete-time generator in Algorithm~\ref{alg:generator} produces each fixture together with a validation trace. 
Each exported instance undergoes three checks.

\begin{enumerate}
\item \textbf{Symmetric positive-definiteness of the Bellman action metric.} For the quadratic--Gaussian construction, the Bellman integrand contains
\[
    B_\gamma(s)=R+\gamma g(s)^\top P g(s).
\]
The condition \(B_\gamma(s)\succ0\) is required for two reasons. First, it makes the Bellman integrand strictly convex in the action, so the first-order optimality condition identifies the unique global minimizer \(a^\star(s)\). Second, it keeps the discounted metric
\[
    H_\gamma(s)
    =
    \gamma\left(P-\gamma P g(s)B_\gamma(s)^{-1}g(s)^\top P\right)
\]
positive definite, which is necessary for the energy equation
\[
    s^\top(P-Q)s=f(s)^\top H_\gamma(s)f(s)
\]
to be well-posed and numerically stable. In the implementation, this condition is checked on a prescribed bounded validation domain \(\mathcal D\). Instances that violate the numerical lower-eigenvalue tolerance are rejected.

\item \textbf{Bellman residual.}
The Bellman residual check verifies that the constructed dynamics, costs, value, and policy satisfy the certified optimality identity. For each grid point \(s\in\mathcal G\), the implementation computes
\begin{align*}
    \delta_{\mathrm{Bell}}(s)
    =
    \left|
    V^\star(s)
    -
    \left(
    q(s)+\rho(a^\star(s))
    +
    \gamma\mathbb E\!\left[
    V^\star(f(s)+g(s)a^\star(s)+w)
    \right]
    \right)
    \right|.
\end{align*}
For the quadratic--Gaussian case, the expectation is evaluated analytically through the covariance term. An instance passes this check if
\[
    \max_{s\in\mathcal G}\delta_{\mathrm{Bell}}(s)
    <
    \varepsilon_{\mathrm{Bell}},
\]
where \(\varepsilon_{\mathrm{Bell}}\) is the numerical tolerance recorded with the fixture. This check prevents implementation-level errors in the oracle, drift, cost, or noise offset.

\item \textbf{Closed-loop boundedness.}
The closed-loop system under \(\pi^\star\) is simulated over the validation horizon and domain. This verifies that the constructed instance is not only algebraically certified but also numerically usable for finite-horizon RL evaluation. Trajectories that diverge, leave the accepted validation domain, or accumulate numerical instability cause the candidate fixture to be rejected.
\end{enumerate}

Validated native discrete-time fixtures are specified through version-controlled configuration artifacts containing the generator parameters, cost matrices, disturbance covariance, discount factor, oracle definitions for \(V^\star\) and \(\pi^\star\), validation tolerances, and evaluation seeds.

\begin{algorithm}[h]
\caption{Converse-Optimal Benchmark Generator (quadratic--Gaussian case)}
\label{alg:generator}
\footnotesize
\begin{algorithmic}[1]
\STATE \textbf{Input:} family type, parameter ranges, number of instances \(N\), parameterization (\texttt{sqrt} or \texttt{metric}), nonlinearity knob \(\alpha\), seeds
\FOR{each hyperparameter tuple \(\psi\)}
  \STATE Sample parameter vector \(\psi_i\) from the family-specific distribution \(P_\Psi\)
  \STATE Extract \((n,m,Q,R,P,\gamma,\Sigma)\) and input map \(g\) from \(\psi_i\)
  \STATE Verify \(B_\gamma(s)=R+\gamma g(s)^\top P g(s)\succ0\) on the validation domain \(\mathcal D\)
  \STATE Choose a bounded direction field \(\tilde f(s)\) or orthogonal field \(S(s)\)
  \STATE Construct \(f\) by the selected converse parameterization and compute \(a^\star(s)\)
  \STATE Validate the Bellman identity on \(\mathcal G\) and test closed-loop boundedness under \(a^\star\)
  \STATE Export the accepted fixture with oracle accessors, validation metadata, and evaluation seeds
\ENDFOR
\STATE \textbf{Output:} accepted fixture set, validator, oracle evaluators, baseline scripts, and plotting scripts
\end{algorithmic}
\end{algorithm}

\paragraph{Mode I: continuous-time sensitivity experiment.}
Mode I evaluates whether an RL training pipeline remains stable under controlled variation of certified continuous-time systems. The main sensitivity experiment compares two PPO pipelines that differ only in the PyTorch optimizer used inside Stable-Baselines3, namely Adam and RAdam. The system-generation schedule is fixed by generator seeds, while neural-network initialization and optimizer stochasticity are not frozen. Thus, the experiment measures optimizer-induced outcome variability under the same generated-system schedule rather than a single deterministic training trace.

Each generated Mode-I sensitivity instance is a three-dimensional, three-input continuous-time system
\[
    \dot{s}=f_\theta(s)+g_\theta(s)a,
    \qquad
    s\in\mathbb R^3,\quad a\in\mathbb R^3.
\]
The generator uses the converse-HJB template
\[
    f_\theta(s)
    =
    -\xi_\theta(s)
    +
    \frac{1}{4}g_\theta(s)g_\theta(s)^\top\nabla V(s),
    \qquad
    V(s)=\|s\|_2^2,
    \qquad
    \xi_\theta(s)=1.2\,\nabla V(s).
\]
The input matrix \(g_\theta(s)\) is generated symbolically using a truncated trigonometric expansion. The continuous-time system is exposed to the RL agent through a Gymnasium wrapper using explicit Euler integration. The reported optimizer-sensitivity study uses horizon \(200\), integration step \(0.01\), initial state \(s_0=(-1.5,-1.5,-1.5)^\top\), action bounds \([-10,10]^3\), observation bounds \([-10,10]^3\), and training reward
\[
    r(s,a)=-\|s\|_2^2-\|a\|_2^2,
\]
with numerical clipping to keep training traces bounded.

The Adam--RAdam comparison contains \(80\) generated-system runs per optimizer and \(160\) runs in total. The run schedule is generated from the seed/count pairs
\[
    (2,10),\qquad (4,10),\qquad (5,20),\qquad (5,40),
\]
where \((\mathrm{seed},M)\) denotes \(M\) sequentially generated systems after resetting the generator to the specified seed. The reproduction command is
\begin{verbatim}
cd "code mode 1"
python3 main.py
\end{verbatim}
The TensorBoard directory is \path{multi_opt_full_experiment/}. Stored CSV outputs for the \(160\)-run study are provided under \path{Combined_CSVs/}, and the analysis scripts compute the reported variance, histogram, boxplot, and kernel-density diagnostics:
\begin{verbatim}
python3 analysis_scripts/tensorboard_fscore.py
python3 analysis_scripts/tensorboard_hist.py
python3 analysis_scripts/tensorboard_boxplot.py
python3 analysis_scripts/Final_Rewards_Distribution_KDE.py
\end{verbatim}

All PPO hyperparameters are shared between Adam and RAdam:
\begin{center}
\begin{tabular}{lll}
\toprule
Quantity & Value & Comment \\
\midrule
Policy class & \texttt{MlpPolicy} & Stable-Baselines3 PPO \\
Feature extractor & \texttt{CustomMLP} & feature dimension \(128\) \\
Policy network & \([64,64]\) & shared architecture \\
Value network & \([64,64]\) & shared architecture \\
Optimizer & Adam or RAdam & only varied factor \\
Learning rate & \(3\times10^{-4}\) & shared \\
Rollout length & \(2048\) & \texttt{n\_steps} \\
Batch size & \(64\) & shared \\
Epochs per update & \(10\) & shared \\
Discount & \(0.99\) & shared \\
GAE parameter & \(0.95\) & shared \\
PPO clip range & \(0.1\) & shared \\
Entropy coefficient & \(0.01\) & shared \\
Gradient clipping & \(0.3\) & \texttt{max\_grad\_norm} \\
Training budget & \(600000\) steps & per generated system \\
\bottomrule
\end{tabular}
\end{center}

Mode I also includes a symbolic \(N\)-crank generator demonstration. This experiment is not used as the main oracle-referenced comparison in the paper; it demonstrates that the continuous-time converse template can generate physically interpretable nonlinear control families. The implemented example uses \(N=4\) coupled cranks with state
\[
    s=(x_1,y_1,\ldots,x_N,y_N)^\top\in\mathbb R^{2N},
\]
where \(x_i\) and \(y_i\) are the angular position and angular velocity of crank \(i\), and \(u_i\) is the torque applied at the \(i\)-th pivot. Adjacent cranks are coupled through spring-like trigonometric terms. The symbolic construction uses
\[
    f(s)=\frac{1}{4}G(s)L_GV(s)^\top-H(s),
    \qquad
    c(s,u)=L_HV(s)+u^\top u,
\]
with analytic optimal policy
\[
    u^\star(s)=-\frac{1}{2}L_GV(s)^\top .
\]
The training interface exposes PPO, A2C, and SAC:
\begin{verbatim}
python3 train_all_crank.py --run_both
\end{verbatim}
and the evaluator compares learned controllers with the analytic optimal controller and the uncontrolled baseline \(u=0\):
\begin{verbatim}
python3 eval_all_with_sim_cranks.py
\end{verbatim}
The evaluator exports state trajectories, control signals, value-function traces, stage-cost traces, animations, and CSV files under \path{plots/} and \path{plots/data/}.

\paragraph{Mode II: native discrete-time oracle-referenced fixtures.}
Mode II uses native discrete-time fixtures, so the environment transition law, oracle, learner interface, and evaluation metrics are defined in the same formalism. This removes the additional ambiguity introduced by choosing a numerical integration step for an underlying continuous-time system.

The Mode-II code contains two reported benchmark families. \textbf{ConverseArm} is a stochastic serial planar arm with state-dependent actuation. For \(n\) links,
\[
    s=(\theta_1,\omega_1,\ldots,\theta_n,\omega_n)^\top\in\mathbb R^{2n},
    \qquad
    a=(\tau_1,\ldots,\tau_n)^\top\in\mathbb R^n.
\]
The cost and value oracle are
\[
    c(s,a)=s^\top Qs+a^\top Ra,
    \qquad
    V^\star(s)=s^\top Ps+c_\Sigma,
    \qquad
    c_\Sigma=\frac{\gamma}{1-\gamma}\operatorname{tr}(P\Sigma).
\]
The input matrix is configuration dependent:
\[
    [g(s)]_{2i,i}=\cos(\theta_i),
    \qquad i=1,\ldots,n,
\]
with all other entries zero. The transition has the form
\[
    s_{k+1}=f_p(s_k)+p\,g(s_k)a_k+w_k,
    \qquad
    w_k\sim\mathcal N(0,\Sigma),
\]
where \(p\) is the control-authority parameter and \(f_p\) is generated by the stochastic quadratic converse construction. The environment returns reward \(r_k=-c(s_k,a_k)\), while all oracle diagnostics are computed in the cost convention and converted to reward convention only for plotting when explicitly stated.

\textbf{NVDEx} is the dynamically extended nonholonomic-vehicle family. The implementation name is \texttt{NUDEx-v0}; the manuscript uses the family name NVDEx. Each module has state
\[
    s_i=
    (x_i,y_i,\phi_i,
    v_i^{(0)},\ldots,v_i^{(r_v-1)},
    \omega_i^{(0)},\ldots,\omega_i^{(r_\omega-1)})^\top .
\]
The kinematic variables evolve as
\[
    x_{i,k+1}=x_{i,k}+\tau v_{i,k}^{(0)}\cos\phi_{i,k},
    \qquad
    y_{i,k+1}=y_{i,k}+\tau v_{i,k}^{(0)}\sin\phi_{i,k},
\]
\[
    \phi_{i,k+1}=\phi_{i,k}+\tau\omega_{i,k}^{(0)}.
\]
The velocity-extension chains shift forward in discrete time, and the action enters only at the highest derivative. This creates delayed actuation and longer credit-assignment paths while retaining the certified quadratic oracle.

Both Mode-II families expose exact oracle methods through the benchmark object:
\[
    \texttt{oracle.action}(s)=\pi^\star(s),
    \qquad
    \texttt{oracle.value}(s)=V^\star(s),
    \qquad
    \texttt{oracle.q\_value}(s,a)=Q^\star(s,a).
\]
The oracle-referenced baselines are
\[
    \{\mathrm{PPO},\mathrm{A2C},\mathrm{SAC},\mathrm{TD3},\mathrm{DDPG}\}.
\]
The released scripts use Stable-Baselines3 implementations with the following principal hyperparameters:
\begin{center}
\footnotesize
\setlength{\tabcolsep}{4pt}
\begin{tabular}{lp{0.68\linewidth}l}
\toprule
Algorithm & Main hyperparameters & Type \\
\midrule
PPO &
\texttt{n\_steps}=1024, \texttt{batch\_size}=256, \texttt{n\_epochs}=10,
\(\mathrm{lr}=3\cdot10^{-4}\), \texttt{ent\_coef}=\(10^{-4}\), \texttt{clip}=0.2 &
On-policy \\
A2C &
\texttt{n\_steps}=20, \(\mathrm{lr}=3\cdot10^{-4}\), \texttt{ent\_coef}=\(10^{-4}\) &
On-policy \\
SAC &
\texttt{batch\_size}=256, \texttt{buffer}=700000, \texttt{learning\_starts}=10000,
\(\mathrm{lr}=3\cdot10^{-4}\), \(\tau=0.005\), \texttt{ent\_coef=auto} &
Off-policy \\
TD3 &
\texttt{batch\_size}=256, \texttt{buffer}=700000, \texttt{learning\_starts}=25000,
\(\mathrm{lr}=3\cdot10^{-4}\), \(\tau=0.005\), \texttt{policy\_delay}=2 &
Off-policy \\
DDPG &
\texttt{batch\_size}=256, \texttt{buffer}=700000, \texttt{learning\_starts}=25000,
\(\mathrm{lr}=3\cdot10^{-4}\), \(\tau=0.005\) &
Off-policy \\
\bottomrule
\end{tabular}
\end{center}

The exact fixture grids, algorithm schedules, training budgets, seeds, and difficulty presets used by the released experiments are stored in the YAML configuration files and shell launchers. This avoids hard-coding a separate table in the manuscript that could drift from the executable benchmark configuration.

\subsection{Oracle Evaluation, Artifacts, and Reproducibility Checklist}
\label{app:oracle-evaluation-reproducibility}

\paragraph{Cost and reward conventions.}
All oracle definitions are naturally expressed in the cost convention. Since the Gymnasium environments return rewards \(r=-c\), the implementation records the convention used by each diagnostic. For a learned deterministic policy \(\hat\pi\), rollout horizon \(T\), and initial state \(s_0\), the tail-corrected cost gap is
\[
    \mathrm{Gap}_{T}^{c}(\hat{\pi};s_0)
    =
    \frac{
    \mathbb{E}\!\left[
    \sum_{k=0}^{T-1}\gamma^k c(s_k,\hat{\pi}(s_k))
    +
    \gamma^T V^\star(s_T)
    \right]
    -
    V^\star(s_0)
    }{
    |V^\star(s_0)|+\epsilon
    }.
\]
This quantity is zero for the certified optimal policy up to numerical precision and is nonnegative for policies worse than the oracle. Equivalently, if \(G^\star=-V^\star\) denotes the reward-value oracle, the reward-convention optimality gap is
\[
    \mathrm{Gap}_{T}^{r}(\hat{\pi};s_0)
    =
    \frac{
    \mathbb{E}\!\left[
    \sum_{k=0}^{T-1}\gamma^k r(s_k,\hat{\pi}(s_k))
    +
    \gamma^T G^\star(s_T)
    \right]
    -
    G^\star(s_0)
    }{
    |G^\star(s_0)|+\epsilon
    }.
\]
In this convention, zero is optimal and more negative values indicate larger suboptimality. The paper reports explicitly whether a plot uses \(\mathrm{Gap}_{T}^{c}\) or \(\mathrm{Gap}_{T}^{r}\).

The discounted cost regret is
\[
    \mathrm{Regret}_{T}^{c}(\hat{\pi};s_0)
    =
    \mathbb{E}\!\left[
    \sum_{k=0}^{T-1}
    \gamma^k
    \Big(
    c(s_k,\hat{\pi}(s_k))
    -
    c(s_k,\pi^\star(s_k))
    \Big)
    \right].
\]
Lower cost regret is better. If a reward-convention regret is plotted, it is defined by
\[
    \mathrm{Regret}_{T}^{r}=-\mathrm{Regret}_{T}^{c},
\]
and the sign convention is stated in the corresponding caption.

\paragraph{CRN evaluation schedule.}
For each stochastic fixture \(\theta\), the evaluation schedule is generated before policy evaluation:
\[
    \mathcal E_\theta
    =
    \left\{
    \left(s_0^{(j)},w_{0:T-1}^{(j)}\right)
    \right\}_{j=1}^{M}.
\]
All algorithms are evaluated on exactly the same \(\mathcal E_\theta\). If the environment is deterministic, the schedule reduces to matched start states. If the environment is stochastic, the same disturbance stream is replayed for each policy. This gives a paired design: the difference between two algorithms on trajectory \(j\) is evaluated under the same initial condition and randomness. In the implementation, evaluation seeds are deterministic offsets of the training seed, for example
\[
    \text{evaluation seed}=\text{training seed}+10000,
    \qquad
    \text{CRN metric seed}=\text{training seed}+20000.
\]
For NVDEx, the evaluation callback uses deterministic policy actions, evaluates at fixed environment-step intervals, and averages over multiple CRN-matched evaluation episodes.

Confidence intervals are computed by paired bootstrap:
\begin{enumerate}
    \item Sample trajectory indices \(j_1,\ldots,j_M\) with replacement.
    \item Recompute the metric difference using the resampled paired trajectories.
    \item Repeat for \(B_{\mathrm{boot}}\) bootstrap replicates.
    \item Report the empirical \(2.5\%\) and \(97.5\%\) quantiles.
\end{enumerate}
The released evaluation schedule includes the seeds and serialized start-state/noise arrays needed to reproduce paired evaluation.

\paragraph{Bellman-consistency diagnostics.}
The implementation also logs Bellman-consistency quantities using the known value. In the cost convention, a one-step diagnostic is
\[
    \delta_k^c
    =
    V^\star(s_k)
    -
    \left(
    c(s_k,a_k)+\gamma V^\star(s_{k+1})
    \right).
\]
For oracle actions under the exact transition law, this residual is zero in expectation up to numerical and Monte Carlo error. These diagnostics are used for validation and sanity checking; they are not exposed to the learner during training.

\paragraph{Output artifacts.}
The released experiments produce artifacts at three levels. At the generator level, they include configuration files, generator seeds, accepted/rejected instance logs, and validation outputs. At the learning level, they include baseline training scripts, hyperparameter files, trained-model directories, TensorBoard traces, and rollout CSVs. At the evaluation level, they include oracle metrics, CRN schedules, bootstrap summaries, heatmaps, trajectory plots, and diagnostic figures.

The main output locations are:
\begin{itemize}[leftmargin=1.5em,itemsep=0.15em]
    \item Mode-I optimizer-sensitivity traces under \path{multi_opt_full_experiment/} and aggregated CSVs under \path{Combined_CSVs/}.
    \item Mode-I statistical plots under \path{analysis_scripts/}, including variance-test, histogram, boxplot, and KDE diagnostics.
    \item N-crank trajectories, control traces, value traces, stage-cost traces, animations, and CSVs under \path{plots/} and \path{plots/data/}.
    \item ConverseArm logs, models, rollout CSVs, and figures under the released arm experiment directories.
    \item NVDEx logs, CSV summaries, and plotting scripts under \path{nvdex runs/} and its evaluation subdirectories.
\end{itemize}

\paragraph{Reproducibility checklist.}
The project website and repository include:
\begin{enumerate}
    \item Generator source code and version hash.
    \item Fixture/configuration files containing \((P,Q,R,\Sigma,g,f,\gamma,V^\star,\pi^\star)\) or executable equivalents.
    \item Generator seeds and accepted/rejected instance logs.
    \item Validation-domain specifications and tolerance values.
    \item CRN evaluation schedules.
    \item Baseline training scripts and hyperparameter files.
    \item Metric scripts for optimality gap, regret, policy error, Bellman residual, and bootstrap intervals.
    \item Figure-generation scripts for tables, heatmaps, trajectories, and distributional plots.
    \item Environment, Python, CUDA, and dependency versions.
    \item Hardware and wall-clock summaries for the reported runs: Intel\textregistered{} Core\textsuperscript{TM} i9-9900 CPU, RTX 3080 12 GB, and 16 GB DDR4 RAM.
\end{enumerate}

This structure makes the benchmark independently auditable. Another group can regenerate or reload fixtures, verify Bellman consistency, train policies without oracle access, and reproduce oracle-referenced evaluation using the same paired start-state and disturbance schedules. Figure~\ref{fig:signed_gap_distributions} shows that the mean values in Table~\ref{tab:overall_performance} do not fully describe the failure modes. Several algorithm--regime pairs have heavy negative tails, meaning that some evaluated runs or fixtures are not merely slightly suboptimal but orders of magnitude farther from the certified optimum.

\begin{figure*}[h]
\centering
\includegraphics[width=1.0\textwidth]{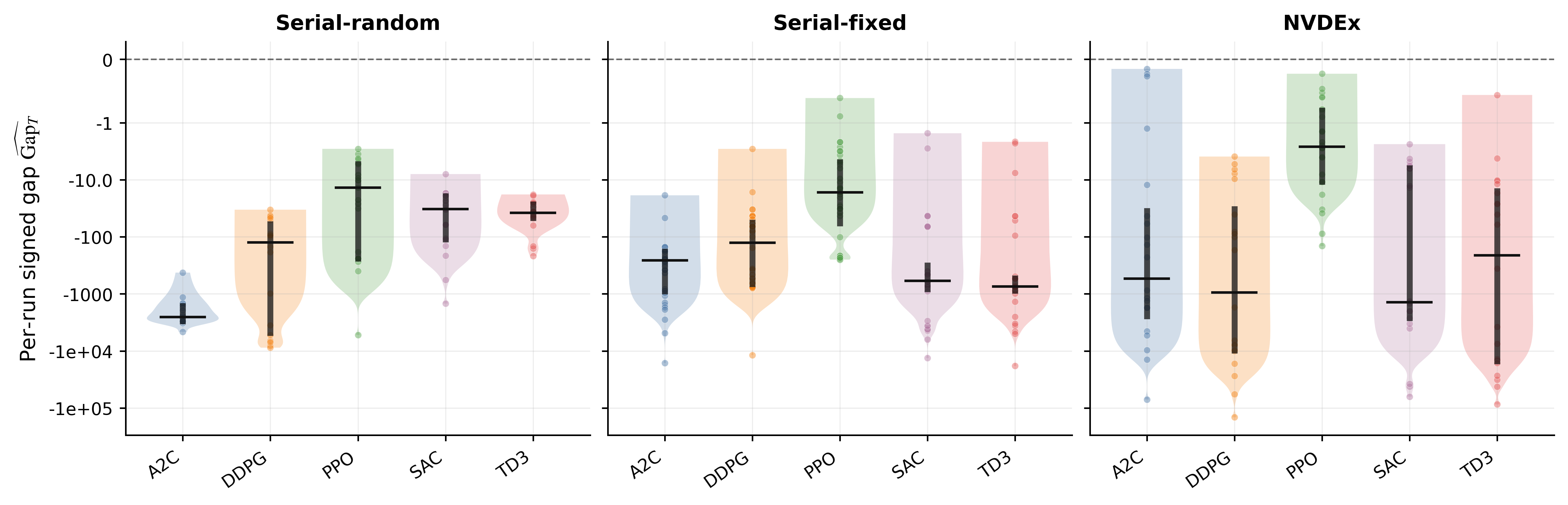}
\caption{Distribution of reward-convention optimality gaps by regime and algorithm. Points correspond to available run/fixture evaluation rows; violin widths show the empirical distribution of optimality gaps. The dashed line marks the certified optimum at zero. Black horizontal markers indicate central tendency. Unlike Figure~\ref{fig:table3_companion}, this is not a confidence-interval plot; it is a distributional diagnostic that exposes variability and lower-tail failures.}
\label{fig:signed_gap_distributions}
\end{figure*}

\begin{table*}[h]
\centering
\footnotesize
\setlength{\tabcolsep}{4pt}
\begin{tabular}{llcccc}
\toprule
Experiment & Algorithm & Optimality gap & Optimality gap 95\% CI & Regret (mean) & Regret 95\% CI \\
\midrule
Serial-random & A2C  & -2.36e+03 & [-2.79e+03, -1.93e+03] & 8.24e+03 & [6.83e+03, 9.58e+03] \\
Serial-random & DDPG & -2.29e+03 & [-3.83e+03, -9.12e+02] & 9.28e+03 & [3.50e+03, 1.63e+04] \\
Serial-random & PPO  & -3.32e+02 & [-8.45e+02, -4.92e+01] & 1.51e+03 & [1.92e+02, 3.93e+03] \\
Serial-random & SAC  & -1.65e+02 & [-3.47e+02, -4.63e+01] & 6.90e+02 & [1.72e+02, 1.53e+03] \\
Serial-random & TD3  & \textbf{-5.85e+01} & \textbf{[-8.64e+01, -3.57e+01]} & \textbf{2.14e+02} & \textbf{[1.25e+02, 3.35e+02]} \\
\midrule
Serial-fixed  & A2C  & -1.11e+03 & [-2.15e+03, -4.58e+02] & 5.78e+04 & [3.12e+04, 9.23e+04] \\
Serial-fixed  & DDPG & -6.48e+02 & [-1.34e+03, -2.51e+02] & 3.06e+04 & [1.61e+04, 5.22e+04] \\
Serial-fixed  & PPO  & \textbf{-6.46e+01} & \textbf{[-9.56e+01, -3.68e+01]} & \textbf{4.08e+03} & \textbf{[2.30e+03, 6.14e+03]} \\
Serial-fixed  & SAC  & -1.51e+03 & [-2.44e+03, -8.00e+02] & 9.03e+04 & [5.18e+04, 1.33e+05] \\
Serial-fixed  & TD3  & -1.53e+03 & [-2.69e+03, -7.61e+02] & 8.31e+04 & [5.05e+04, 1.23e+05] \\
\midrule
NVDEx         & A2C  & -5.60e+03 & [-1.33e+04, -1.06e+03] & 5.94e+02 & [1.34e+02, 1.33e+03] \\
NVDEx         & DDPG & -1.46e+04 & [-3.13e+04, -3.37e+03] & 3.22e+03 & [4.76e+02, 7.39e+03] \\
NVDEx         & PPO  & \textbf{-1.58e+01} & \textbf{[-3.07e+01, -4.57e+00]} & \textbf{2.97e+00} & \textbf{[7.01e-01, 5.82e+00]} \\
NVDEx         & SAC  & -7.76e+03 & [-1.56e+04, -1.15e+03] & 1.56e+03 & [1.13e+02, 3.13e+03] \\
NVDEx         & TD3  & -1.22e+04 & [-2.24e+04, -4.20e+03] & 2.78e+03 & [6.96e+02, 5.67e+03] \\
\bottomrule
\end{tabular}
\caption{Oracle-referenced comparison on native discrete-time fixtures. Metrics are mean and 95\% CI via paired bootstrap under CRN. For optimality gap, zero is optimal and values closer to zero are better; for regret, smaller is better.}
\label{tab:overall_performance}
\end{table*}

\end{document}